\documentclass[twoside]{article}

\usepackage[accepted]{aistats2026}

\usepackage[utf8]{inputenc} %
\usepackage[T1]{fontenc}    %
\usepackage{hyperref}       %
\usepackage{url}            %
\usepackage{booktabs}       %
\usepackage{amsfonts}       %
\usepackage{nicefrac}       %
\usepackage{microtype}      %
\usepackage{xcolor}         %
\usepackage{xfrac}    %
\usepackage{amsmath}
\usepackage{amssymb}
\usepackage{amsthm}
\usepackage{hyperref}
\usepackage[capitalise]{cleveref}
\usepackage{dsfont}
\usepackage{graphicx}
\usepackage{multirow}
\usepackage{makecell}
\usepackage{xspace}
\usepackage{caption}
\usepackage{algorithm}
\usepackage{algorithmic}

\usepackage[natbib=true,style=authoryear-comp,sorting=none,mincitenames=1, maxcitenames=1, uniquelist=true]{biblatex}
\usepackage[toc,page,header]{appendix}
\usepackage{minitoc}

\hypersetup{
    colorlinks,
    linkcolor={red!50!black},
    citecolor={blue!50!black},
    urlcolor={blue!80!black}
}

\theoremstyle{plain}
\newtheorem{theorem}{Theorem}[section]
\newtheorem{proposition}[theorem]{Proposition}

\theoremstyle{definition}
\newtheorem{definition}[theorem]{Definition}

\theoremstyle{remark}

\Crefname{equation}{Equation}{Equations}
\Crefname{figure}{Figure}{Figures}
\Crefname{tabular}{Table}{Tables}
\Crefname{section}{Section}{Sections}

\DeclareMathOperator{\Beta}{\mathrm{Beta}}
\DeclareMathOperator{\Bernoulli}{\mathrm{Bernoulli}}
\DeclareMathOperator{\BetaBern}{\mathrm{BetaBern}}
\DeclareMathOperator{\BernMix}{\mathrm{BernMix}}
\DeclareMathOperator{\Bern}{\mathrm{Bern}}
\DeclareMathOperator*{\argmax}{arg\,max}
\DeclareMathOperator*{\argmin}{arg\,min}
\DeclareMathOperator{\KL}{\mathrm{KL}}
\DeclareMathOperator{\MI}{\mathrm{MI}}
\DeclareMathOperator{\MAMI}{\mathrm{PAMI}}
\newcommand{\Lmax}{{L_{\text{max}}}}
\def\cX{{\cal X}}
\def\vs{{\mathbf{s}}}
\def\ud{{\mathrm{d}}}
\def\vx{{\mathbf{x}}}
\def\eps{\varepsilon} %

\newcommand\numberthis{\addtocounter{equation}{1}\tag{\theequation}}

\newcommand{\DqN}{{\mathcal{D}_{N}}}
\newcommand{\ZERO}{{\textsc{zero}}\xspace}
\newcommand{\ZEROs}{{\textsc{zero}s}\xspace}
\newcommand{\ONE}{{\textsc{one}}\xspace}
\newcommand{\ONEs}{{\textsc{one}s}\xspace}
\newcommand{\pbayes}{{p_{\textrm{B}}}}
\newcommand{\vL}{{v_{\textrm{L}}}}
\newcommand{\vR}{{v_{\textrm{R}}}}
\newcommand{\pdot}{{p_{(\cdot)}}}

\def \thetitle {Why Is Prompting Hard?\\ Understanding Prompts on Binary Sequence Predictors}
\begin{document}

\runningtitle{Understanding prompts on binary sequence predictors}

\runningauthor{Li K. Wenliang, Anian Ruoss, Jordi Grau-Mayo, Marcus Hutter, Tim Genewein}

\twocolumn[

\aistatstitle{\thetitle}

\aistatsauthor{Li K. Wenliang\textnormal{*\textsuperscript{\dag} } \And Anian Ruoss\textnormal{*}\And Jordi Grau-Moya\textnormal{*}\And  Marcus Hutter\textnormal{*} \And Tim Genewein\textnormal{*\textsuperscript{\dag}}}

\aistatsaddress{*Google DeepMind ~~~~~ \textsuperscript{\dag}Correspondence to \texttt{\{kevinliw,timgen\}@google.com}}
]

\begin{abstract}
Frontier models can be prompted or conditioned to do many tasks, but finding good prompts is not always easy, nor is understanding some performant prompts.
We view prompting as finding the best conditioning sequence on a near-optimal sequence predictor.
On numerous well-controlled experiments, we show that unintuitive optimal conditioning sequences can be better understood given the pretraining distribution, which is not usually available.
Even using exhaustive search, reliably identifying optimal prompts for practical neural predictors can be surprisingly difficult.
Popular prompting methods, such as using demonstrations from the targeted task, can be surprisingly suboptimal.
Using the same empirical framework, we analyze optimal prompts on frontier models, revealing patterns similar to the binary examples and previous findings. 
Taken together, this work takes an initial step towards understanding optimal prompts, from a statistical and empirical perspective that complements research on frontier models.
\end{abstract}

\doparttoc %
\faketableofcontents %

\section{INTRODUCTION}
Successful frontier models are pretrained over large datasets, which are generated by authors and creators of diverse knowledge domains, individual styles and embedded sentiments. The distribution over these high-level \emph{latent factors}, together with the data distribution of each source, implicitly defines a hierarchical generative process, or meta-distribution. As a notable example, text data can be regarded as a meta-distribution of tokens used to pretrain large language models (LLMs).
Theoretically, minimizing the next-token prediction error on the meta-distribution yields a Bayes-optimal predictor for the meta-distribution, as clearly shown by~\citet{ortega2019meta} using binary sequences.
One hallmark feature of such a predictor is its capability to be steered through human-readable prompts to execute various tasks. In the case of LLMs, prompting~\citep{brown2020language} triggers implicit \emph{inference} 
over the latent factors desired for a specific behavior~\citep{xieexplanation,wang2023large,jiang2023latent,wies2023learnability,arora2024bayesian}. Nonetheless, heuristic and handcrafted prompts often fall short of expectations, while certain demonstrably powerful prompts appear eccentric and warrant deeper fundamental understandings. 
Previous work on (binary) sequence predictors studied mostly unconditioned predictive performance~\citep[e.g.,][]{ortega2019meta,mikulik2020meta,genewein2023memory,bhattamishra2020ability,wei2022statistically,deletang2022neural}; however, how sequence predictors behave under controlled conditioning is not well-studied, nor do we know what the optimal conditioning sequences are for a given downstream task.

While many factors, such as Transformer architecture and post-training, are at play, we provide robust and fundamental insights on prompting from a Bayesian meta-learning perspective. Specifically, we study prompting on binary sequence predictors with a carefully designed framework, complementing frontier model research. Using a large collection of idealized and neural predictors (>1000 instances, see \cref{sec:compute}), we observe unexpected characteristics in optimal prompts, even in simplified scenarios. For example, consider a binary generator pretrained on sequences of i.i.d.\ coin flips, with $\mathbb{P}(\mathrm{HEADS})$ randomly chosen for each sequence. How would one prompt it to generate $70\%$ heads? A natural heuristic prompt is a sequence of $70\%$ heads; however, as we will show, this is not always the most effective prompt: the optimal prompt depends on the pretraining distribution, which is often unknown and overlooked in the past.
Employing this framework, we address further fundamental questions: Are longer prompts better? Does more data help identify optimal prompts? Can expert behavior be induced by expert demonstrations? Not always, as we elaborate in the upcoming sections.
Our main contributions are:
\vspace{-1em}
\begin{enumerate}
    \itemsep -0.2em 
    \item We establish a rigorous \emph{framework} for analyzing optimal conditioning sequences (prompts), encompassing notions of optimality, experimental designs, and prompt visualizations.
    \item Applying this framework to examples of binary data generators (DGs), we collected numerically-guaranteed optimal prompts by exhaustive search (>250k in total). Optimal prompts on seemingly trivial examples can be \emph{atypical} for the task, partly due to biases in the usually unknown pretraining DG; interpreting the prompts can thus be hard.
    \item Given a finite task dataset, we show that optimal prompts can be \emph{unreliable} to find, leading to inconsistent results for every draw of the dataset. The reliability follows \emph{unintuitive trends} depending on the task setup, e.g., increasing the task dataset size does not always improve reliability.
    \item Using an in-context learning setup involving two-arm bandits, we show that prompting by expert demonstrations is \emph{less effective} compared to optimal prompts; knowing the pretraining distribution is crucial in comprehending performant prompts.
    \item We adapt this framework for studying prompts on real LLMs, and reveal interesting patterns consistent with findings on the binary examples. 
\end{enumerate}
\paragraph{Limitations and scope.} 
We remark upfront that, although our binary data tasks make our results clear and interpretable, they are drastic simplifications of natural language data. 
Consequently, while our findings offer foundational insights, most of our results (excluding our LLM investigations) are not tied to specific frontier models, preventing definitive claims about their exact behaviors at scale. To ensure optimality, our exhaustive search fully explores the prompt space, and comparisons with practical partial-search methods are out of scope. This paper takes a Bayesian meta-learning view of sequence predictors, and we do not fully study the effects of post-training or neural architecture; as a result, a comprehensive theory of prompting in real LLMs is beyond the current aim. Extended limitations are discussed in \cref{sec:limitations}.

\section{RELATED WORK}
A large body of work aims to find effective prompts. Prompt engineering~\citep[e.g.,][]{
wei2022chain,liu2023pre,chen2023unleashing,marvin2023prompt,sahoo2024systematic,song2024communication,khattab2023dspy,xu2024prompting}
improves performance using intuitions from system design, search and planning. Prompt optimization~\citep[e.g.,][]{
pryzant2023automatic,wang2023promptagent,fernando2023promptbreeder,guo2023connecting,wu2024prompt,hao2024optimizing}
instead employs variants of discrete optimization and search techniques. %
Another use of prompting, particularly in agentic or robotic tasks, is by providing expert demonstrations to induce desired policies for the task 
~\citep[e.g.][]{dong2022survey,ruoss2024lmact,agarwal2024many,laskincontext}.
While much focus has been put on improving the performance of the resulting prompts, we do not know whether they are actually optimal, since the prompt space on text is enormous. It has also been challenging to \emph{interpret} the prompts by relating to the desired tasks, and to understand concretely why they work
~\citep{webson2022prompt,daras2022discovering,murr2023testing,he2024does},
why some prompts work better than others despite no obvious difference to humans~\citep{kojima2022large,mizrahi2024state}, and why some prompts that are intuitively expected to work may disappoint~\citep{zamfirescu2023johnny,khurana2024and}.
These difficulties are exacerbated by the variations across architectures, pretraining mixtures, finetuning processes, and nuanced parameters that vary across public frontier models~\citep{chen2023chatgpt}.
The unintuitive nature of prompting has caused safety and ethical concerns
~\citep{wei2024jailbroken,wu2023jailbreaking,xu2024comprehensive,cherepanova2024talking,liuautodan,zou2023universal}.
It is thus worth taking a step back to understand this issue more fundamentally.

\begin{figure*}[t]
    \centering
    \includegraphics[page=1,width=0.95\textwidth]{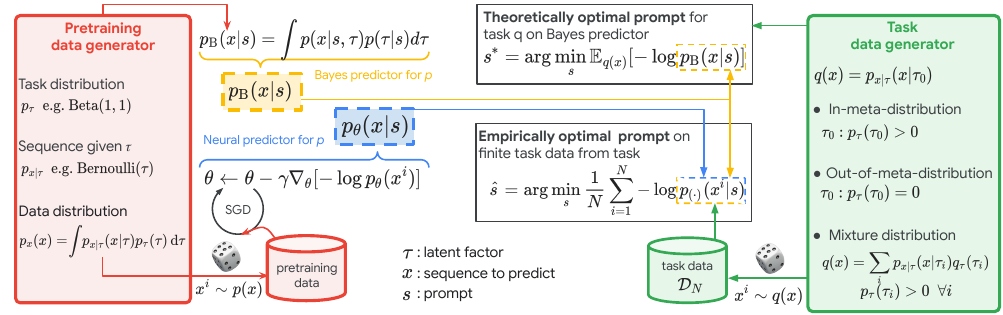}
    \caption{Framework for obtaining optimal prompts under pretraining and task data generators.
    }
    \label{fig:experiment_figure}
    \vspace{-0.2cm}
\end{figure*}

\section{BACKGROUND}
To study prompting under well-controlled conditions, we design synthetic data generators (DGs) with latent factors (akin to~\citet{xieexplanation,jiang2023latent}) to define pretraining distributions and (downstream) tasks. These DGs,
from Bernoulli sequences in \cref{sec:CIB-DG_methods,sec:CIB-DG_results} to bandit problems in \cref{sec:bandit}, 
 induce a (meta-)distribution over binary token sequences. We then adapt this framework to study real LLMs in \cref{sec:real_llms} where the latent factors are given but the pretraining distribution over texts is unknown.

To draw a sequence, the DG first draws a random latent value $\tau$ from a distribution $p_\tau(\tau)$; such as sampling a bias of a coin from a distribution on [0, 1]. Then, sample a sequence of length $T\in\mathbb{N}^+$ from a conditional distribution $p_{x|\tau}(x_{1:T}|\tau)$, such as a sequence of flips from a coin with bias $\tau$.
The distribution induced by this DG is then the marginal (mixture)
$
    p_x(x_{1:T}) = \int p_{x|\tau}(x_{1:T}|\tau) \ud p_\tau(\tau)\,.
$
The sequence $x_{1:T}$ is composed of binary tokens from the alphabet $\mathcal{A}:=\{0,1\}$. 
Subscripts are omitted when no confusion should arise.
We consider (piecewise) conditionally independent DGs where the conditional $p_{x|\tau}$ satisfies the condition that,
for some lengths $T, L\in\mathbb{N}^+$, we have $p_{x|\tau}(s_{1:L}x_{1:T} | \tau)=p_{x|\tau}(s_{1:L}|\tau)p_{x|\tau}(x_{1:T} | \tau)$ for all sequences $s_{1:L}\in\mathcal{A}^L$ and $x_{1:T}\in\mathcal{A}^T$.
In this case, when only $s_{1:L}$ is given to predict an unknown $x_{1:T}$, the Bayes-optimal predictor (or Bayes predictor) under the sequence distribution $p_x$ is
\begin{equation}\label{eq:analytical_predictive}
\textstyle
  \pbayes(x_{1:T}|s_{1:L}):=\int p_{x|\tau}(x_{1:T}|\tau) \ud p_{\tau|x}(\tau|s_{1:L}) \,,
\end{equation}
where $p_{\tau|x}(\tau|s_{1:L})=p_\tau(\tau)p_{x|\tau}(s_{1:L}|\tau) / p_x(s_{1:L})$ is the posterior over the latent factors given $s_{1:T}$. We regard $s_{1:T}$ as the \emph{prompt} for the sequence(-to-predict) $x_{1:T}$.
Conditional independence helps interpret $s$, and 
is a common assumption for synthetic datasets~\citep{xieexplanation,jiang2023latent,wang2023large}.
\subsection{Conditionally Independent Bernoulli DGs}

Our first set of experiments in \cref{sec:CIB-DG_methods} uses DGs from the family
of conditionally independent Bernoulli DGs (CIB-DGs), where
$p_\tau$ has support on the unit interval, and 
$p_{x|\tau}(x_{1:T} | \tau)=\prod_{t} \mathrm{Bernoulli}(x_t; \tau)$.
The Bayes predictor of a CIB-DGs has the property that 
two prompts $s_{1:L}$ and $s'_{1:L}$ are equivalent in the sense that $\pbayes(\cdot | s_{1:L})\equiv \pbayes(\cdot| s'_{1:L})$ if they have the same counts of zeros and ones, defined as 
\begin{equation}\label{eq:counts}
\begin{aligned}
S_0(s_{1:L})\!:=\!\sum_{t=1}^L\!\mathds{1}[s_t\!=\!0],~
S_1(s_{1:L})\!:=\!\sum_{t=1}^L\!\mathds{1}[s_t\!=\!1].
\end{aligned}
\end{equation}
This permutation invariance reduces the cost of prompt search on the Bayes predictor, but is unlikely to hold for neural networks with standard architectures~\citep{mikulik2020meta}.
We define the particular CIB-DGs and their Bayes predictors used in our first set of experiments:
\begin{definition}
$\Bern(\tau)$ is the CIB-DG 
with $p_\tau=\delta_{\tau}$, where $\delta_\tau$ is Dirac delta.
The Bayes predictor evaluated on $x_{1:T}$ is trivially $\Bern(\tau)$ itself, that is,
$$
  \pbayes(x_{1:T}|s_{1:L})=p_x(x_{1:T})=\tau^{S_1(x_{1:T})}(1-\tau)^{S_0(x_{1:T})}
$$
\end{definition}
\begin{definition}\label{thm:bernoulli-mixture}
$\BernMix(w, \tau_1, \tau_2)$ 
$p_\tau=(1-w)\delta_{\tau_1} + w\delta_{\tau_2}$ for $0<\tau_1<\tau_2<1$. 
In our experiments we use equal mixture weights, and write: $\BernMix(\tau_1,\tau_2):=\BernMix(0.5, \tau_1,\tau_2)$.
Its Bayes predictor is $\pbayes(\cdot|s_{1:L})=\BernMix(w_L(s_{1:L}), \tau_1,\tau_2)$, where
\begin{gather}\label{eq:mixture_weight_posterior}
w_L(s_{1:L})^{-1}=1+\left(\frac{\tau_1}{\tau_2}\right)^{S_1(s_{1:L})}\!\left(\frac{1-\tau_1}{1-\tau_2}\right)^{S_0(s_{1:L})}.
\end{gather}
\end{definition}
\begin{definition}\label{thm:beta-bernoulli}
$\BetaBern(\alpha,\beta)$ is the CIB-DG with $p_\tau=\mathrm{Beta}(\alpha, \beta)$.
The Bayes predictor is
$\pbayes(\cdot|s_{1:L}) = \BetaBern(\alpha+S_1(s_{1:L}), \beta + S_0(s_{1:L})).$
\end{definition}
Despite the apparent simplicity, these DGs can be adapted to generate simple text data, such as synthetic movie reviews: let $\tau$ be the sentiment over movies, with 0 being most negative and 1 most positive, and take a categorical $p_{x|\tau}$ over adjectives conditioned on the sentiment; a template for reviews can be \{``I have never seen such a \_ story'',  ``The visuals are \_.'', ``How \_ is the cast!'', \ldots \}, which can be concatenated to form a full review. Predicting the next adjective then resembles binary sequence prediction. Using binary tokens makes the prompts more interpretable and visualizable than using natural language tokens, and allows us to better appreciate the challenges of prompting. 

\subsection{Neural Predictors and Meta-Learning}\label{sec:neural_predictor}
In our experiments, we train autoregressive neural predictors on samples from meta-distributions.
Denote such a predictor by $p_\theta(\cdot)$ where $\theta$ are parameters. The objective for a single sequence 
$x_{1:T}$ is
$
   - \log p_\theta(x_{1:T}).
$
The latent $\tau$ is resampled for each $x_{1:T}$.
Under realizability and convergence~\citep{ortega2019meta}, neural predictors can \emph{meta-learn} to predict sequences by adapting to different latent factors, giving predictions that are indistinguishable from the Bayes predictor over $x\sim p_x$~\citep{wenliang2018neural,mikulik2020meta,genewein2023memory,wu2023many,grau2024learning}.
However, this theory does not predict or describe the optimal prompts from (near-)optimal predictors when $x$ is drawn from a task distribution $q$ (see \cref{sec:super_bayes}). This motivates our empirical framework below.
We note upfront that, the results derived from simplifications of CIB-DGs cannot exhaustively explain or accurately predict the behaviors of full-scale LLMs. 

\section{INVESTIGATION FRAMEWORK}\label{sec:CIB-DG_methods}
Our framework is illustrated in \cref{fig:experiment_figure}, which is first applied to CIB-DGs  in \cref{sec:CIB-DG_results} before more complex DGs.
For each pretraining DG $p$, we obtain the ideal Bayes predictor $\pbayes$ (by \cref{eq:analytical_predictive}) and practical neural predictors $p_\theta$ of recurrent and Transformer-based architectures; see \cref{sec:neural_predictor_detail}.
We then prompt a predictor towards a task specified by a task DG $q$. The task is directly manipulated by the DG, distinct from the activation or features from a network~\citep[e.g.][]{mittal2024does,hendel-etal-2023-context,todd2024function}, which is discovered rather than directly controlled.

\vspace{-0.5em}
\paragraph{Theoretically optimal prompt $s^*$.} Given the Bayes predictor $\pbayes$ of a pretraining DG $p$, the quality of a prompt $s_{1:L}$ towards a task DG $q$ is
\begin{equation}\label{eq:log_loss_L}
\textstyle
   \mathcal{L}(\pbayes, q, s_{1:L}) := -\sum_{x\in\mathcal{A}^T}q(x)\log \pbayes(x|s_{1:L})\,.
\end{equation}
The \emph{theoretically optimal} prompts of length $L$ and up to length $\Lmax$ are, respectively,
\begin{equation}\label{eq:optimal_prompt_L}
\begin{gathered}
    s^*_{1:L}(\pbayes, q) := \argmin_{s_{1:L}} \mathcal{L}(\pbayes, q, s_{1:L}); \\
    s^*_\Lmax(\pbayes,q) := \argmin_{\ell\in\{1,...,\Lmax\},s_{1:\ell}} \mathcal{L}(\pbayes, q, s_{1:\ell}).
\end{gathered}
\end{equation}
Both $s^*_{1:L}$ and $s^*_\Lmax$ depend on the full distributions of the pretraining and task DGs. In this work, we find them by evaluating \cref{eq:log_loss_L} for each possible prompt (exhaustive search) under the length constraints; see \cref{sec:prompt_search}. Previous work~\citep{bhargava2023s,renze2024benefits,kusano2024longer,wang2025towards,lester2021power}
reported better performance of shorter prompts on LLMs
(but see~\citep{liu2025effects,leidinger2023language}) without exhaustive search; here, we investigate this by explicitly controlling the prompt length.  

\vspace{-0.3em}
\paragraph{Empirically optimal prompt $\hat s\,$.}
Given a dataset $\DqN:=\{x_{1:T}^{i}\}_{i=1}^N$ of $N$ sequences from $q$, we optimize the prompt for a predictor $\pdot\in\{\pbayes, p_\theta\}$ under the empirical version of the loss \eqref{eq:log_loss_L}:
\begin{equation}\label{eq:log_loss_L_est}
\textstyle
   \mathcal{\hat{L}}(\pdot, \DqN, s_{1:L}) := -\frac{1}{N}\sum_{i=1}^N \log \pdot(x_{1:T}^{i}|s_{1:L}).
\end{equation}
The resulting \emph{empirically} optimal prompts are denoted by $\hat{s}_{1:L}$ and $\hat{s}_{\Lmax}$ under the respective length constraints before. Each exhaustive search uses a \emph{fixed} $\DqN$, and we repeat for different draws of $\DqN$ to give a distribution of $\hat s$.
Both the theoretical $s^*$ and empirical $\hat{s}$ depend on the pretraining and task DGs, the sequence length $T$, and the prompt length $L$ or $\Lmax$; the empirical $\hat{s}$ depends also on the predictor $\pdot$ and dataset size $N$.
Each configuration forms a \emph{prompt setup}, which we vary experimentally.
Previous work on prompt search~\citep[e.g.][]{deng2022rlprompt,pryzant2023automatic,hao2024optimizing} did not systematically study the effects of these hyperparameters, while brittleness of optimized prompts is known~\citep{he2024does,li2023robust,ngweta2025towards,shi2024robustness,peng2025dlpo}.

To interpret optimal prompts, we consider two classes of task distributions: the task DGs $q$ can be either \emph{in-meta-distribution} (IMD) w.r.t.\ $p$, where $q\in\mathcal{M}_p\!:=\!\{p_{x|\tau}(\cdot|\tau_0)\,|\,p_\tau(\tau_0) \!>\! 0\}$, or \emph{out-of-meta-distribution} (OOMD) w.r.t.\ $p$ ($q\notin\mathcal{M}_p$).
See \cref{tab:opt_prompt_summary} for examples illustrated by CIB-DGs.
This separation formalizes the notion of a task being ``within'' or ``outside'' the pretraining distribution in frontier model research~\citep{wei2021pretrained,krishna2023downstream,wang2023large,petrov2024when}.
In the IMD case, we can relate the prompt to the task sequence $x$:
\begin{proposition}\label{thm:info}
Prompting a predictor $p$ by $s^*$ for all $q\in\mathcal{M}_p$ maximizes the mutual information between the prompt $s$ and the sequence $x$ penalized by a ``prompt alignment'' with the predictor $p$.
\end{proposition}
\vspace{-0.3em}
This informal version of \cref{thm:info_formal}, proved in \cref{sec:info}, reveals that the theoretical $s^*$ must induce a desired distribution on $x$, but is also affected by the predictor $p$ and in turn the pretraining distribution. Specifically, the theoretically optimal prompt is, in general, \emph{not} the most likely sequence or a sample from $q$.
As such, we do not expect the optimal prompts to be interpretable based on the task distribution alone.
Meanwhile, the empirical $\hat{s}$ is stochastic due to randomness in $\DqN$ and pretraining. We visualize its   distribution for each prompt setup. To measure the reliability of empirical $\hat s$, we estimate the probability that $\hat{s}$ matches $s^*$, or \emph{proportion correct} (see \cref{sec:proportion_correct}), averaged over dataset draws and network instances.  
We also visualize the loss ``landscape'' (see \cref{sec:loss_landscape}) of each prompt setup. A sharper loss ``landscape'' around $s^*$ can improve reliability, and \emph{vice versa}. 
A flat landscape can lead to worse identifiability. Although the near-optimal prompts are acceptable for performance, inconsistent results often complicate interpretation.

\vspace{-0.5em}
\paragraph{Advantages of simplified setup.} Despite its obvious simplicity, using binary sequences and small-scale neural predictors brings the following advantages that are infeasible at full-scale LLMs:
\vspace{-0.7em}
\begin{enumerate}\setlength{\itemsep}{-2pt}
    \item The optimal Bayes predictor is available: the idealized limit any model or data scaling approaches;
    \item Exhaustive search to obtain numerically guaranteed optimal prompts;
    \item Visualizing the prompt distribution as the prompt set up changes;
    \item Repeating experiments over many random seeds to produce robust results.
\end{enumerate}
\begin{table*}[t]
\footnotesize
\centering
\caption{Summary of results on conditionally independent Bernoulli data generators. 
\label{tab:opt_prompt_summary}
}\vspace{-0.5em}
\begin{tabular}{ll|ccc}
\hline
\multirow{2}{*}{\textbf{Pretraining $p$}} & \multirow{2}{*}{\textbf{Task $q$}}            & \multicolumn{1}{c}{\multirow{2}{*}{\textbf{$s^*$ matches $\tau$ in $q$?}}}& \multicolumn{2}{c}{\textbf{$\hat s = s^*$ reliably }} \\
                                             &                                                      & \multicolumn{1}{c}{}                                              & \textbf{Bayes}    & \textbf{Neural}    \\ \hline
\multirow{2}{*}{$\BernMix(0.2, 0.7)$}        & $\Bern(0.7)\in\mathcal{M}_p$                                &  No, extreme counts                           & No                & No                                     \\
                                             & $\Bern(0.6)\notin\mathcal{M}_p$                             & {No, unintuitive optimal length}              & No                & No                                     \\\hline
\multirow{2}{*}{$\BetaBern(1, 1)$}           & $\Bern(\tau)\in\mathcal{M}_p$                               & {Yes if $p(\tau)$ uniform, $N, T$ large }     & Yes               & Yes?                                    \\
                                             & $\BernMix(\tau_1, \tau_2)\notin\mathcal{M}_p$               & {Yes, for mean bias}                          & Yes               & Yes?                                   \\ \hline
\end{tabular}
\end{table*}
\begin{figure*}[t] %
    \centering
    \includegraphics[width=0.45\textwidth]{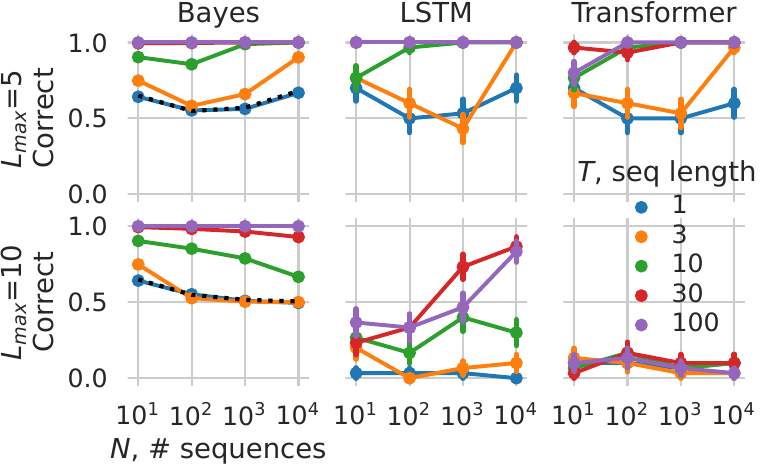}\hfill
    \includegraphics[width=0.54\textwidth]{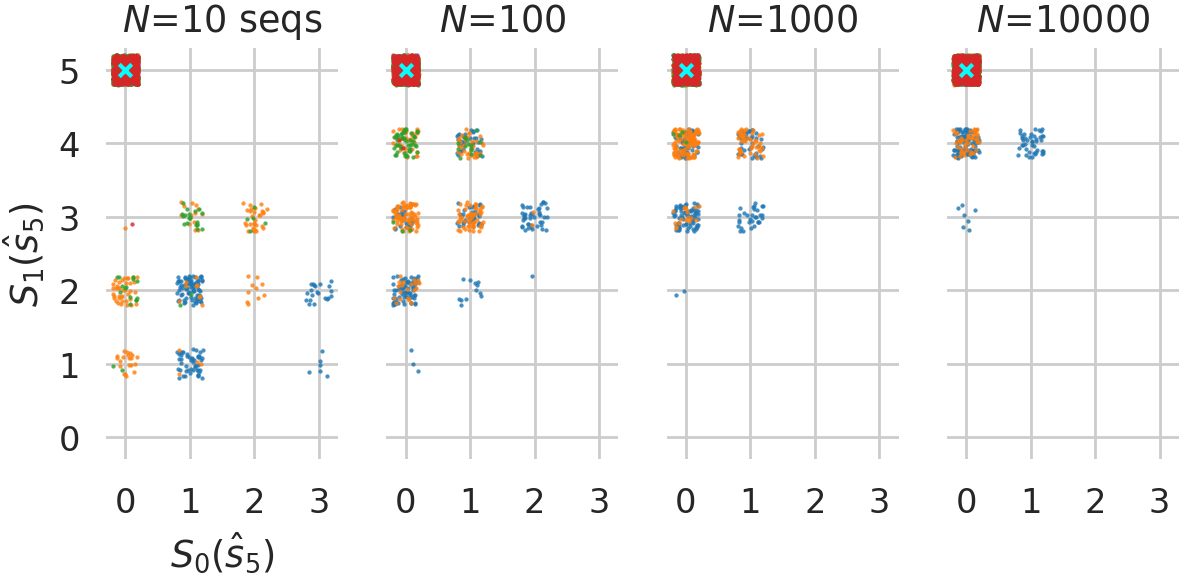}
    \caption{
    Results for a pretraining DG $p=\BernMix(0.2, 0.7)$ and a task DG $q=\Bern(0.7)$.
    Left, the proportion correct for Bayes predictor ($10^3$ seeds per data point)
    and two neural predictors (30 seeds per data point). Error bars show 1 SEM.
    The black dotted line is the theoretical value for $T=1$ (\Cref{sec:bernmix_bern_0.7_nonmonotonic}).
    More data does not always improve reliability. 
    Additional results are in \cref{sec:bernmix_bern_0.7_correct}.
    Right, empirically optimal $\hat{s}$ at $\Lmax=5$ for the Bayes predictor 
    for different values of $T$ (colors) and $N$ (panels); 
    100 repetitions per setting. The cyan cross  shows the correct all-\ONE $s^*_5$. The set of $\hat{s}$ moves towards $s^*$, but the proportion of $\hat{s}=s^*$ does not necessarily increase.
    }
    \label{fig:sensitivity}
    \vspace{-1em}
\end{figure*}

\section{RESULTS ON CIB-DGS}\label{sec:CIB-DG_results}
We take four pairs of pretraining and task CIB-DGs ($p$ and $q$), listed in \cref{tab:opt_prompt_summary}; the motivation is to cover as many prompt setups as possible, including discrete/continuous pretraining $\tau$, and IMD/OOMD tasks, while being able to easily visualize the optimal prompts.
For each pair, to find $s^*$, we sweep sequence length $T\in\{1,3,10,30,100\}$, and maximum prompt length $\Lmax\in\{5, 10, 15\}$; to find $\hat s$, we additionally sweep optimization datasets consisting of $N\in\{10^1,10^2,10^3,10^4\}$ sequences, and 7 predictor types (3 in main text).
For the Bayes predictor, we draw $\DqN$ $10^3$ times with different seeds; 
for each neural predictor type, we pretrain 30 instances to be near-optimal (\cref{fig:pretrain_kl} in \cref{sec:additional_results}) and find $\hat s$ for each on drawn $\DqN$. All neural predictors are trained using Jax/Haiku~\citep{jax2018github,deepmind2020jax,haiku2020github}, with parameters initialized by their default method, and optimized by Adam~\citep{kingma2015adam} with their default hyperparameters; see \cref{sec:compute} for detail.  

\begin{figure*}[t] %
    \centering
    \includegraphics[width=0.45\textwidth]{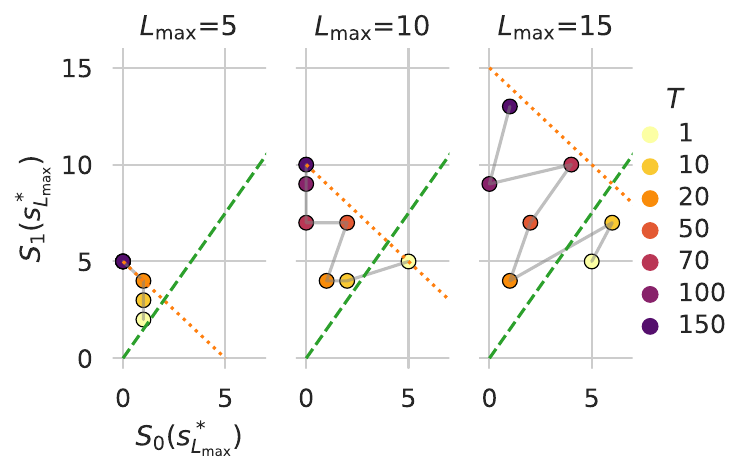}~~~
    \includegraphics[width=0.53\textwidth]{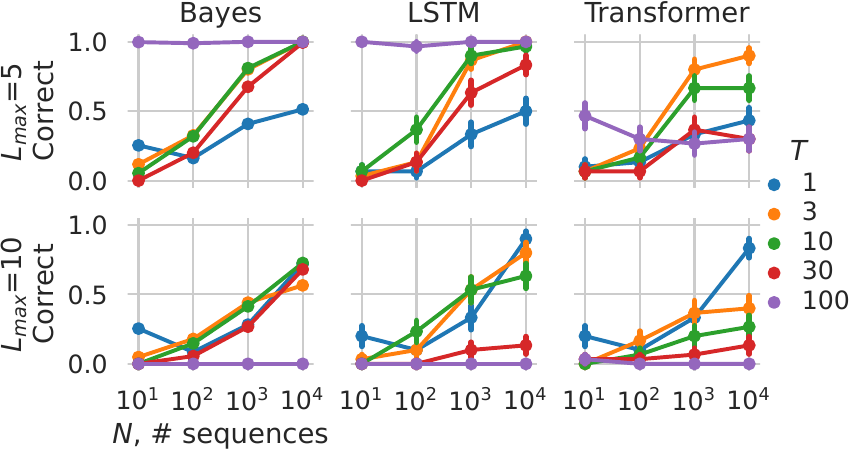}
    \vspace{-0.5em}
    \caption{Results for $p=\BernMix(0.2, 0.7)$ and $q=\Bern(0.6)$. 
    Left, each circle represents the \ZERO/\ONE counts of the theoretical $s^*$. 
    The orange dotted line indicates $\Lmax$.
    The green dashed line marks 60\% \ONEs. The $s^*$ jumps unpredictably before converging to all-\ONE.
    Right, the proportion correct of $\hat{s}=s^*$.
    They mostly increase except for $T=100$ when $\Lmax=10$.
    \Cref{fig:sensitivity_0.6_supp} shows additional results.
    }
    \label{fig:length_dependence}
    \vspace{-0.5em}
\end{figure*}
\subsection{Optimal Prompt is Atypical}\label{sec:sensitivity_to_sampling}
Consider a pretraining DG $p=\BernMix(0.2, 0.7)$ and task DG $q=\mathrm{Bern}(0.7)$ ($q\in\mathcal{M}_p$).
The theoretical $s^*$ is always a sequence of all \ONEs, since it causes the posterior $p_{\tau|x}$ to concentrate on 0.7 most rapidly. However, such a prompt is atypical to the task DG, and would be unintuitive without knowing $p$. Optimized natural language prompts can also be surprising for the given task %
~\citep[e.g.][]{daras2022discovering,he2024does,melamed2025demystifying};
which may be more effective at shifting the model's belief than more intuitive
prompts.

How likely can we find such a prompt empirically? As \cref{fig:sensitivity}(left) shows, the empirical $\hat s$ is mostly correct for longer sequences when $\Lmax=5$, but not at $\Lmax=10$ for the Transformer, even with a large $N$ and $T$. 
In addition, increasing sequence length $T$ can \emph{lower} the proportion correct, even when prompting the Bayes predictor; see \Cref{sec:bernmix_bern_0.7_nonmonotonic} for an explanation.
In addition, \Cref{sec:bernmix_bern_0.7_landscape} shows a flat loss ``landscape'': there are many suboptimal prompts with similar losses \eqref{eq:log_loss_L} as $s^*$. Thus, the randomness in sampling $\DqN$ can shift the optimum to a different prompt near $s^*$. On frontier models, prompt search typically uses a small batch size ($\approx 200$)~\citep{deng2022rlprompt,pryzant2023automatic,fernando2023promptbreeder,hao2024optimizing}, raising the question how reliable the optimized prompts are.

How can the proportion correct decrease with increasing $N$? \Cref{fig:sensitivity}(right) shows the distribution of empirical $\hat{s}$ for $\Lmax=5$. As $N$ increases, the \emph{support} of this distribution approaches $s^*$, but the \emph{probability} of matching $s^*$ may decrease (e.g., yellow dots). \Cref{sec:sensitivity_emp_dist} presents the distribution of $\hat s$ from all predictors, showing non-convergence to $s^*$, with inconsistent counts of \ZEROs and \ONEs across runs.
Note that the prompting strategy of using samples from the task $q\in\mathcal{M}_p$ is most likely suboptimal: sampling the all-\ONE sequence from the task DG is increasingly unlikely as $T$ increases.
Knowledge of the pretraining DG is crucial to come up with and make sense of $s^*$.

\begin{figure*}[t]
    \centering
    \includegraphics[height=0.21\textheight]{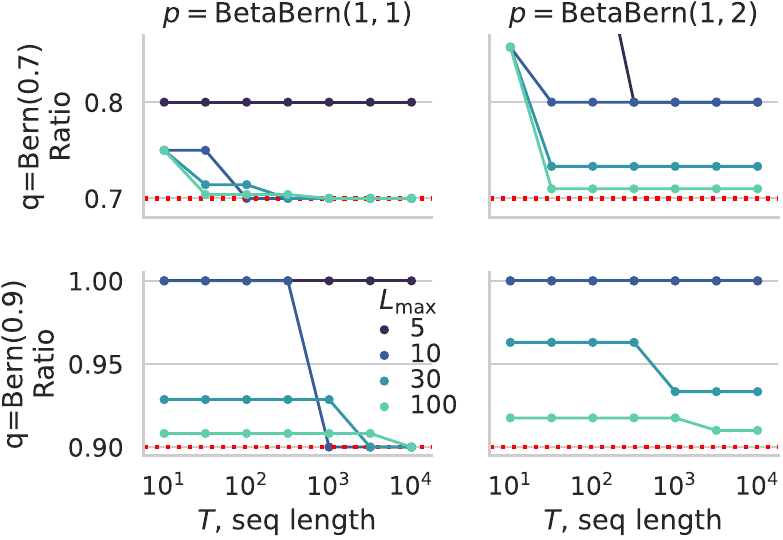}~~~~
    \includegraphics[height=0.21\textheight]{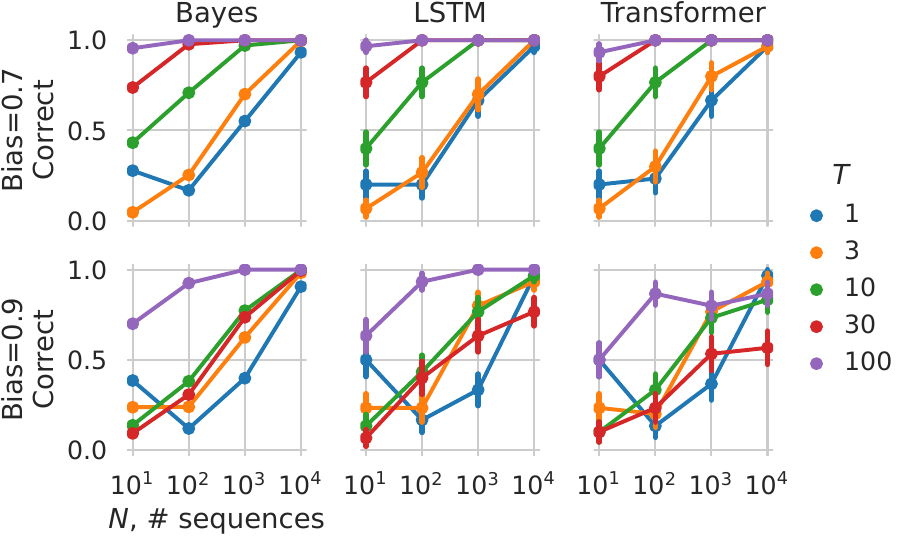}
    \vspace{-0.8em}
    \caption{
    for $p=\BetaBern(1, \beta)$ with $\beta\in\{1, 2\}$, and $q=\Bern(\tau)$ with $\tau\in\{0.7, 0.9\}$.
    Left, the ratio of \ONEs in the theoretical $s^*$. Red dotted line shows true bias of $q$. The ratio goes towards the task $\tau$ when $N$ and $T$ are large and $p_\tau$ uniform.
    Right, proportion correct ($\beta=1$) increases as $N$ increases. \Cref{fig:beta_categorical_supp} has more results.
    }
    \label{fig:beta_categorical}
    \vspace{-1em}
\end{figure*}

\subsection{Are Shorter Prompts Better?}\label{sec:ood_prompting}
In the previous example, the optimal prompt had maximal length $\Lmax$. Are longer prompts better? Let $p=\BernMix(0.2, 0.7)$ as before, and now $q=\Bern(0.6)$ ($q\notin\mathcal{M}_p$).
\Cref{fig:length_dependence}(left) shows that the ratio of \ONEs in the theoretical $s^*$ appears irrelevant to $\tau=0.6$ in $q$, so interpreting it given $q$ alone is futile. 
See \cref{sec:bernmix_bern_0.6_length} for an explanation.
More interestingly, at a given maximum length constraint $\Lmax$, the theoretical $s^*$ can shorten or lengthen with no obvious pattern as $T$ increases (\cref{fig:bernmix_bern_0.6_optimal_length_vs_Lmax}).
It can leave an \emph{impression} that the best prompts remain short even when longer prompts are tested, but, as we gradually increase $\Lmax$, $s^*_\Lmax$ never shortens and can suddenly jump to a longer one. 
These results suggest that previously reported higher efficacy of shorter natural language prompts ~\citep{bhargava2023s,renze2024benefits,kusano2024longer,wang2025towards,lester2021power} could be due to insufficient search at longer prompt lengths, which is inevitable given the much larger prompt space. 
\Cref{fig:length_dependence}(right) shows that, for $\Lmax=10$, the empirical $\hat s$ gets more unreliable as $T$ increases, regardless of the dataset size and the predictor, due to the flat loss ``landscape'' (\cref{sec:bernmix_bern_0.6_landscape}). 
The distribution of $\hat s$ does \emph{not} always concentrate around $s^*$, complicating the interpretation of empirically optimized prompts.

\subsection{Continuous Latent Factors}\label{sec:continuous_latent}
\vspace{-0.2cm}
We now consider a pretraining $p=\BetaBern(1, 1)$ (uniform over $[0, 1]$) and two tasks $q=\Bern(\tau)$ with $\tau\in\{0.7, 0.9\}$ ($q\in\mathcal{M}_p$). Intuitively, the theoretical $s^*_\Lmax$ should have roughly $\tau\Lmax$ \ONEs and $(1-\tau)\Lmax$ \ZEROs.
As \cref{fig:beta_categorical}(left) shows, this only holds for large enough $T$ and $\Lmax$.
The detailed mismatch for shorter $T$ can be explained if we know the pretraining $p$; see \cref{sec:betabern_bern_shorterT}.
In addition, \Cref{fig:beta_categorical}(right) shows a close match in the proportion correct between all predictors, in contrast to previous cases. Overall, increasing $T$ and $N$ helps identify the theoretical $s^*$ more reliably, consistent with the sharper loss ``landscape'' in \cref{sec:betabern_bern_landscape} compared to previous cases.
However, uniformity on $\tau$ is unlikely to hold for real languages that exhibit biases in multiple domains~\citep{SONG2023103277,kotek2023gender}.
Indeed, the ratio of \ONEs in $s^*$ deviates from the true bias in $q$ if the $p_\tau$ is not uniform, 
such as when $p=\BetaBern(1, 2)$; see \cref{fig:beta_categorical}(left) and \cref{sec:betabern_bern_supp}.
Methods that control for the bias~\citep[e.g.][]{gagne2023inner,shirafuji2024bias,delobelle2022fairdistillation} could improve interpretability and reliability of the prompts. 
In all previous examples, the task DG was a coin with fixed bias ($\Bern(\tau)$).
In \cref{sec:betabern_bernmix}, we present an example where $q=\BernMix(\tau_1, \tau_2)$ is a mixture, which requires switching between two values, rather than a continuum, of $\tau$.
The optimal prompt length varies depending on $q$, and would again be difficult to explain without knowing the pretraining DG.
As a side note, since $p_{\tau|x}$ is always single-mode during pretraining, prompting can only partially improve performance if $|\tau_1-\tau_2|$ is large. This leads to an idea of using a mixture of prompted predictors, each focusing on one $\tau$, as done in~\citep{choi-etal-2023-smop,dun2025sweeping}. 

\section{MORE COMPLEX BINARY SEQUENCES}
The summarizing \Cref{tab:opt_prompt_summary} suggests that optimal prompts may be more interpretable if $p_\tau$ is uniform and $q\in\mathcal{M}_p$. However, this is not always the case, as we show in \Cref{sec:switching} with non-i.i.d.\ DGs. Going beyond conditional independence, we constructed DGs with more complex structures in \Cref{sec:hierarchical,sec:topic_model} to simulate basic features of natural language. The optimal prompts are harder to understand, suitable for more experienced readers. Below, we apply our framework to an in-context learning scenario and real LLMs. 

\subsection{Bandit Decision-Maker}\label{sec:bandit}
We have seen that optimal prompts may not be typical. This is relevant to popular in-context (reinforcement) learning approaches, where one prompts a predictor with expert demonstrations on example problems in a problem class (e.g.\ some chess games) to induce expert behavior on new problems from the class (a new chess game)
\citep[e.g.,][]{luo2024context,xu2023expertprompting,WangPQLZWGGN00024,ruoss2024lmact,dai2024context,huang2022language,yao2023react,costarelli2024gamebench,mirchandani2023large,agarwal2024many,dai2024context,laskincontext}.
However, ensuring optimality for the LLM prompts is virtually impossible.
In our final experiment, we setup a well-controlled problem class suitable for binary sequence predictors as the agents, and prompt the binary agents (including a Bayes predictor) to attain maximal returns, rather than minimal log-losses used before.

The problem class is a two-arm Bernoulli bandit (\Cref{fig:bandit}).
In each episode, problems (games) are created by drawing the latent reward probabilities of the left and right arms ($\vL$ and $\vR$) i.i.d.\ from $\text{Uniform}([0, 1])$.
Then, at each step $t$, the agent chooses an action $a_t\in\{\text{L},\text{R}\}$ given past actions and rewards, and obtains a new reward $r_t\sim\Bernoulli(v_{a_t})$. The goal is to find a policy that maximizes the expected total return $\mathbb{E}[\sum_{i=1}^T r_i]$
for $T=300$, estimated over many episodes.
\begin{figure*}[t]
    \centering
    \includegraphics[width=0.98\textwidth]{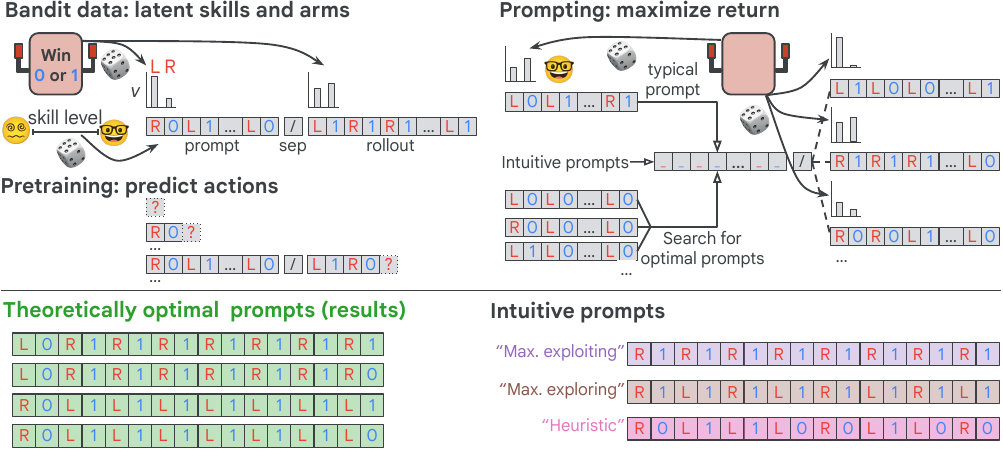}\\
    \vspace{1mm}
    \includegraphics[width=0.98\textwidth]{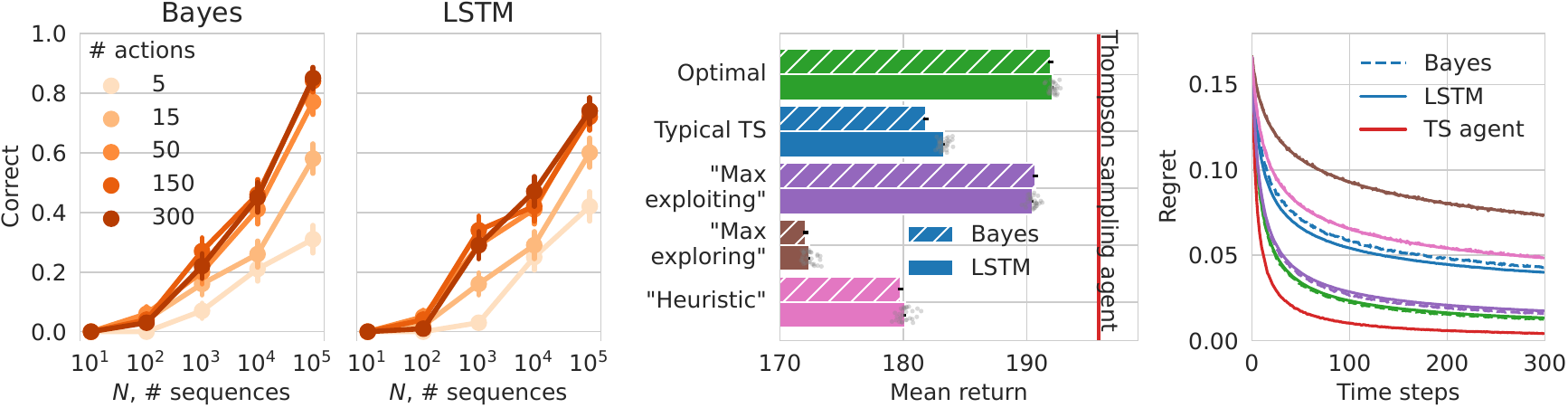}
    \caption{Upper, experiment design for the bandit task. 
    Lower, 4 equivalent theoretical $s^*$, 3 heuristic prompts, and performance metrics on the bottom panels: Left two panels, the proportion correct
    increases slowly for the Bayes and LSTM predictors; error bars show 1 SEM over 100 pretraining runs with different weight initializations. Right two panels, the mean return and mean instantaneous regret of different prompts. We use the $\hat s$ found on $N=10^5$ rollouts. Grey dots show 30 out of 100 runs.
    }
    \label{fig:bandit}
    \vspace{-0.8em}
\end{figure*}
\vspace{-1em}
\paragraph{Agent mixture.} We design a distribution of agents with different skill levels $\tau$ (hidden from predictors), simulating the diverse skills in the text corpora of LLMs.
At time $t$, each agent maintains the counts of rewarded $S_{b,1,t}$ and unrewarded $S_{b, 0,t}$ actions on each arm ($b\in\{\text{L},\text{R}\}$) up to time $t$, and takes an action $a_{t+1}=\argmax_{b\in\{\text{L},\text{R}\}}\{v_{b,t,\tau}\}$,
where 
$
   v_{b,t,\tau} \sim \textrm{Beta}(\tau S_{b,1,t}+1, \tau S_{b,0,t}+1) \text{~for~} b\in\{\text{L},\text{R}\}.
$
Here, $\tau\in[0,1]$ is the latent skill, 
with $\tau=0$ giving uniformly random actions and $\tau=1$ corresponding to the asymptotically optimal Thompson sampling (TS) agent
\citep{thompson1933likelihood,agrawal2012analysis}, see \cref{sec:bandit_details_skills} for details of the skill levels.

\vspace{-0.5em}
\paragraph{Pretraining.}
Given a trajectory predictor, the purpose of a prompt is to induce a \emph{skill} level, not the arm probabilities $v$'s, 
to solve new problems of the bandit class with unknown $v$'s.
Thus, we design each pretraining trajectory by concatenating a prompt segment and a rollout segment, separated by a special token (see \cref{fig:bandit} and \cref{alg:bandit_pretrain}). The prompt segment consists of 8 actions and rewards by an agent of a randomly sampled skill. The rollout segment consists of 300 actions and rewards with the same skill but $v$'s redrawn.
Actions in the prompt segment are thus informative about the skill only. 
Since actions and rewards are binary and interleaved, the sequence can be represented as a binary string.  Trajectories from this DG are used to pretrain neural predictors, with log-loss measured only on the action tokens. 
We also obtain a Bayes predictor (see \cref{sec:bandit_details_bayes}).
The predictor then acts stochastically according to the predicted probabilities at each action step.
Since the agent mixture contains the expert TS agent, prompting for maximal return can be regarded as IMD.  

\vspace{-0.5em}
\paragraph{Prompting methods.}
We prompt each predictor to maximize the expected total return in the rollout segment, from which we also estimate the instantaneous regret $\mathbb{E}[\max_{a}v_a - v_{a_t}]$ at each time step. 
We compare three prompting approaches.
First, we exhaustively search through all binary prompts on a predictor, running $N$ rollouts to estimate the return, giving $\hat{s}$.
Using a large $N$, we found four equivalent theoretically optimal prompts $s^*$ shown in \cref{fig:bandit}. Surprisingly, they show little exploration but strong exploitation to the more rewarding arm. We address this pattern in \cref{sec:bandit_details_optimal_prompts}, using knowledge of the pretraining distribution. 
Second, we prompt the model by demonstrations from the expert TS agent ($\tau=1$).
Finally, we handcraft several intuitive prompts: a ``maximally exploring'', a ``maximally exploiting'', and a ``heuristic'' prompt handcrafted by an author (\cref{fig:bandit}).
The latter two approaches mimic common methods of in-context learning on LLMs.

\vspace{-0.5em}
\paragraph{Results.}
\Cref{fig:bandit}(bottom left) shows the reliability (proportion correct) of the empirical $\hat s$ found on the Bayes and LSTM predictors. It increases very slowly as $N$ increases, reaching around $0.8$ when using as many as $10^5$ rollouts to estimate the expected return of each prompt.
\Cref{sec:bandit_interp} shows that these suboptimal prompts can lead to largely inconsistent interpretations in terms of a classical behavioral metric~\citep{bruner1957perceptual,nowak1993strategy}.
The performances of the three prompt types are in \cref{fig:bandit}(bottom). Typical prompts from the TS expert and the ``heuristic'' prompt both yield lower return than the theoretically and empirically optimal prompts, and are even worse than the ``max. exploiting'' prompt. The ``max. exploring'' prompt gives the lowest return. These results show that the common practice of prompting by expert demonstrations does not induce expert behavior as effectively as the optimal prompt, or even the ``maximally exploiting'', which is surprising given that the latter does not encourage exploration at all. 
Therefore, even using uniform priors over the latent and ensuring the task is IMD gives unintuitive prompts. 
The effectiveness of ``wrong'' examples has been reported~\citep{min2022rethinking,yang2024prompts} on LLMs, which could be due to subtle interactions between latent factors in the unknown pretraining distribution, as suggested in \cref{sec:bandit_details_optimal_prompts}.
\Cref{sec:bandit_typical_prompt} shows that the optimal prompt uses 1/4 of the length of expert demonstrations for the same expected return, demonstrating again its desirable high efficiency.

\subsection{Real LLMs}\label{sec:real_llms}

We adapt our framework above to study prompts on LLMs, using GPT-2~\citep{radford2019language}, Gemma-3~\citep{team2025gemma}, and Gemini 2.5 Pro~\citep{comanici2025gemini} and the IMDB reviews~\citep{imdb}.
Movie reviews may not exactly satisfy the conditional independence assumptions in our synthetic datasets: people do not always start writing a movie review with a statement or summary that reveals the sentiment. In addition, longer prompts can become less consistent with reviews for specific films or genres; for example, if a prompt mentions quality of animation, this would be inconsistent with non-animated film reviews, even if the sentiments are both positive. Nonetheless, IMDB dataset offers a handle to a single latent variable (sentiment) governing the review text, and has sufficient linguistic diversity compared to other synthetic languages dataset~\citep[e.g.,][]{xieexplanation,yuan2021synthbio}.

The Bayes predictor for movie reviews is unknown, and extrapolating our findings above to predict the LLM results is also hard; instead, our goal is to explore properties of optimized prompts. 
First, we finetune GPT-2 and Gemma-3 on the training set of IMDB, and ask Gemini 2.5 Pro to come up with a large number of prompts of variable lengths ($\approx \Lmax$) that can induce positive or negative (task $\tau$) movie reviews. As a sanity check, we verified that these prompts do modulate the log-losses of predicting reviews of either sentiment as intended. Then, we evaluate the conditional log-loss of the reviews with variable length cut-offs ($8\le T \le 512$) and dataset size ($250 \le N \le 25000$) given each of the prompts. This allows us to explore many basic properties of optimized prompts as before. The results are presented in \cref{sec:real_llms_details}. In summary, we found that the prompt sentiment is not always consistent with the intended sentiment of the reviews, and it also depends on the pretrained model. As discovered in binary sequences, the reliability of identifying the optimal prompt is very low at $N=250$, while $N\approx 1700$ is needed for a 50\% proportion correct. Finally, the optimal prompt length tends to be very short, as discovered before in the literature.

\section{CONCLUSIONS}

Through a large set of carefully designed experiments, we revealed fundamental patterns of optimal prompts in conceptually simple binary sequences.
Our well-controlled experiments provided ample examples of binary optimal prompts and visualizations, and revealed interesting fundamental properties of optimal prompts on LLMs. This work thus presented a plausible candidate theory that (partially) accounts for the surprising prompts seen on frontier models. 

Since the pretraining distribution is often unknown, interpreting optimal prompts in terms of the task alone is hard: these prompts may be atypical sequences under the task distribution, hard to find reliably given a finite dataset, and may vary with complex patterns depending on the prompt setup.
Further, demonstrations from the task, as often used for in-context learning, are less efficient at inducing the intended behavior than some unintuitive prompts.
These results are robust across a large number of predictors and their instances, including the idealized Bayes predictors. This indicates that some of the issues will persist at any scale, including that of frontier models, and thus contribute to previously observed difficulties in interacting with those models.
We provide a practical guidance in \cref{sec:practical_guidance}, elaborate on how our results relate to previous findings on prompting LLMs in \cref{sec:related_findings}, and discuss detailed limitations and potential future work in \cref{sec:limitations}.

\subsubsection*{Acknowledgements}
We thank the following collaborators for useful discussions and valuable feedback: Laurent Orseau, Grégoire Delétang, Peter Dayan, Noémi Éltető, and Mark Rowland.

\printbibliography

\section*{Checklist}

\begin{enumerate}

  \item For all models and algorithms presented, check if you include:
  \begin{enumerate}
    \item A clear description of the mathematical setting, assumptions, algorithm, and/or model. [Yes/No/Not Applicable]
    \item An analysis of the properties and complexity (time, space, sample size) of any algorithm. [Yes/No/Not Applicable]
    \item (Optional) Anonymized source code, with specification of all dependencies, including external libraries. [Yes/No/Not Applicable]
  \end{enumerate}

  \item For any theoretical claim, check if you include:
  \begin{enumerate}
    \item Statements of the full set of assumptions of all theoretical results. [Yes]. \cref{thm:info_formal}
    \item Complete proofs of all theoretical results. [Yes]
    \item Clear explanations of any assumptions. [Yes]     
  \end{enumerate}

  \item For all figures and tables that present empirical results, check if you include:
  \begin{enumerate}
    \item The code, data, and instructions needed to reproduce the main experimental results (either in the supplemental material or as a URL). [Yes, the experiments are clearly described; also see \cref{alg:bandit_pretrain}. We will not provide code.] 
    \item All the training details (e.g., data splits, hyperparameters, how they were chosen). [Yes] \Cref{sec:real_llms_details}
    \item A clear definition of the specific measure or statistics and error bars (e.g., with respect to the random seed after running experiments multiple times). [Yes] In all figures we defined the error bars. Also in \cref{sec:compute}.
    \item A description of the computing infrastructure used. (e.g., type of GPUs, internal cluster, or cloud provider). [Yes] \Cref{sec:compute}
  \end{enumerate}

  \item If you are using existing assets (e.g., code, data, models) or curating/releasing new assets, check if you include:
  \begin{enumerate}
    \item Citations of the creator If your work uses existing assets. [Yes]
    \item The license information of the assets, if applicable. [Yes]
    \item New assets either in the supplemental material or as a URL, if applicable. [Not Applicable]
    \item Information about consent from data providers/curators. [Not Applicable]
    \item Discussion of sensible content if applicable, e.g., personally identifiable information or offensive content. [Not Applicable]
  \end{enumerate}

  \item If you used crowdsourcing or conducted research with human subjects, check if you include:
  \begin{enumerate}
    \item The full text of instructions given to participants and screenshots. [Not Applicable]
    \item Descriptions of potential participant risks, with links to Institutional Review Board (IRB) approvals if applicable. [Not Applicable]
    \item The estimated hourly wage paid to participants and the total amount spent on participant compensation. [Not Applicable]
  \end{enumerate}

\end{enumerate}

\clearpage
\appendix
\thispagestyle{empty}

\onecolumn
\aistatstitle{\thetitle:\\  Supplementary Materials}

\part{} %
\parttoc %

\section{METHOD DETAILS}

\subsection{Near-Bayes-Optimal Prediction Can Lead to ``Super-Bayes'' Performance on Certain Tasks.}\label{sec:super_bayes}
Previous work has shown that well-trained neural predictors behave almost identically to the Bayes predictor in terms of their predictions given data from the pretraining distribution $p$~\citep[e.g.][]{mikulik2020meta,genewein2023memory,grau2024learning}. However, these results do not imply a bound on the neural predictor by the same Bayes predictor given an intended task. Specifically, even though we have the optimality bound from positivity of the KL divergence:
$$
    \mathbb{E}_{s\sim p} [\log p_\theta(s)] \le \mathbb{E}_{s\sim p} [\log p(s)] \,,
$$
this does \emph{not} imply that the log-loss under a potentially different task distribution $q$ preserve the ordering:
$$
    \mathbb{E}_{s\sim q} [\log p_\theta(s)] \not\le \mathbb{E}_{s\sim q} [\log p(s)] \,.
$$
This holds even if $q\in\mathcal{M}_p$. This happens when $q$ favors specific regions of the sequence space where the neural predictor $p_\theta$ is accidentally better than the Bayes predictor $\pbayes$ that is optimal for the pretraining $p$ but not for the task $q$.

\subsection{Training Neural Predictors}\label{sec:neural_predictor_detail}

The neural predictors share the same overall structure, involving the following stages:
\begin{enumerate}
    \itemsep0em
    \item Map each binary tokens $x_t\in\{0,1\}$ to embeddings $e_t\in\mathbb{R}^h$ through $e_t = W_\text{emb} x_t$ where $W_\text{emb}\in\mathbb{R}^{h\times 2}$, and $h\in\mathbb{N}^+$ is the hidden size.
    \item Sequentially map $h_{1:T}$ through some neural architecture, called the \emph{torso}, such as an LSTM and a multi-head attention, to obtain hidden activations $u_t\in\mathbb{R}^h$ for each $t$.
    \item For each $t$, map $u_t$ through the fully connected MLP to $v_t\in\mathbb{R}^h$ that is usually found after the attention layer in a transformer block~\citep{vaswani2017attention}.
    \item For each $t$, map $v_t$ to output logits through a linear map.
\end{enumerate}
There is also a residual connection from step 2 to 3 and from 3 to 4.
The different neural architectures differ only by the torso. This maintains a flexible enough architecture for different tasks while controlling for the model complexity between different architectures.

For the torso, we use the following variants of recurrent networks and transformers:
\begin{enumerate}
    \itemsep0em
    \item Vanilla recurrent neural networks~\citep{elman1990finding};
    \item Long-short term memory (LSTM)~\citep{hochreiter1997long}, reported in main text;
    \item sLSTM~\citep{beck2024xlstm}
    \item Softmax-attention transformer (Transformer)~\citep{vaswani2017attention}, reported in main text;
    \item Linear transformer~\citep{katharopoulos2020transformers}
    \item Another variant of Linear transformer we refer to as Inner-product transformer (IP transformer)~\citep{li2020linear,shen2021efficient}
\end{enumerate}

We found that step 3 above is crucial for transformer architectures to perform some of the tasks, although this was not essential for LSTMs to perform well. We did not use normalization or dropout layers for simplicity. These hyperparameters for all networks and pretraining DGs are listed in \cref{tab:hypers}.
During pretraining, we made sure that the sequence length was long enough to avoid bad generalization over unseen sequence lengths~\citep{deletang2022neural,anil2022exploring} in downstream tasks.

\clearpage

\begin{table}
\centering
\footnotesize
\caption{Hyperparameters for neural predictors. $^*$For vanilla RNN, we train 5 million steps. \label{tab:hypers}}
\begin{tabular}{l|ccc|c|ccc}
               & \multicolumn{3}{c|}{\textbf{Common parameters}} & \textbf{Recurrent}              & \multicolumn{3}{c}{\textbf{Transformer}}        \\
\textbf{Pretraining DG} & hidden size $h$             & \# steps     &   $T$    & learning rate     & \# head & \# layer & learning rate     \\
\hline
BernMix        & 128             & 300k      &  180    & $1\times 10^{-4}$         & 1       & 1        & $1\times 10^{-4}$         \\
BetaBern       & 128             & 500k$^*$  &  180    & $1\times 10^{-4}$         & 1       & 1        & $1\times 10^{-4}$         \\
Switching      & 128             & 3 m        &  180    & $3\times 10^{-5}$         & 8       & 1        & $3\times 10^{-5}$ \\
Bandit         & 256             & 1 m        &  630    & $1\times 10^{-4}$         & 8       & 2        & $3\times 10^{-5}$
\end{tabular}
\end{table}

\subsection{Prompt Search Method}\label{sec:prompt_search}

On the CIB-DG experiments in \cref{sec:CIB-DG_methods}, both the prompt $s$ and the sequence(-to-predict) $x$ are permutation invariant under the Bayes predictor $\pbayes$, which means that the counts of \ZEROs and \ONEs in \cref{eq:counts} form sufficient summary of the prompt and the sequence under $\pbayes$. This drastically reduces the search space of the theoretically optimal prompt on the Bayes predictor. Also, by using the equivalent binomial distribution defined over the counts, the summation in \cref{eq:log_loss_L} can also be reduced for the Bayes predictor. This makes searching for $s^*(\pbayes, \cdot)$ and $\hat s(\pbayes, \cdot, \cdot)$ very efficient. However, these savings do not apply to neural predictors, so we searched through all possible prompts and sequences.

The DGs in bandit task (\cref{sec:bandit}) and the switching task (\cref{sec:switching}) have no nontrivial symmetry in the space of all prompts and all sequences/rollouts, so we exhaustively searched through all possible prompts and sequences/rollouts to find $s^*$. This is also done on finite datasets, we also searched through all possible prompts on the given task dataset to find $\hat s$.

\subsection{Information-Theoretic Justification}\label{sec:info}
We show that prompt optimization over the objective in \cref{eq:log_loss_L} is equivalent to maximizing an information-theoretic objective relating the prompt and the sequence-to-predict for the IMD case $q\in\mathbb{M}_p$.
In this subsection only, to reduce notational clutter and avoid specifying lengths $T$ and $L$, we temporarily define $\vx$ as the sequence and $\vs$ as the prompt. The latent variable $\tau$ is not restricted to a scalar. We \emph{do} still assume that $\vx$ is conditionally independent of $\vs$ given $\tau$.
As before, the pretraining DG is then $p(\vx):=\int p_\tau(\tau)p_{x|\tau}(\vx|\tau) \ud \tau$.

Denote the pretrained predictor by $p_{(\cdot)}(\vx|\vs)$, which can be the Bayes predictor $\pbayes(\vx|\vs)$ or a pretrained neural predictor $p_\theta(\vx|\vs)$.
In addition, define the \emph{prompting strategy} as $\nu(\vs|\tau)$ mapping from a given $\tau$ to a distribution over the prompt. This is a strategy because this maps our intended task $\tau$ to a prompt sequence that is passed to the sequence predictor. We also allow this strategy to be stochastic. This strategy leads to a prompt policy-augmented joint distribution over the prompt and the sequence: 
\begin{equation}\label{eq:joint_nu_vec}
p^\nu(\vs, \vx):= \int p_\tau(\tau)\nu(\vs|\tau)p_{x|\tau}(\vx|\tau)\ud \tau\,.
\end{equation}
The prompt distribution $\nu$ need not be the same as $p_{x|\tau}$, which would suggest using samples from 
$p_{x|\tau}$ as the prompt. By construction, the augmentation leaves the marginal over $\vx$ intact in the absence of any prompt: $p^\nu(\vx)=p(\vx)$.

What would be a best prompting strategy? First, there should be high mutual information between $\vx$ and $\vs$ under $p^\nu$, so that varying $\vs$ effectively manipulates the distribution over $\vx$. 
However, this condition alone is not sufficient.
Consider the strategy that maps some partition over the space of the task variable $\tau$ to distinct sequences of $\vs$ unrelated to the pretraining distribution.
This prompt strategy can achieve a high mutual information, but $\vs$ completely ignores any statistical regularities in $p(\vx)$ baked into the predictor $p_{(\cdot)}$. Thus, this strategy is unlikely to steer the predictor $p_{(\cdot)}$ in the same way as it conditions the pretraining DG (through the task latent $\tau$).

To fix this, the best prompting strategy must ensure that the prompt must condition the predictor in a similar way to how it conditions the pretraining DG, i.e.\ $p_{(\cdot)}(\vx|\vs)$ must be close to $p(\vx|\vs)$, or aligned with the predictor. Combining these ideas together we arrive at the following objective.
\begin{definition}
The predictor-aligned mutual information (PAMI) is defined as
\begin{equation}\label{eq:mami}
\MAMI(\vs, \vx):= \MI_{p^\nu}(\vs; \vx) - \mathbb{E}_{\vs\sim p^\nu}[\KL[p^\nu(\vx|\vs)\|p_{(\cdot)}(\vx|\vs)]]\,.
\end{equation}
\end{definition}
This objective trades off the specification of $\vx$ through $\vs$ under $p^\nu$ and the alignment between conditioning on $p^\nu$ and on $p_{(\cdot)}$. Though heuristically defined, the optimal strategy under PAMI is one that optimizes the predictive log-likelihood under all specified (IMD) tasks.

\begin{proposition}\label{thm:info_formal}
The deterministic prompt distribution $\nu(\vs|\tau)=\delta_{\vs^*(\tau)}$ centered at 
$$
\vs^*(\tau):=\argmax_\vs \sum_\vx p(\vx|\tau)\log p_{(\cdot)}(\vx|\vs)
$$
for all $\tau$ maximizes $\MAMI(\vs;\vx)$.
\end{proposition}
\begin{proof}
PAMI can be rewritten into a form similar to the conventional mutual information:
\begin{align*}
\MAMI(\vs, \vx) &=\sum_{\vs, \vx} p^\nu(\vs,\vx)\log\frac{p^\nu(\vx|\vs)}{p^\nu(\vx)} - \sum_{\vs,\vx}p^\nu(\vs,\vx)\log\frac{p^\nu(\vx|\vs)}{p_{(\cdot)}(\vx|\vs)} =\sum_{\vs, \vx} p^\nu(\vs,\vx)\log\frac{p_{(\cdot)}(\vx|\vs)}{p^\nu(\vx)} \\
&=H[p(\vx)] + \sum_{\vs, \vx}  p^\nu(\vs,\vx)\log p_{(\cdot)}(\vx|\vs).\numberthis\label{eq:mami_mi_form}
\end{align*}
The first term in \cref{eq:mami_mi_form} is independent of $\nu$, so we only need to show that $\delta_{\vx^*(\tau)}$ maximizes the second term. This term expands to
$$
\int p_\tau(\tau)\sum_\vs \nu(\vs|\tau)\sum_\vx p_{x|\tau}(\vx|\tau) \log p_{(\cdot)}(\vx|\vs) \ud \tau\,.
$$
The Dirac delta measure $\delta_{\vx^*(\tau)}$ assigns all mass to $\vs^*(\tau)$ that maximizes $\sum_\vx p_{x|\tau}(\vx|\tau)\log p_{(\cdot)}(\vx|\vs)$ for each given $\tau$. This holds for all $\tau$, and thus $\delta_{\vs^*(\tau)}$ maximizes the second term of \cref{eq:mami_mi_form}.
\end{proof}

Now it may be obvious that PAMI is simply a re-arrangement of the log-likelihood,
the decomposition of prompting into Informativeness (MI) and Alignment (KL) perfectly encapsulates the reason by prompting by intuition usually fails: humans tend to only optimize for the mutual information, completely ignoring the KL penalty imposed by the model's unknown pretraining distribution. An optimal prompt might look like gibberish to a human, but it perfectly hacks the pretraining priors to reduce that KL penalty.

\subsection{Optimal Prompt and Criterion for an Empirically Optimal Prompt to be Correct}\label{sec:proportion_correct}
For each prompt setup in our experiment, we found multiple theoretical $s^*$'s that give the same expected log-loss \eqref{eq:log_loss_L}.
For CIB-DGs, we found empirically that the \emph{counts} in the theoretical $s^*$ are unique.
As such, an empirically optimized prompt for CIB-DG experiments on neural predictors is deemed correct if it has the same counts (\cref{eq:counts}) as the theoretically optimal prompt. 
For other DGs (bandit and switching DGs), correctness requires an exact match between $\hat s$ and any one of the theoretical $s^*$'s.

To estimate the proportion correct, we trained a large number of networks with different random seeds, found the empirically optimal prompt $\hat s$ for each network using $\DqN$ drawn with another different seed, and then calculated the empirical ratio of correct $\hat s$'s out of all prompts found on all networks.

\subsubsection{Limitations.}
This definition of optimal prompt uses the Bayes optimal predictor, not the trained neural predictor. In practice, neural predictors introduce further complications when used as a reference model to define optimality, because they can vary depending on many pretraining hyper-parameters (e.g.,\ neural architecture, learning rate, dataset size, any post-training, etc.). As such, given a specific neural predictor $p_\theta$, the optimal prompt for a task DG (evaluated under infinitely large dataset) $s^*_{1:L}(p_\theta,q)$ (similarly defined in \eqref{eq:optimal_prompt_L}) may differ from the optimal prompt $s^*_{1:L}(\pbayes,q)$ under the Bayes predictor, and also from other neural predictor instances.

Another consequence of using a model-dependent optimality is that the optimal log-loss under a neural predictor can be lower than that of the Bayes predictor
$$
\max_s \mathbb{E}_{q(x)}[\log p_\theta(x|s)] > \max_s \mathbb{E}_{q(x)}[\log \pbayes(x|s)].
$$
This notion of practical optimality thus depends on the network instance, with nuances beyond the scope of the Bayesian meta-learning perspective in this work. Future work that takes a more architectural perspective should investigate this important issue further.

\subsection{The Loss ``Landscape'' of Bayes Predictor}\label{sec:loss_landscape}
Why is it difficult to identify the optimal prompt with a finite task dataset? We hypothesize that this is because there are many suboptimal prompts that are only slightly worse compared to the optimal one in terms of the expected log-loss \eqref{eq:log_loss_L}. To test this, we take the Bayes predictor for each prompt setup, compute expected log-loss given all possible prompts with length $L\le\Lmax$. We also subtract this log-loss with the best possible log-loss, giving the Kullback--Leibler divergence
\begin{equation}\label{eq:kl_divergence}
\text{KL}[q(\cdot)\| p(\cdot|s_{1:L})] := \mathcal{L}(p_{(\cdot)}, q, s_{1:L}) - \mathcal{L}(q, q, s_{1:L}).
\end{equation}
If many prompts yield KL divergences close to the optimal KL divergence (under $s^*$), then it is less likely that the order of these prompts is still preserved when evaluated under a finite dataset, and even less so on approximate neural predictors. 

Since the prompts are discrete with no obvious order, we show the KL divergences of all prompts, sorted in increasing order. Each prompt is then associated with a rank. To see if the optimal prompt has a distinctively smaller KL divergence than other prompts, we plot the KL divergences against the rank of the prompt. 
For CIB-DGs, the prompts expressed in the counts $(S_0(s_{1:L}), S_1(s_{1:L}))$ are ordered, so we show the prompt rank by their counts, and also plot the loss ``landscape'': the KL divergence as the color of each dot representing the counts. We show these results for each of the four CIB-DG experiments below.

\textbf{Limitations} The ``loss-landscape'' is a loosely defined term to describe the visualization method for understanding the reliability of finding the optimal prompt. Importantly, it is a \emph{post hoc} technique that cannot be used to \emph{predict} the reliability given a predictor and a task DG, before running any evaluations. The ability to predict is more challenging due to the nonlinear log-loss function and discrete  domain. For example, the problem of integer  quadratic programming is NP-complete~\citep{vavasis1990quadratic}, let alone predicting its reliability under a noisy dataset. 

\clearpage
\section{ADDITIONAL EXPERIMENTAL RESULTS}\label{sec:additional_results}

For each of pretraining CIB-DGs and the random switching DG, we trained 30 networks with different random seeds for the weight initialization and minibatch sampling.
For the bandit DG, we trained 100 networks. Below, we report, for each pretraining DG, the estimated KL divergence between $p_x$ and $p_\theta$ using samples from $p_x$:
\begin{equation}\label{eq:pretrain_kl}
\frac{1}{N}\sum_{i=1}^N\sum_{t=1}^T \log \left(\frac{\pbayes(x_t|x_{1:t-1})}{p_\theta(x_t|x_{1:t-1})} \right).
\end{equation}
The training sequence length and minibatch size for each pretraining DG $p$ are as shown in \cref{tab:hypers}. In most cases, the network achieved near-zero KL divergence, consistent with previous findings~\citep{mikulik2020meta,genewein2023memory,grau2024learning}.
Under the bandit DG, we use an approximate Bayes predictor described in \cref{sec:bandit_details_bayes}. We see that the LSTM and Transformer outperform this predictor by a small but statistically significant amount.

\begin{figure}[ht]
    \centering
    \includegraphics[height=0.163\textheight]{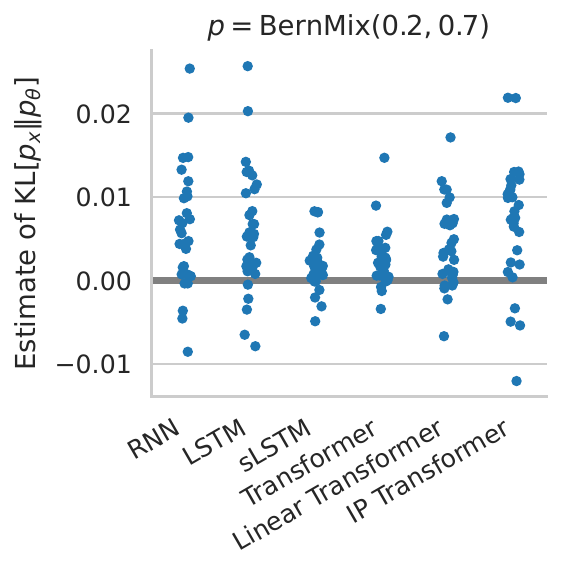}\hfill
    \includegraphics[height=0.16\textheight]{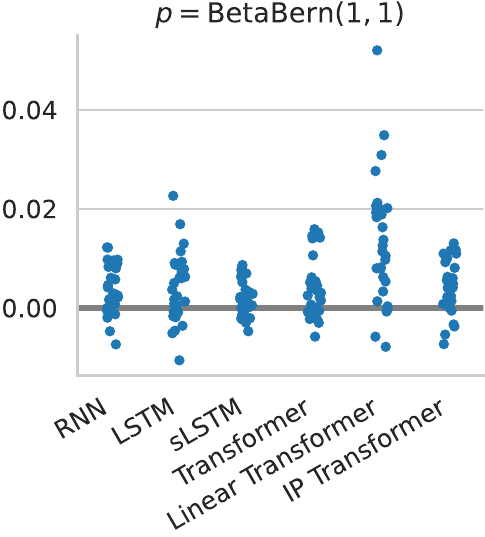}\hfill
    \includegraphics[height=0.16\textheight]{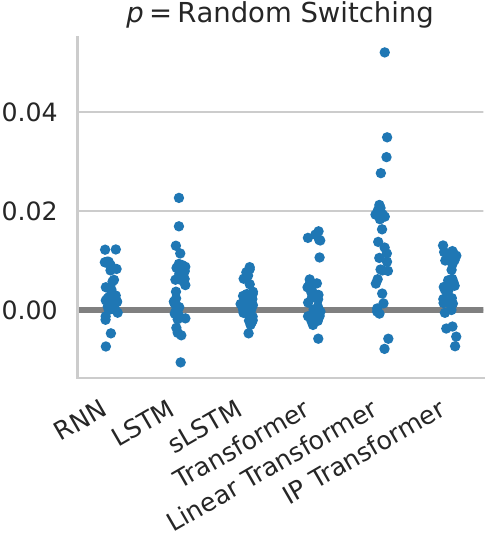}\hfill
    \includegraphics[height=0.16\textheight]{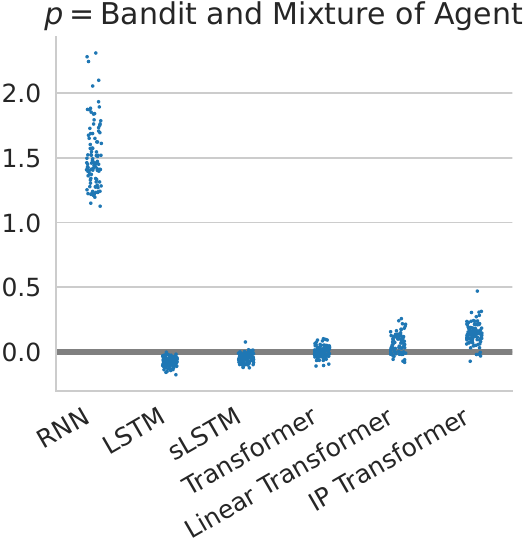}
    \caption{Final pretraining loss under various DGs $p$ and neural predictors. Each dot is an Monte Carlo estimate of the KL divergence through \cref{eq:pretrain_kl} in the last training step, summed (not averaged) over entire training sequences (180 time steps for CIB-DGs and 630 for Bandit experiments; see \cref{tab:hypers}) , then averaged over a random training minibatch.    
    }
    \label{fig:pretrain_kl}
\end{figure}

\subsection{Computational Scale and Resources}\label{sec:compute}

All experiments were implemented using JAX (Apache 2.0)~\citep{jax2018github,deepmind2020jax}, NumPy~\citep[modified BSD,][]{harris2020array} and Haiku~\citep[Apache 2.0,][]{haiku2020github}.
We used in-house computational infrastructure comprising P100 and V100 NVIDIA GPUs, TPUv2, and TPUv3.
In total, we pretrained 3 (2 CIB-DGs + random switching) $\times$ 6 (architectures) $\times$ 30 (seeds) + 1 (Bandit DG) $\times$ 6 (architectures) $\times$ 100 (seeds) = 1140 neural predictors. Each pretraining run takes from several hours to 1 week, with the longest training being on RNN networks to get a low log-loss.

\paragraph{Searching for the theoretically optimal $s^*$.} This required minimal ($<10$ seconds) on the Bayes predictors of CIB-DGs, as we only need to enumerate over the counts of \ZEROs and \ONEs. For the switching task, we searched through all possible prompt sequences of up to length $\Lmax=15$, and evaluated the expectation \cref{eq:log_loss_L} over all possible sequences of length $T\in\{10, 30\}$. For $T=30$, each run took less than 3 days. For the bandit task, we estimated the log-loss using $10^5$ rollouts for all prompts of length 16, and confirmed the best out of the top 20 prompts using 10 million rollouts.

\paragraph{Searching for the empirically optimal prompt.}
For Bayes predictors, on CIB-DG experiments, we searched for the best prompt 2 (pretraining DGs) $\times$ 2 (task DGs) $\times$ 3 ($\Lmax\in\{5, 10, 15\}$) $\times$ 5 ($T\in\{1,3,10,30,100\}$) $\times$ 4 ($N\in\{10, 10^2, 10^3,10^4\}$) $\times$ 1000 (seeds for $\DqN$) $=240000$ times. These are carried out fast over the counts, so we do not count them as a significant computation. 
On the switching task, we searched 4 (task DGs) $\times$ 1 ($\Lmax\in\{15\}$) $\times$ 2 ($T\in\{10,30\}$) $\times$ 4 ($N\in\{10, 10^2, 10^3,10^4\}$) $\times$ 250 (seeds for $\DqN$) $=8000$ times over binary sequences. 
On the bandit task, we
searched 5 ($T\in\{1,3,10,30,100\}$) $\times$ 5 ($N\in\{10, 10^2, 10^3,10^4,10^5\}$) $\times$ 100 (seeds for $\DqN$) $=2500$ times. Each of these took less than a day.

For neural predictors, on CIB-DG experiments, this amounts to 360 (model instances) $\times $ 3 ($\Lmax\in\{5, 10, 15\}$) $\times$ 5 ($T\in\{1,3,10,30,100\}$) $\times$ 4 ($N\in\{10, 10^2, 10^3,10^4\}$)$= 21600$ exhaustive searches on binary sequences (not counts), using a different seed to draw $\DqN$ for each network instance. On the switching experiment, this is 
180 (model instances) $\times $ 1 ($\Lmax\in\{15\}$) $\times$ 2 ($T\in\{10,30\}$) $\times$ 4 ($N\in\{10, 10^2, 10^3,10^4\}$)$= 1440$ exhaustive searches on binary sequences. Each of these runs took less than 3 days.
On the bandit tasks, this is 600 (model instances) $\times$ 1 ($\Lmax\in\{8\}$) $\times$ 5 ($T\in\{5,15,50,150,300\}$) $\times$ 5 ($N\in\{10, 10^2, 10^3,10^4,10^5\}$)$= 15000$ exhaustive searches on binary sequences.
Thus, our results are derived from $>48000$ optimal prompts sequences, all of which are numerically guaranteed to be optimal under the prompt setup. For RNNs and transformers that can be cast as RNNs, this took less than a day. For the softmax transformer, without using any optimized attention computation, each run took around 2 weeks.
We estimate the compute used during the entire project to be around 30 times the compute used to generate the final results.

\subsection{Pretraining on \texorpdfstring{$\BernMix(0.2, 0.7)$}{Bern(0.2, 0.7)}, Prompting towards \texorpdfstring{$\Bern(0.7)$}{Bern(0.7)} (IMD)}\label{sec:bernmix_bern_0.7_supp}

\subsubsection{Proportion Correct}\label{sec:bernmix_bern_0.7_correct}

\Cref{fig:sensitivity_supp} shows the proportion correct results for the task DG $q=\Bern(0.7)$, so $q\in\mathcal{M}_p$. The recurrent networks show similar behaviors to the Bayes predictor, and the Transformers have worse proportion correct at $T=100$ even for a short prompt with length up to $\Lmax$. The proportion correct is lower for longer prompts for all predictors, likely because the prompt of all \ONEs requires a perfect match on each token and is more difficult to identify exactly when the maximum prompt length is longer. 

\begin{figure}[ht] %
    \centering
    \includegraphics[width=\textwidth]{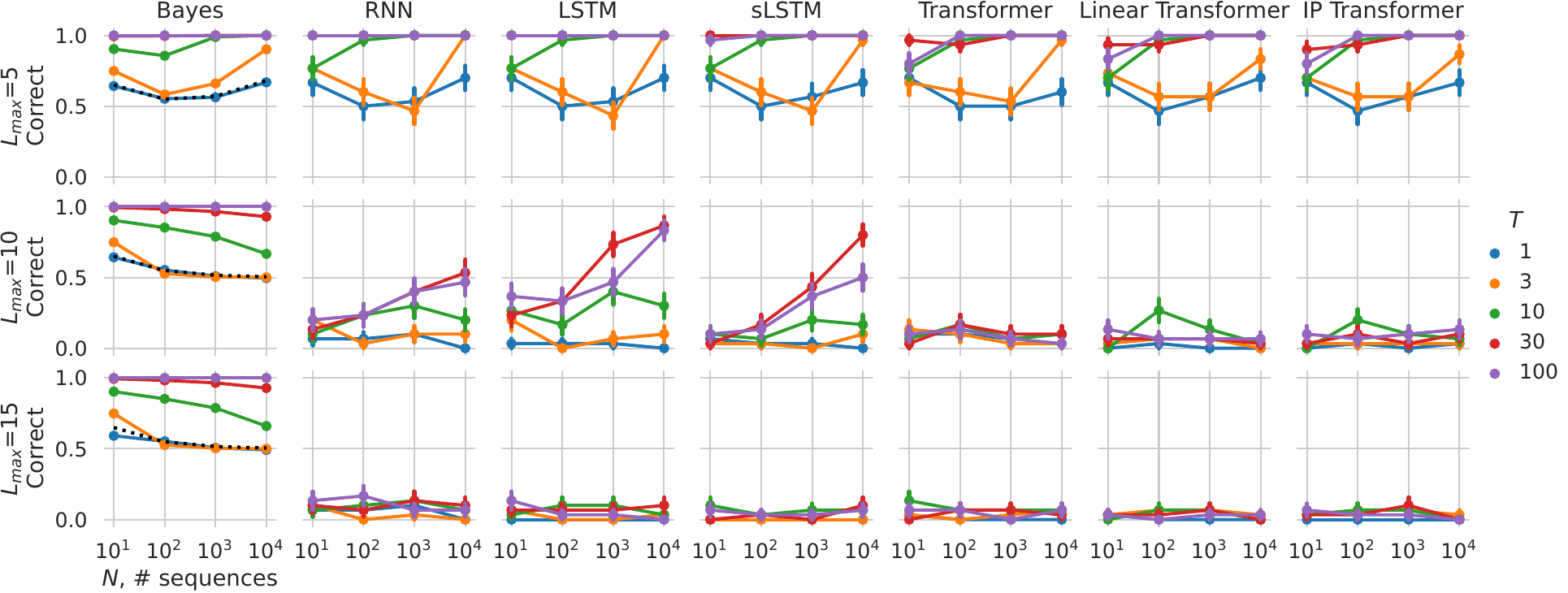}
    \caption{
    The proportion correct of the empirically optimal prompt for 
    $p=\BernMix(0.2, 0.7)$ and $q=\Bern(0.7)$.
    Same as \cref{fig:sensitivity} but with more prompting setups.}
    \label{fig:sensitivity_supp}
\end{figure}

\begin{table}[ht] %
    \centering
    \caption{Prompts of length up to $\Lmax=5$ that induce top 10 values of the posterior belief $p(x_1|s)$ under the pretraining DG $\BernMix(0.2, 0.7)$. \label{tab:tau_hat}}
    \begin{tabular}{l|*{11}{c}}
        \multicolumn{1}{c|}{} & \multicolumn{11}{c}{Prompt counts} \\
        \hline
        $S_0(s)$ & 0 & 0 & 1 & 0 & 1 & 0 & 2 & 1 & 0 & 2 & 1 \\
        $S_1(s)$ & 5 & 4 & 4 & 3 & 3 & 2 & 3 & 2 & 1 & 2 & 1 \\
        $p(x_1=1|s)$ & 0.699 & 0.697 & 0.691 & 0.689 & 0.671 & 0.662 & 0.629 & 0.611 & 0.589 & 0.516 & 0.484 \\
    \end{tabular}
\end{table}

\subsubsection{Non-Monotonicity of Proportion Correct versus \texorpdfstring{$N$}{N}}\label{sec:bernmix_bern_0.7_nonmonotonic}

Let us gain some insights into this phenomenon by considering the case of $T=1$. Here, we can compute theoretically the probability of an empirical prompt being correct and predict this trend as the black dotted line. The short reason is that the task dataset may have an \emph{empirical} ratio of \ONEs below 0.7, which happens with nonzero probability even though the task DG is $\Bern(0.7)$.
When this happens, this dataset may be better explained by a mixture of Bernoulli that has nonzero weight on the component with $\tau_1=0.2$, giving a prediction that is close to the empirical ratio. Such a mixing weight can be induced by a prompt other than the all-\ONE prompt; see \cref{tab:tau_hat} for example prompts with lengths up to $\Lmax=5$. Note that some of the prompt counts in \cref{tab:tau_hat} appear in the clusters shown in \cref{fig:sensitivity}(right, blue dots). At $N=10$, if the empirical mean in the task dataset takes on values of 0.6 or 0.5, the best prompts to induce such biases have counts, respectively, (1 \ZERO \& 2 \ONEs) or (1  \ZERO \& 1 \ONE). Such datasets are responsible for the clusters of empirical prompts.

More precisely, for $T=1$ we can compute the probability that $\hat{s}(\pbayes, \DqN)= s^*(\pbayes,q)$ given a randomly drawn task dataset $\DqN$ from $q=\Bern(0.7)$. 
Recall that the log-loss on the first token following a prompt $s$ is 
\begin{equation}\label{eq:length_1_seq_likelihood}
\begin{gathered}
\hat{\mathcal{L}}(\pbayes, \DqN, s) = -N\hat\tau(\DqN) \log(\bar\tau(s)) - N(1-\hat\tau(\DqN)\log(1-\bar\tau(s))\,,~\text{where}\\
\hat\tau(\DqN) := \frac{1}{N}\sum_{n=1}^N x_1^n\,,  \quad \bar\tau(s):=\pbayes(x_1=1|s)=(1-w_L(s))\tau_1 + w_L(s) \tau_2\,,
\end{gathered}
\end{equation}
and $w_L(s)$ is given in \cref{thm:bernoulli-mixture}.
For prompts of length up to $\Lmax$, it is easy to see that $\hat{s}(\pbayes, \DqN)=s^*(\pbayes,q)$ if and only if, under the Bayes predictor $\pbayes$, the optimal prompt $s^*$ with $\Lmax$ \ONEs gives a lower log-loss than the prompt $s^+$ that has $0$ \ZERO and $(\Lmax-1)$ \ONEs (see \cref{tab:tau_hat} for an example when $\Lmax=5$). Then we have, 
$$
\mathbb{P}(\hat{s}=s^*) = \mathbb{P}( \hat{\mathcal{L}}(\pbayes, \DqN, s^*)< \hat{\mathcal{L}}(\pbayes, \DqN, s^+)).
$$ 
Substituting in \cref{eq:length_1_seq_likelihood} and after some manipulation, we get
\begin{equation*}
\begin{gathered}
\mathbb{P}(\hat{s}=s^*) =
    \mathbb{P}\left(N\hat\tau(\DqN) > \kappa\right)\,,~\text{where}~
    \kappa := N\log\left(\frac{1-\tau^+}{1-\tau^*}\right)
        \bigg/
        \log\left(\frac{\tau^*(1-\tau^+)}{\tau^+(1-\tau^*)}
    \right)
\end{gathered}
\end{equation*}
Noting that $N\hat\tau(\DqN)$ is a binomial distribution $\text{Binom}(0.7, N)$, we can easily find $N\hat\tau(\DqN)$ using its cumulative distribution function. For sequence length $T>1$, the loss becomes more complicated, so is its dependence on $\bar\tau$.

\def\landscapelegend{The leftmost column shows the KL divergence of each prompt sorted in increasing order against the prompt rank (sort indices, or argsort), with lower rank meaning lower KL divergence \eqref{eq:kl_divergence}. The other columns show KL divergence of each individual prompt. The prompts here are all expressed by their counts. See \cref{sec:loss_landscape} for a detailed explanation.
}
\begin{figure}[ht] %
    \centering
    \includegraphics[height=0.25\textheight]{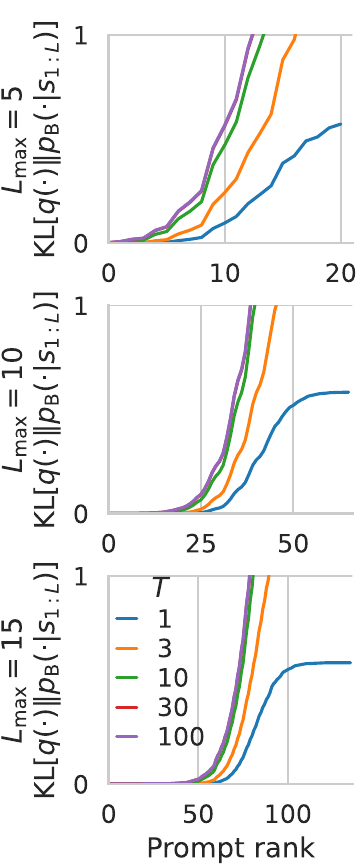}
    \includegraphics[height=0.25\textheight]{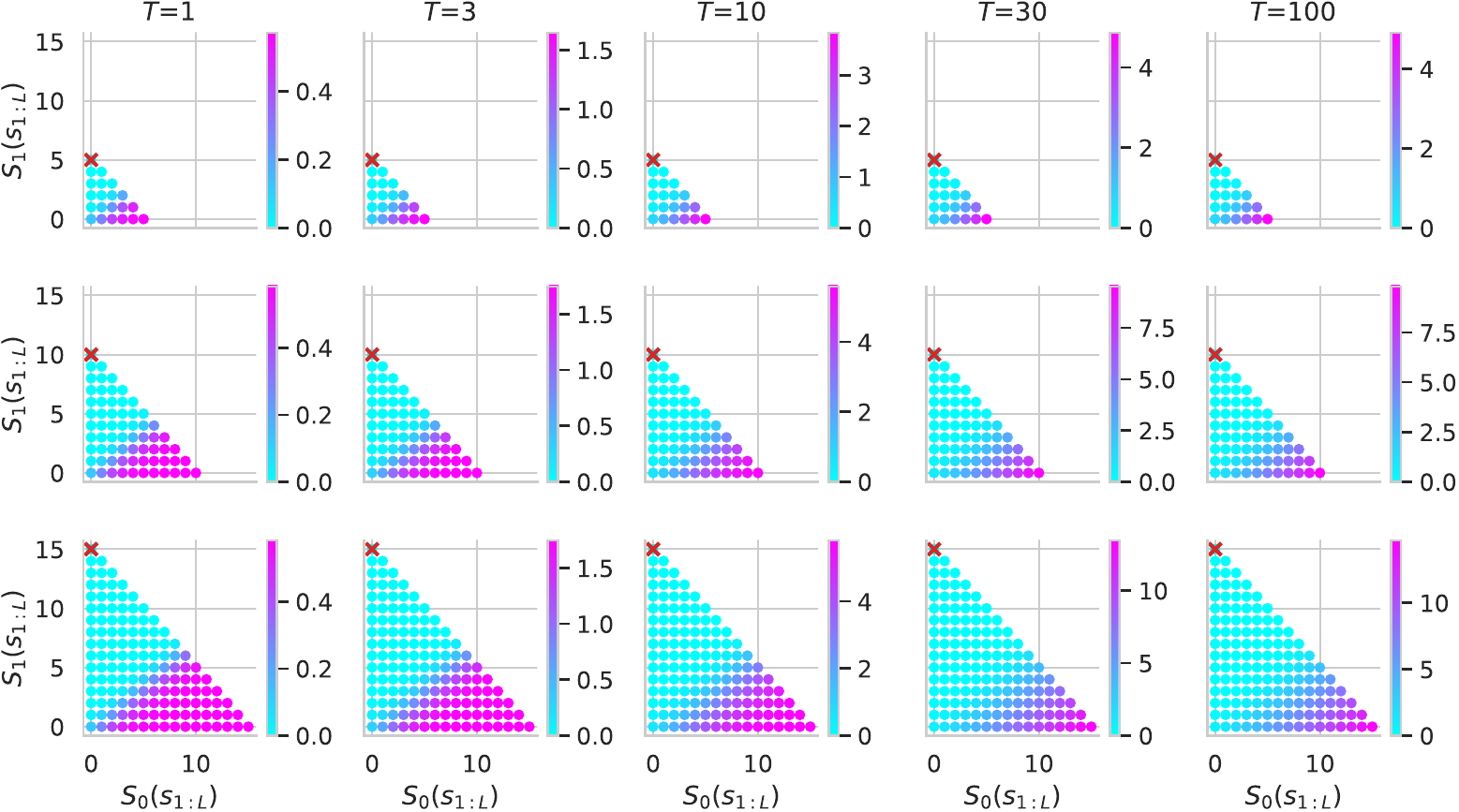}
    \caption{
    The loss ``landscape'' of $p=\BernMix(0.2, 0.7)$ and $q=\Bern(0.7)$. \landscapelegend
    \label{fig:sensitivity_landscape}
    }
\end{figure}

\subsubsection{Loss ``Landscape''}\label{sec:bernmix_bern_0.7_landscape}

We show the KL divergence of all possible prompts in two ways in \cref{fig:sensitivity_landscape}. On the leftmost column, we order the prompts (expressed in counts) according to their values of the KL divergence, producing a rank plot of in ascending order. The best prompt is at rank 0, the second best prompt is at rank 1, etc.
We see that there is a relatively flat region next to the best prompt with $rank 0$, suggesting that it would be more difficult to distinguish the best few prompts.
As $\Lmax$ increases and $T$ decreases, the flat region expands towards the right, meaning that there are more close-to-optimal prompts. This trend is consistent with the results shown in \cref{fig:sensitivity_supp}. The other columns of \cref{fig:sensitivity_landscape} show KL divergence of all possible prompts with length less than $\Lmax$. It is evident that the many prompts around the optimal prompt have very similar KL divergences close to zero. 

\subsubsection{Distribution of Empirically Optimal Prompts}\label{sec:sensitivity_emp_dist}

\Cref{fig:sensitivity_emp_dist} shows the distribution of empirically optimal prompts for different neural predictors. We show the setting of $N=10\,000$ and $T=100$, which is the most reliable setting. At the shortest $\Lmax=5$, the empirical $\hat s$ is mostly correct. At larger $\Lmax$, the empirical $\hat s$ gets further away from $s^*$ with an inconsistent ratio of \ONEs, leading to unreliable identification that hurts the interpretability of the scattered empirical $\hat s$'s.

\def\theempdistcaption{Red cross shows the theoretical $s^*$. The black dashed line shows the boundary set by $\Lmax$.}

\begin{figure}[ht] %
    \centering
    \includegraphics[width=\textwidth]{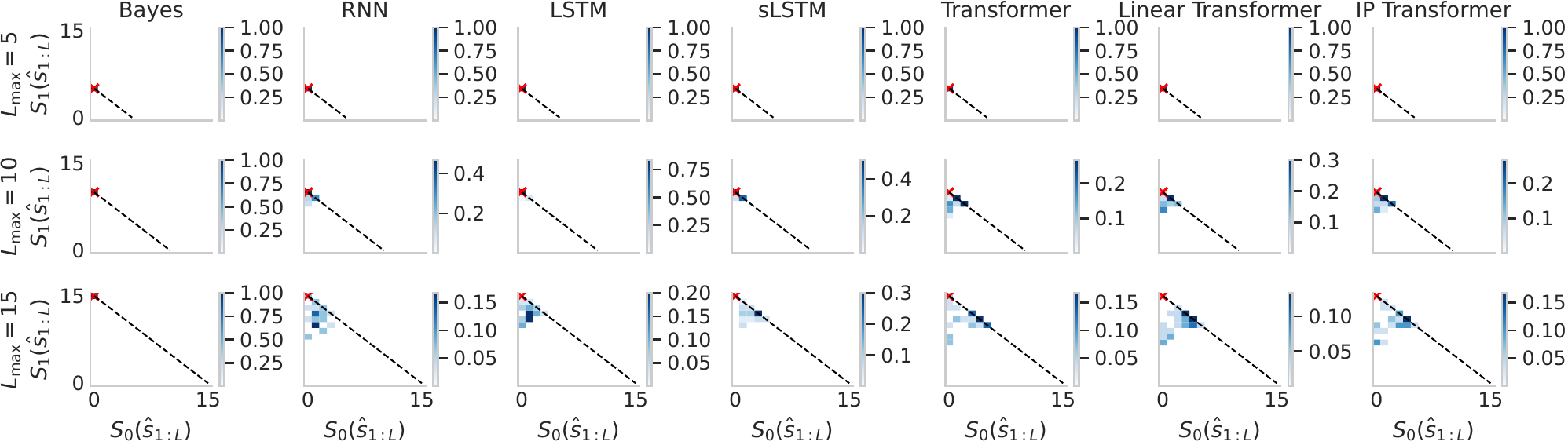}
    \caption{
    The distribution of the prompt counts for $p=\BernMix(0.2, 0.7)$ and $q=\Bern(0.7)$ for different neural predictors. We pick the most reliable prompt setup of $N=10\,000$ and $T=100$.
    \theempdistcaption
    }
    \label{fig:sensitivity_emp_dist}
\end{figure}

\clearpage

\subsection{Pretraining on \texorpdfstring{$\BernMix(0.2, 0.7)$}{BernMix(0.2, 0.7)}, Prompting towards \texorpdfstring{$\Bern(0.6)$}{Bern(0.6)} (OOMD)}\label{sec:bernmix_bern_0.6_supp}
\subsubsection{Zig-Zag Pattern of Theoretically Optimal Prompts}\label{sec:bernmix_bern_0.6_length}
We provide an intuitive explanation for the zig-zag pattern of the theoretical optimal prompts in \cref{fig:length_dependence}. Suppose that we only predict a single $x_1$ for $T=1$, then the best prompt should be one such that $w_L$ in \eqref{eq:mixture_weight_posterior} is as close as possible to 0.6. The counts of \ZEROs and \ONEs in the theoretically optimal prompt for different values of $\Lmax$ are as follows: 1 \ZERO and 2 \ONEs for $3\le\Lmax\le9$, or a prompt of 5 \ZEROs and 5 \ONEs when $10\le\Lmax\le34$, or 16 \ZEROs and 19 \ONEs when $\Lmax\le35$, etc.; see \cref{fig:bernmix_bern_0.6_optimal_length_vs_Lmax}. This explains why the optimal prompt does not take maximal length for a fixed $T$.
Interestingly, the ratio of \ONEs does not converge to 0.6 for up to $\Lmax=10000$.

As $T$ increases, the sequence $x_{1:T}$ continues to update $w_L$, and will eventually, with high probability, converge to 1.0, giving a constant prediction of $p(x_t=1)=0.7$ for large $T$. The prompt can only improve this prediction in the short-term. At $\Lmax=15$, a prompt longer than the optimal length 10 provides more robust predictive ratio (not easily modified by incoming $x$), but this sacrifices the predictability of the first few tokens, which could be better predicted using the 5 \ONEs and 5 \ZEROs optimal for $T=1$. This creates a highly complex interplay between $T$ and the optimal prompt length.

\subsubsection{Unexpected Length Increase in Theoretically Optimal Prompts}\label{sec:bernmix_bern_0.6_length_vs_Lmax}

\Cref{fig:length_dependence} shows the optimal prompt for a fixed $\Lmax$, which has unexpected variations in the optimal length as a function of sequence length. In \cref{fig:bernmix_bern_0.6_optimal_length_vs_Lmax}, we show how the optimal prompt length depends on the upper limit $\Lmax$ for different sequence lengths $T$. We see that the optimal length never decreases, but can jump to $\Lmax$ or stay unchanged as $\Lmax$ increases, again with almost unpredictable pattern. This can result in an illusion that the optimal prompt length is short even though longer prompts have been searched, since theoretically the optimal prompt length does not always increase even if longer prompts are allowed. 

\begin{figure}[h]
    \centering
    \includegraphics[width=0.98\linewidth]{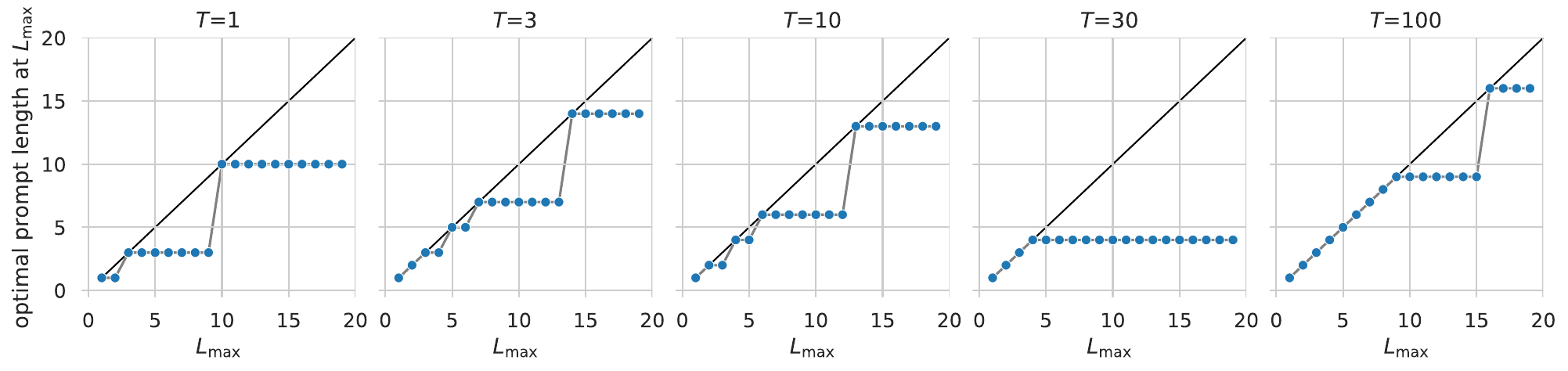}
    \caption{Optimal prompt length as a function of the maximum prompt length $\Lmax$.}
    \label{fig:bernmix_bern_0.6_optimal_length_vs_Lmax}
\end{figure}

\subsubsection{Proportion Correct}

\Cref{fig:sensitivity_0.6_supp} extends the results in \cref{fig:sensitivity}. For $\Lmax=5$, the Transformer predictors have lower proportion correct than other predictors when $T=100$. At $\Lmax=10$ and $\Lmax=15$, all neural predictors got worse than the Bayes predictor when $T$ reaches $30$; and all predictors, including the Bayes predictor, failed to identify $s^*$ when $T=100$.

\begin{figure}[h] %
    \centering
    \includegraphics[width=\textwidth]{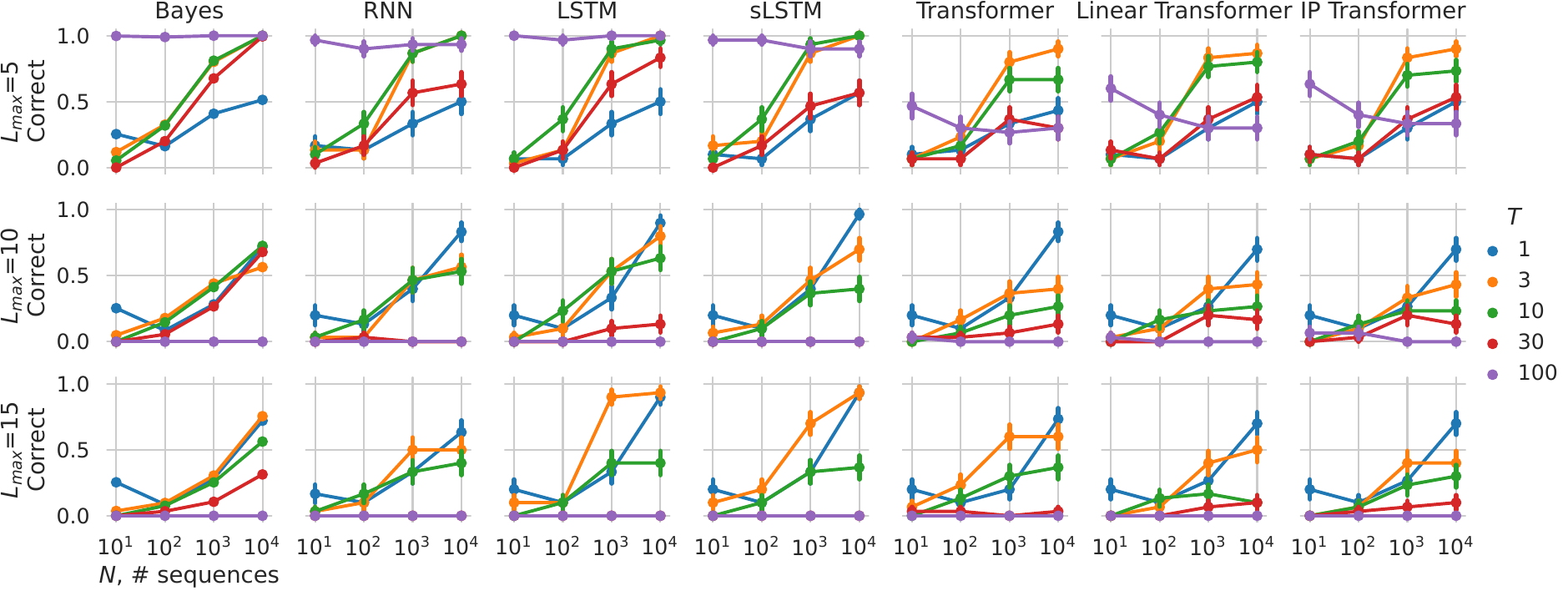}
    \caption{
    The proportion correct of empirically optimal prompt for 
    $p=\BernMix(0.2, 0.7)$ and $q=\Bern(0.6)$.
    Same as \cref{fig:length_dependence} but with more prompting setups.}
    \label{fig:sensitivity_0.6_supp}
\end{figure}

\subsubsection{Distribution of Empirically Optimal Prompts}
\Cref{fig:sensitivity_0.6_emp_dist} shows the distribution of empirically optimal prompts for different neural predictors. 
When the maximum prompt length $\Lmax=5$, the empirical $\hat s$'s are mostly correct. At larger $\Lmax$, some empirical $\hat s$ are close to the $s^*$, but there is a cluster of empirical $\hat s$'s at around (5 \ZEROs and 10 \ONEs) for $\Lmax=15$, which is quite far away from the theoretical $s^*$. In this case, even knowing the pretraining distribution does not explain the existence of this cluster, as the Bayes predictor only has a cluster at the all-\ONE prompt.

\begin{figure}[ht] %
    \centering
    \includegraphics[width=\textwidth]{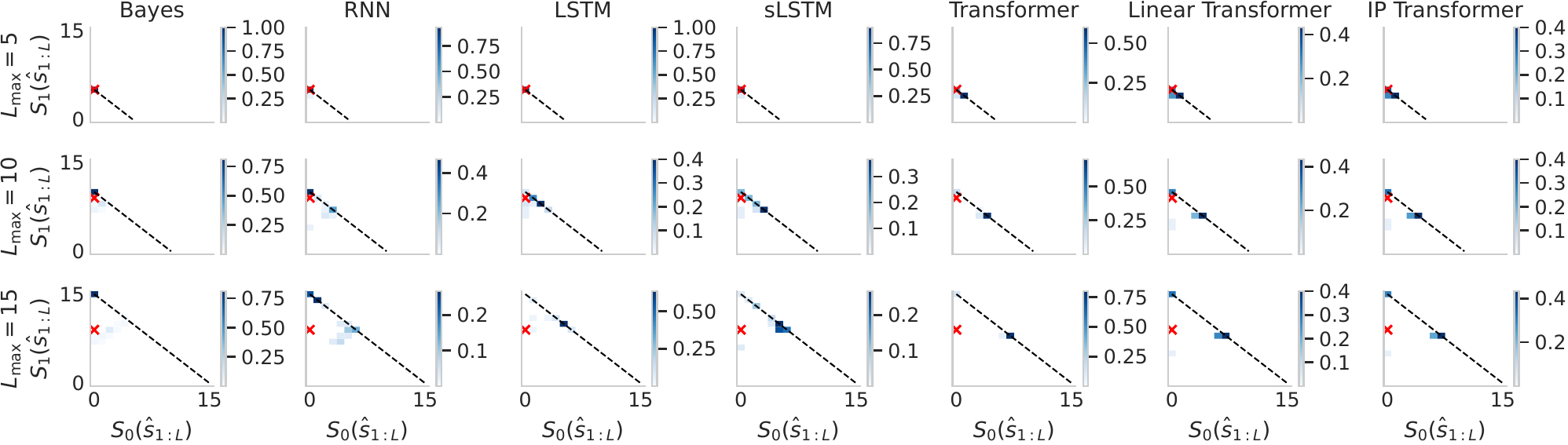}
    \caption{
    The distribution of the prompt counts for $p=\BernMix(0.2, 0.7)$ and $q=\Bern(0.6)$ for different neural predictors. We pick the most reliable prompt setup of $N=10\,000$ and $T=100$.
    \theempdistcaption
    }
    \label{fig:sensitivity_0.6_emp_dist}
\end{figure}

\clearpage
\subsubsection{Loss ``Landscape''}\label{sec:bernmix_bern_0.6_landscape}
\Cref{fig:sensitivity_0.6_landscape} shows the ``loss landscape'' of prompts on the Bayes predictor. In addition to the flat landscape similar to \cref{fig:sensitivity_landscape}, the KL divergence does not go towards zero, which is a sign of OOMD prompting. Thus, the optimal prompts are harder to identify, and the best prompt does not improve the performance much.

\begin{figure}[ht] %
    \centering
    \includegraphics[height=0.26\textheight]{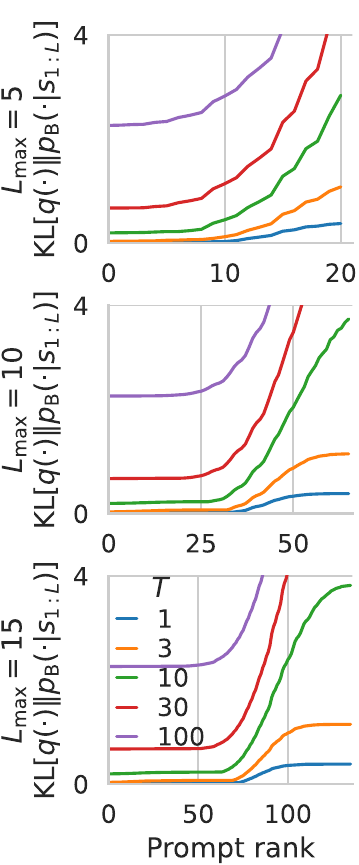}
    \includegraphics[height=0.26\textheight]{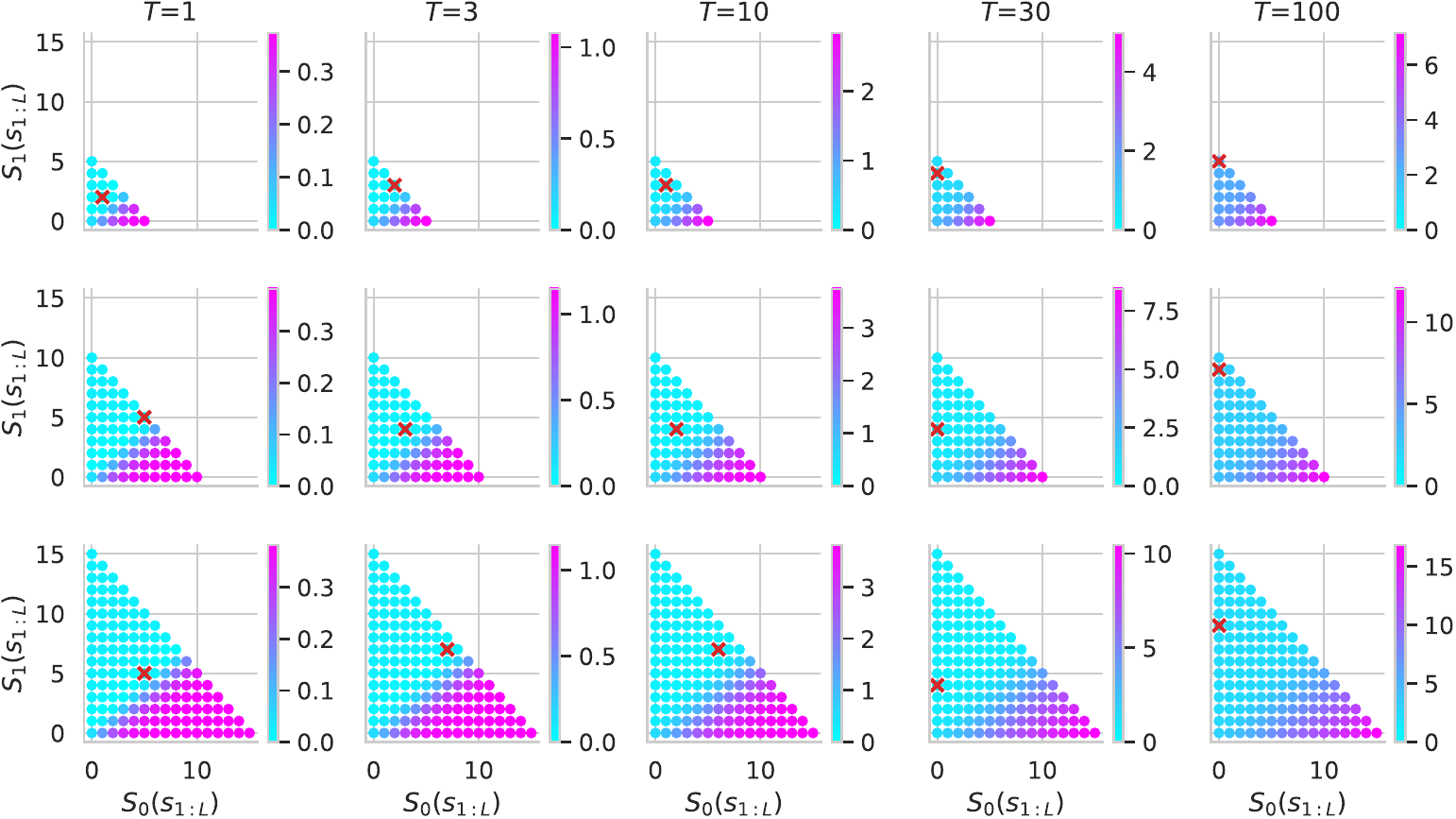}
    \caption{The loss ``landscape'' of $p=\BernMix(0.2, 0.7)$ and $q=\Bern(0.6)$.\landscapelegend
    \label{fig:sensitivity_0.6_landscape}}
\end{figure}

\clearpage
\subsection{Pretraining on \texorpdfstring{$\BetaBern$}{BetaBern}, Prompting towards \texorpdfstring{$\Bern$}{Bern} (IMD)}\label{sec:betabern_bern_supp}

\subsubsection{Proportion Correct}

\Cref{fig:beta_categorical_supp} extends the results in \cref{fig:beta_categorical}. In almost all prompt setups, all neural predictors show very similar trends compared to the Bayes predictor. 
\begin{figure}[ht] %
    \centering
    \includegraphics[width=\textwidth]{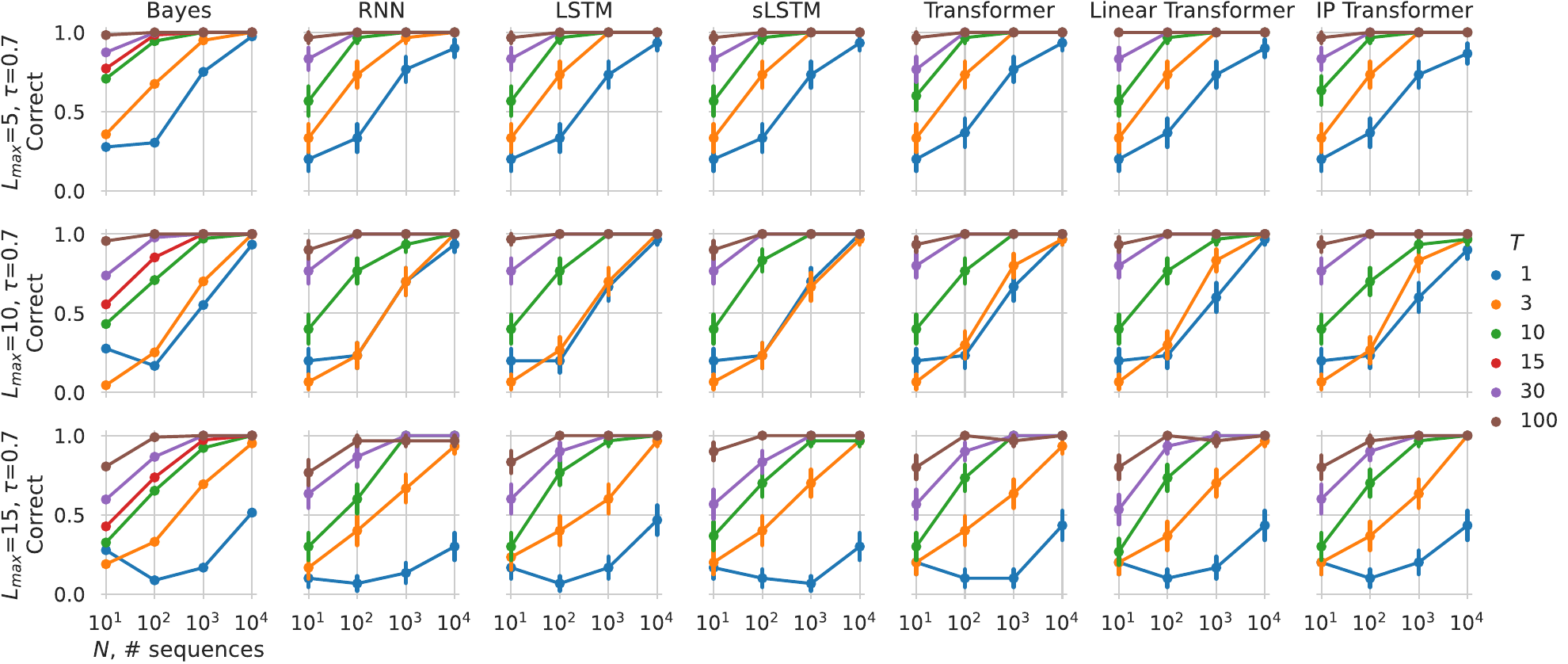}\\
    \vspace{0.5cm}
    \includegraphics[width=\textwidth]{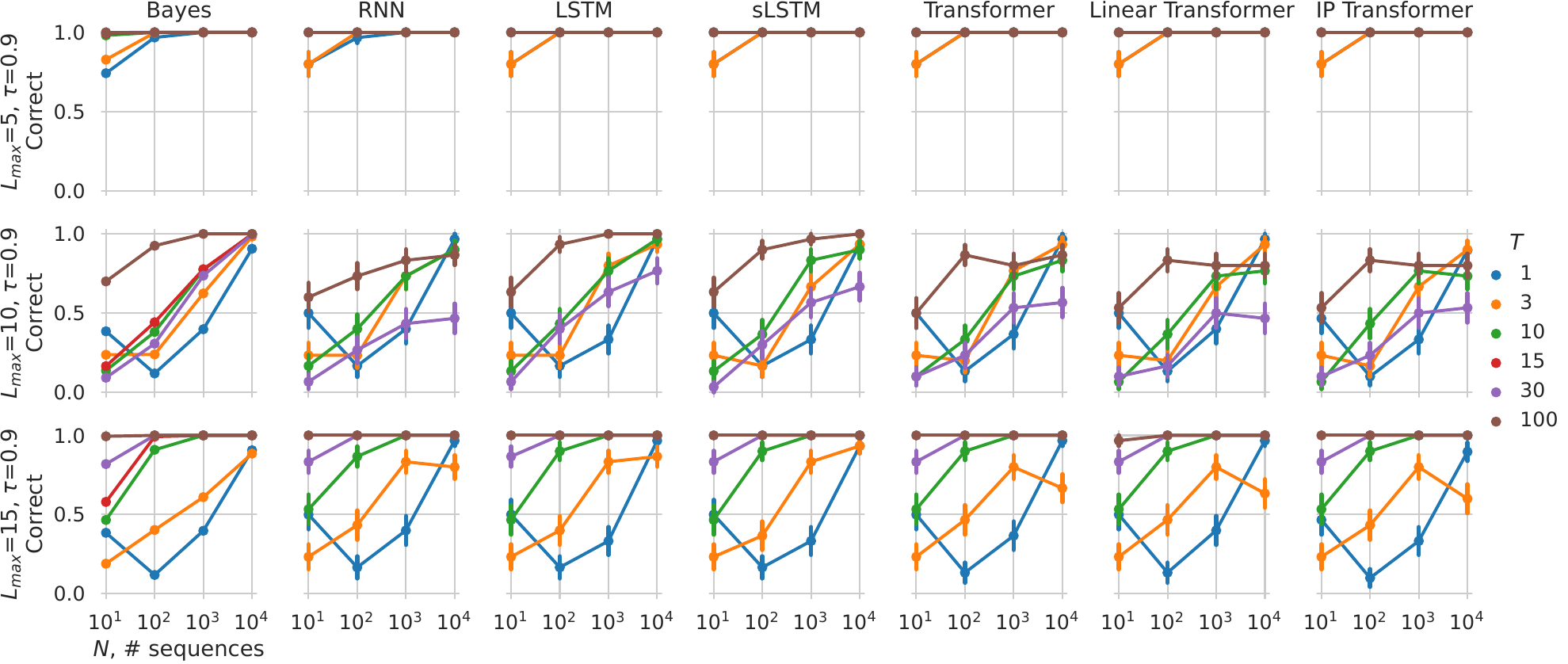}\\
    \caption{Proportion correct of $\hat s$. Same as \cref{fig:beta_categorical} but with more prompting setups.}
    \label{fig:beta_categorical_supp}
\end{figure}

\subsubsection{Distribution of Empirically Optimal Prompts}
As \cref{fig:beta_categorical_emp_dist} shows, the empirically optimal prompts are highly likely correct, suggesting that the distribution is very concentrated at $s^*$. To show how the empirical $\hat s$ differs, we check the distribution of $\hat s$ for the setting $q=\Bern(0.7)$, $N=100$ and $T=3$ in \cref{fig:beta_categorical_emp_dist}. Unlike in previous cases, the prompt distribution is still very close to $s^*$, giving more or less the same ratio of \ONEs and hence consistent interpretation.

\begin{figure}[ht] %
    \centering
    \includegraphics[width=\textwidth]{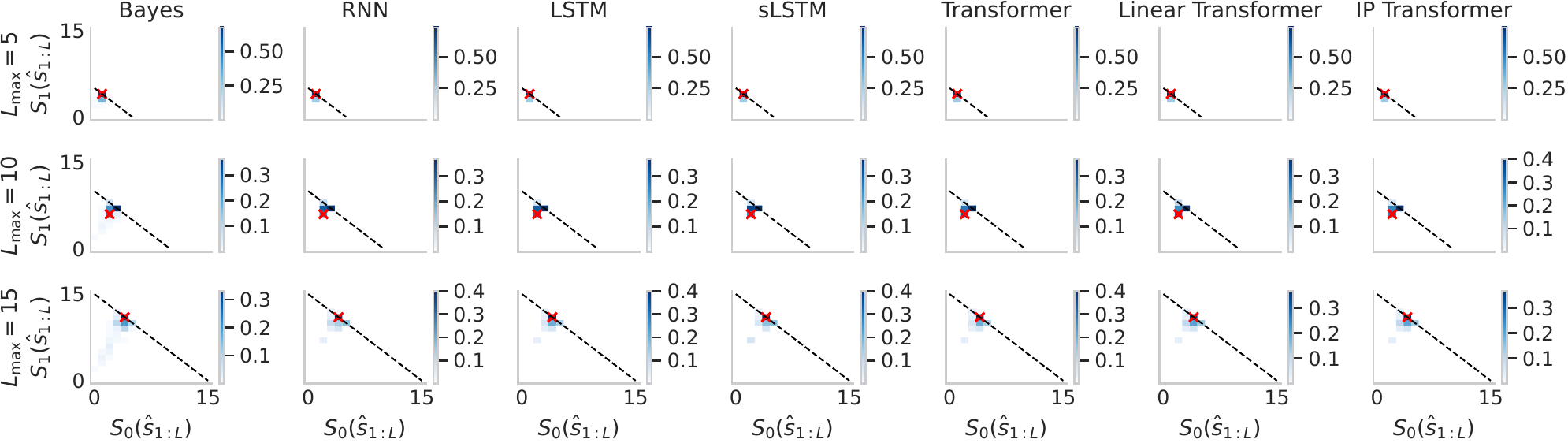}
    \caption{
    The distribution of empirically optimal prompt for $p=\BetaBern(1,1)$, $q=\Bern(0.7)$, $N=100$, $T=3$ for different neural predictors. \theempdistcaption
    \label{fig:beta_categorical_emp_dist}
    }
\end{figure}

\subsubsection{Effect of the Prior at Shorter Sequence Length}\label{sec:betabern_bern_shorterT}

The mismatch at shorter $T$ is due to the uniform prior $\Beta(1,1)$ having pseudocounts of 1 \ONE and 1 \ZERO, which regularizes the predicted bias towards 0.5 by adding 1 to each counter. For shorter $T$, say $T=1$, it is important to get the ratio correct after this regularization, since the Bayes optimal predictor gives $p(x_1=1|s)=(S_1(s) + 1) / (S_0(s) + S_1(s) + 2)$. If we add 1 to each count, then the regularized ratio is closer to the true bias. To reduce the gap between the regularized ratio, the optimal prompt length under the upper limit $\Lmax$ may be less than $\Lmax$. For example, for $\Lmax=10$ and $q=\Bern(0.7)$, the optimal prompt is 2 \ZEROs and 6 \ONEs and has length 8 (see \cref{tab:betabern_seq_len}); the regularized ratio is then $(6+1) / (2 + 6 + 2) = 0.7$, which is optimal for predicting the next $x_1$.

However, as more tokens arrive, the regularized ratio fluctuates due to the randomness in the draws from $\Bern(0.7)$, while ideally it should be fixed at 0.7. Hence, there is additional benefit from being certain about the bias (with larger counts), which can be gained by using more tokens in the prompt. For the same example of $\Lmax=10$ and $q=\Bern(0.7)$, but now $T=10$, the optimal prompt becomes 3 \ZEROs and 7 \ONEs \cref{tab:betabern_seq_len}. There will be a small cost for predicting earlier $x_t$'s, as the regularized ratio is initially not exactly 0.7.

\subsubsection{Non-Uniform Latent Factor}

\begin{table}%
\centering
\caption{Example optimal prompts $s^*_\Lmax$ for $p=\BetaBern(1, \beta)$ and $q=\Bern(0.7)$ of various values of 
$\beta$, $T$ and $\Lmax$.}%
\label{tab:betabern_seq_len}
\begin{tabular}{lccccc}
\hline
\multicolumn{2}{c}{SDs $p$ and $q$} & \multicolumn{4}{c}{Optimal prompt ${s^*_\Lmax}$} \\
$\beta$    & $T$              & $L_{\text{max}}$         & ($S_0$, $S_1$)   &  $S_0$+ $S_1$ & $S_1/(S_0+S_1)$      \\
\hline
1           & 1                & 3                        & (0, 1)          &  1            &  1.00    \\
1           & 3                & 3                        & (0, 2)          &  2            & 1.00   \\
1           & 5                & 3                        & (1, 2)          &  3            & 0.67   \\
1           & 1                & 10                       & (2, 6)          &  8            & 0.75  \\
1           & 10               & 10                       & (3, 7)          &  10            & 0.70  \\
1           & 3                & 12                       & (3, 8)          &  11           & 0.72 \\
1           & 5                & 12                       & (3, 8)          &  11           & 0.72 \\
\hline
2           & 10                & 10                       & (2, 8)         &  10           & 0.8  \\
2           & 100               & 20                      & (5, 15)         &  20           & 0.75  \\
2           & 100               & 50                       & (14, 36)       &  50           & 0.72     \\
\hline
\end{tabular}
\end{table}

In \cref{tab:betabern_seq_len}, we show the optimal prompts on a few prompt setups. When $\beta=1$ and $p_\tau$ is uniform, the optimal prompt contains roughly the correct proportion of \ZEROs and \ONEs even for a short task sequence length $T$, converging to the true bias in $q$ as $\Lmax$ increases. When $\beta=2$ and thus $p_\tau$ is biased towards zero, the optimal prompts have to debias the prior, and thus the empirical ratio of \ONEs is further away from the ground truth bias 0.7, and requires larger $T$ and $\Lmax$ to converge to 0.7.

\subsubsection{Loss ``Landscape''}\label{sec:betabern_bern_landscape}

\Cref{fig:betabern_bern_landscape}(left) shows that the theoretical optimal prompt has a distinctively lower KL divergence compared to other prompts, with the exception that $T=1$ still has a very flat landscape. The other columns show a clearer optimal region, especially for $\tau=0.7$. These are in stark contrast to \cref{fig:sensitivity_landscape,fig:sensitivity_0.6_landscape} where the optimal points do not standout among other suboptimal prompts.

\begin{figure}[!th]
    \centering
    \includegraphics[height=0.27\textheight]{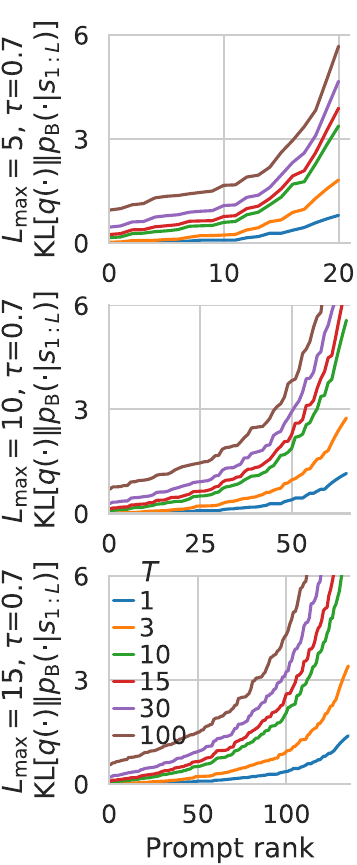}
    \includegraphics[height=0.27\textheight]{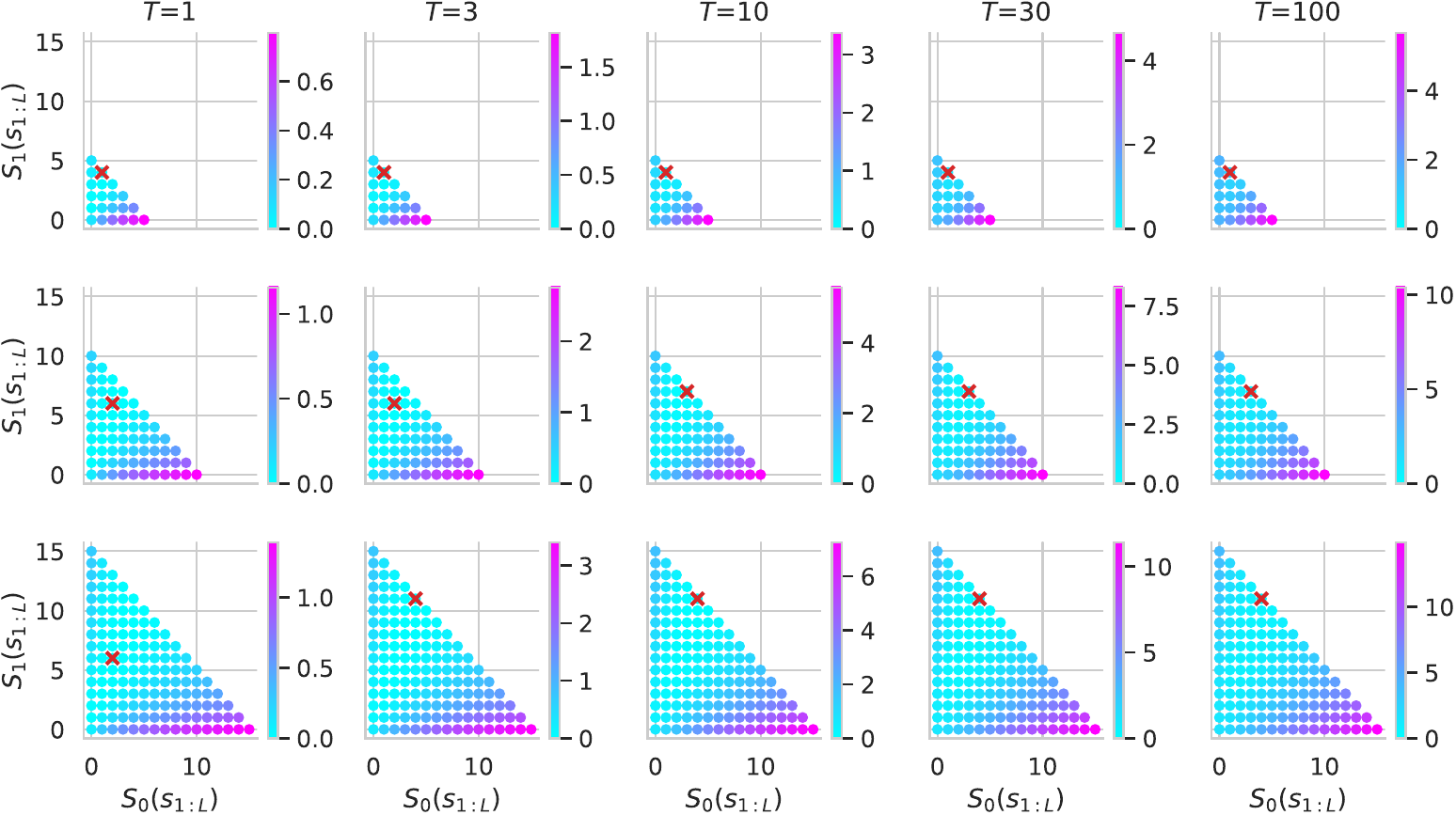}\\
    \vspace{1cm}
    \includegraphics[height=0.27\textheight]{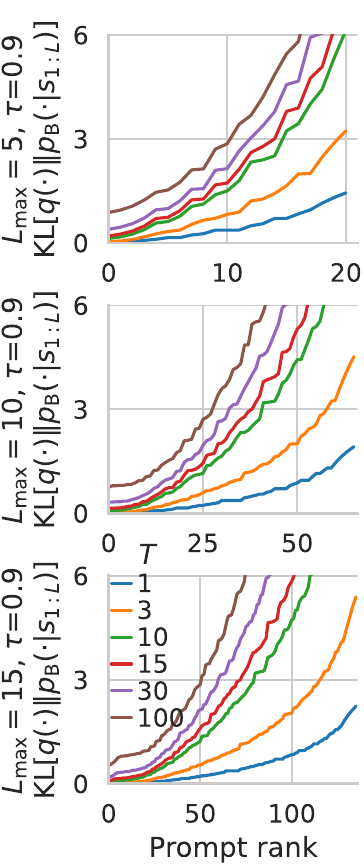}
    \includegraphics[height=0.27\textheight]{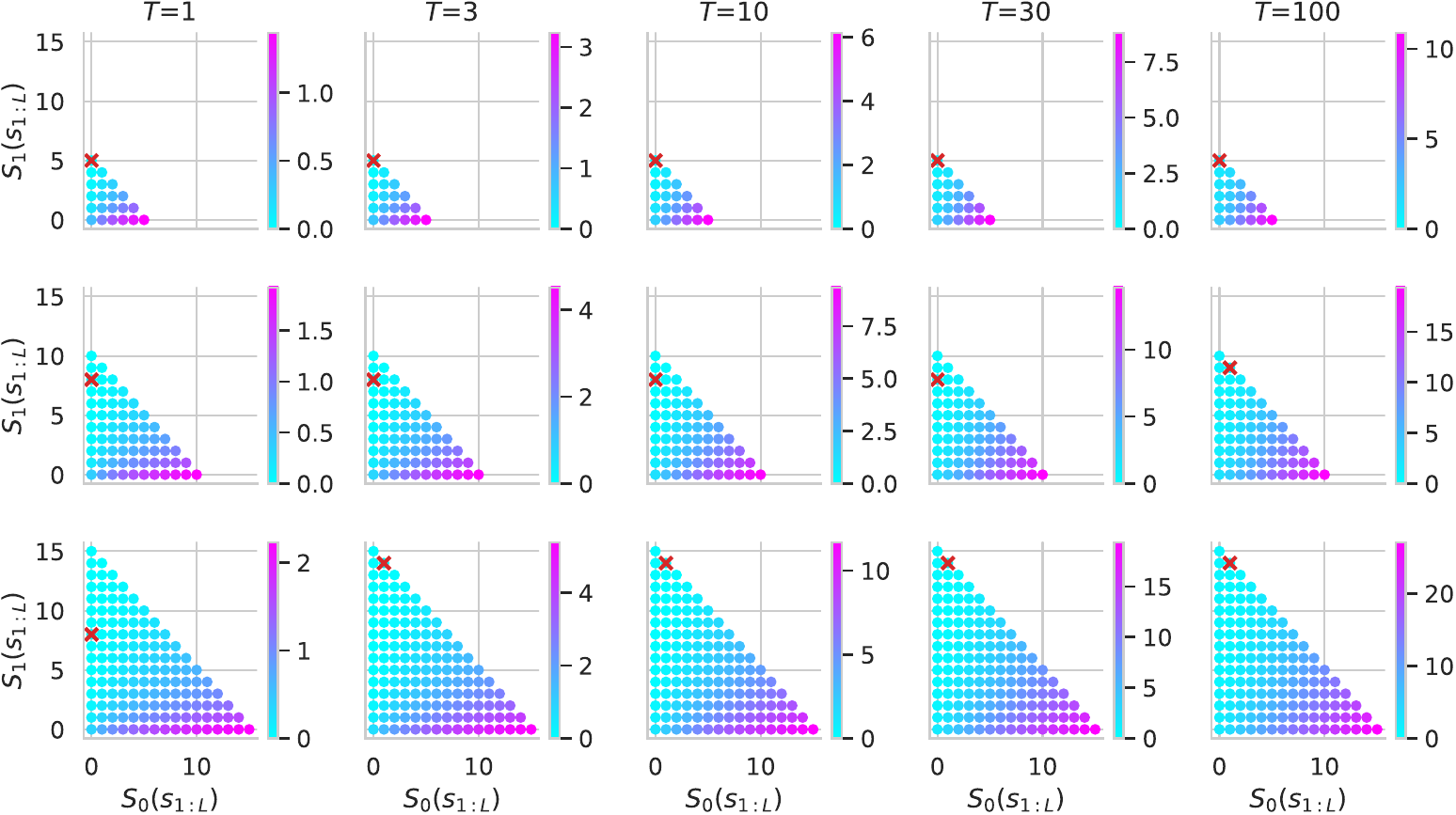}\\
    \caption{The loss ``landscape'' of $p=\BetaBern(1, 1)$ and $q=\Bern(0.7)$ or $q=\Bern(0.9)$.\landscapelegend %
    \label{fig:betabern_bern_landscape}}
\end{figure}

\clearpage

\subsection{Pretraining on \texorpdfstring{$\BetaBern$}{}, Prompting towards \texorpdfstring{$\BernMix$}{} (OOMD)}\label{sec:betabern_bernmix}
\begin{figure*}[ht] %
\centering
\begin{minipage}{0.42\textwidth}
\centering
\footnotesize
\begin{tabular}{cccc|cc}
\hline
\multicolumn{4}{c|}{\textbf{$q=\BernMix(\tau_1,\tau_2)$}} & \multicolumn{2}{c}{\textbf{$s^*_{100}$}} \\
$\tau_1$    & $\tau_2$  & $\Delta\tau$  & $\bar{\tau}$   & ($S_0$, $S_1$)       & ratio  \\
\hline
1/2         & 1/3       & 1/6           & 0.42                  & (20, 14)      & 0.41      \\
2/5         & 3/5       & 1/5           & 0.50                  & (11, 11)      & 0.50      \\
1/5         & 2/5       & 1/5           & 0.30                  & (13, 5)       & 0.28      \\
1/4         & 1/2       & 1/4           & 0.38                  & (9, 5)        & 0.36      \\                  
1/3         & 2/3       & 1/3           & 0.50                  & (3, 3)        & 0.50      \\
1/4         & 3/4       & 1/2           & 0.50                  & (1, 1)        & 0.50      \\
1/5         & 4/5       & 3/5           & 0.50                  & (0, 0)        & -         \\
2/5         & 1         & 3/5           & 0.70                  & (0, 1)        & 0.70\\
\hline
\end{tabular}
\captionsetup{type=table}
\caption{Empirically optimal prompts for ${p=\BetaBern(1,1)}$ and several task DGs ${q=\BernMix(\tau_1, \tau_2)}$. For each case, the optimal $s_\Lmax^*$ is found for large ${\Lmax=100}$ and ${T=100}$. ${\Delta\tau:=(\tau_2-\tau_1)}$, and ${\bar\tau:=(\tau_1+\tau_2)/2}$.
\label{tab:beta_bernmix_optimal}
 }
\end{minipage}\hfill
\begin{minipage}{0.55\textwidth}
    \includegraphics[width=\textwidth]{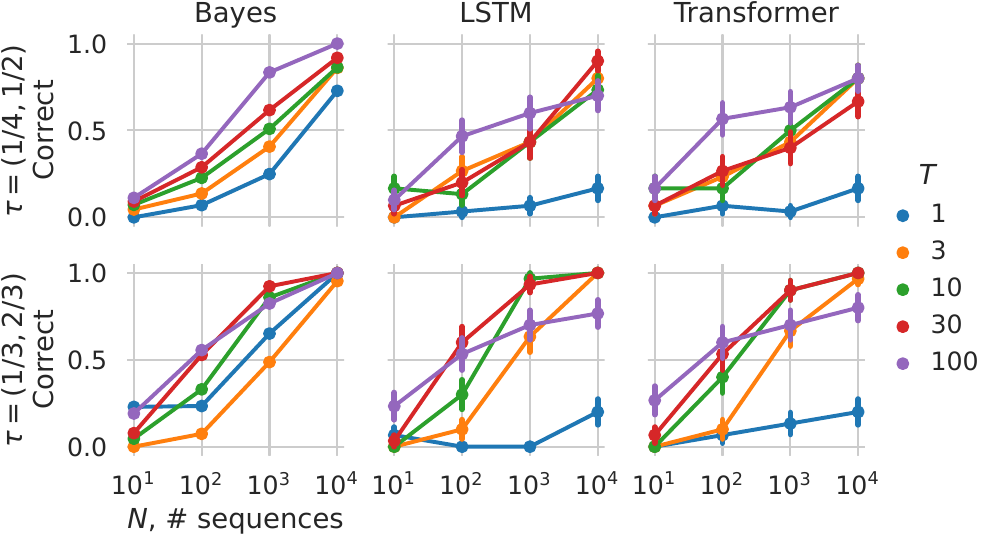}
    \caption{Proportion of empirical prompts that match the theoretically optimal prompt for two of the task DGs in \cref{tab:beta_bernmix_optimal}. The prompt length is capped at $\Lmax=15$, which is sufficient for the empirically optimal prompt in both cases. \Cref{fig:betabern_bernmix_supp} shows additional results. %
    \label{fig:beta_bernmix}
    }
\end{minipage}
\end{figure*}

The difficulty with this mixture task $q$ is that it is impossible to provide a single input sequence (the prompt) that leads to a bimodal $p_{\tau|x}$. Thus, while each component of the task mixture is IMD, the mixture itself is OOMD.
The theoretically optimal prompts $s^*_\Lmax$'s for different instances of such $q$ under a large prompt length limit $\Lmax=100$ are shown in \cref{tab:beta_bernmix_optimal}. Although the empirical ratio reflects the mean bias of $q$ well, it is still unclear why some prompts are longer and some are shorter. Given just the information about $q$, it is difficult to make sense of the varying optimal prompt length.

If we know $p$ fully, then this pattern makes more sense from a posterior concentration perspective. If $\tau_1$ and $\tau_2$ are close together, a helpful prompt should make the posterior  concentrate loosely around these values. The closer the two values, the longer thus the optimal prompt (with a ratio of \ONEs close to the mean of the two mixture components). On the other hand, if the two values are far apart, a helpful prompt leaves the posterior as broad as possible, leading to very short prompts or no prompt at all.

The empirical prompts found on neural predictors are more likely to be correct for larger $T$ and larger $N$ (\cref{fig:beta_bernmix}), similar to the in-meta-distribution case in \cref{sec:continuous_latent}. The overall trend of proportion correct also generally agrees with empirical prompts found on Bayes predictors, except for $T=100$ for the case of $\tau=(1/3, 2/3)$. \Cref{fig:betabern_bernmix_supp} extends \cref{fig:beta_bernmix}.
Across all prompt setups, the agreement between Bayes and neural predictors is worse compared to the previous IMD case shown in \cref{fig:beta_categorical_supp}.

Overall, although the optimal prompts cannot reveal the bimodality nature and the bias in each component of the task DG, they still reveal overall statistical properties. Compared to the OOMD case in \cref{sec:ood_prompting}, here we are able to make sense of the ratio of \ONEs in the optimal prompts by matching them to the mean bias of $q$, but a detailed interpretation of the prompt length relies on knowing $p$.

\cref{fig:betabern_bernmix_emp_dist} shows the distribution of empirical prompts. To show how they are different to the theoretical $s^*$, we pick the prompt setup $N=100$ and $T=30$. For the least reliable case $q=\BernMix(1/4, 1/2)$, the suboptimal prompts turn out to have consistent ratio of \ONEs, giving consistent interpretation of the mean bias in $q$. This also holds for the suboptimal prompts in the case $q=\BernMix(1/3, 2/3)$.

\Cref{fig:betabern_bernmix_landscape} shows the loss ``landscape''. The optimal prompt produces a more distinctive optimal KL divergence compared to the two cases with $\tau$ supported on a finite mixture.
In particular, for the case of $q=\BernMix(1/4, 3/4)$, there is a sharp dip around the optimal prompt, which should imply more reliable identification.
This is consistent with the reliability under the Bayes predictor \cref{fig:betabern_bernmix_supp}.

\begin{figure}[ht] %
    \centering
    \includegraphics[width=0.9\textwidth]{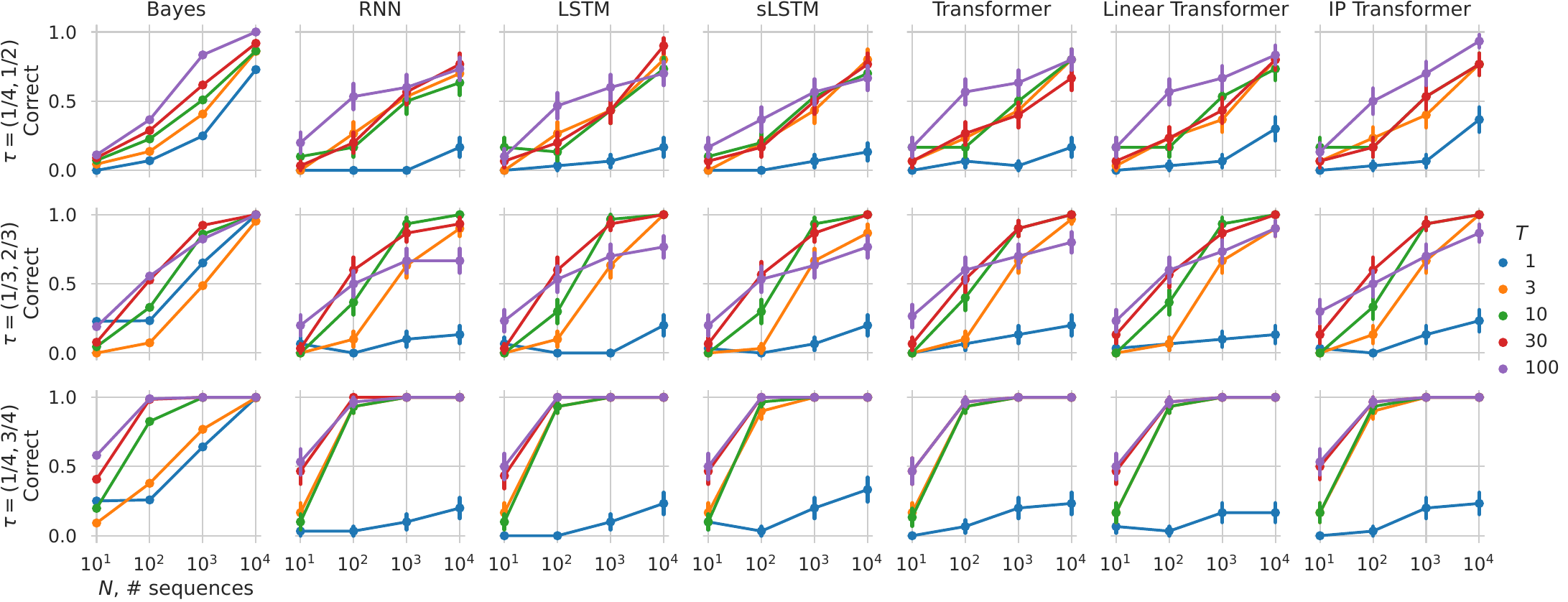}
    \caption{Proportion correct of $\hat s$. Same as \cref{fig:beta_bernmix} but with more prompting setups, using $\Lmax=15$.}
    \label{fig:betabern_bernmix_supp}
\end{figure}

\begin{figure}[ht] %
    \centering
    \includegraphics[width=0.9\textwidth]{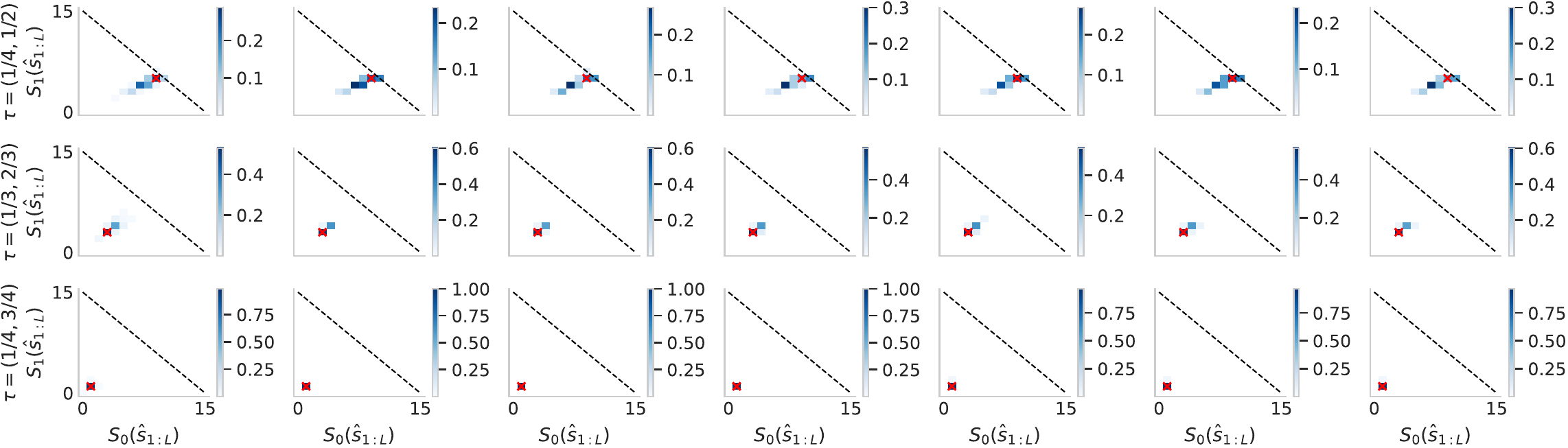}
    \caption{Distribution of empirically optimal prompts. Here, we set the $\Lmax=15$ for all predictors, which is greater than the length of the theoretical $s^*$ in all cases. \theempdistcaption
    \label{fig:betabern_bernmix_emp_dist}
    }
\end{figure}

\begin{figure}[ht] %
    \centering
    \includegraphics[height=0.27\textheight]{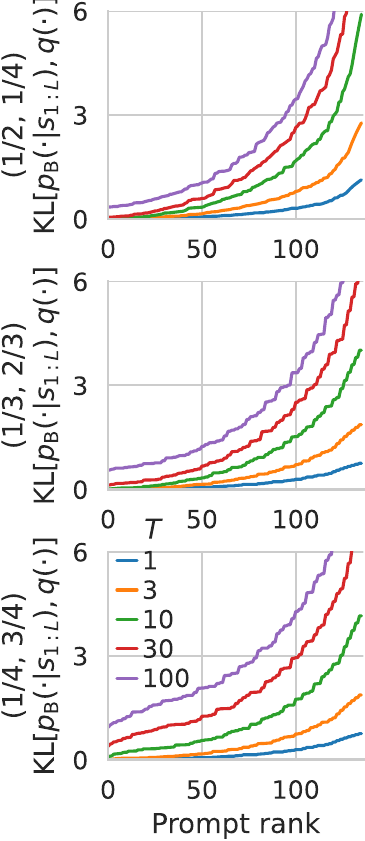}
    \includegraphics[height=0.27\textheight]{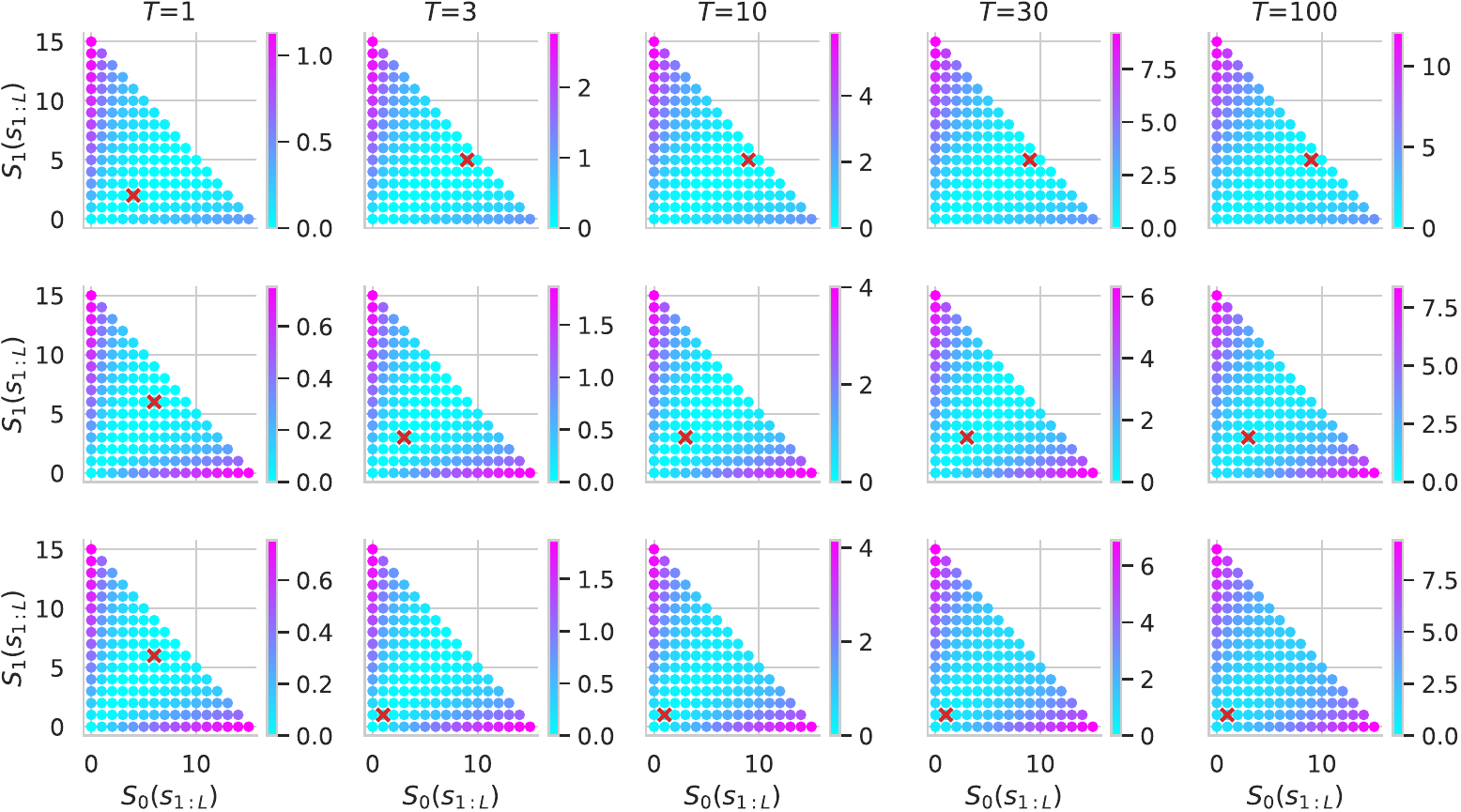}
    \caption{The loss ``landscape'' of $p=\BetaBern(1.0, 1.0)$ and $q=\BernMix(\tau_1, \tau_2)$.\landscapelegend
    \label{fig:betabern_bernmix_landscape}}
\end{figure}

\clearpage

\section{MORE COMPLEX DGS}\label{sec:more_complex_DGs}

We gradually move towards more complex DGs beyond conditionally independent sequences, and show that they too may produce unintuitive behaviors. Interpreting the optimal prompts is easier in the first example with switching latent factor \cref{fig:switching},
but it gets more complicated for the hierarchical and topic model DGs 

\subsection{Switching DGs}\label{sec:switching}

In the first example, we use a DG that switches periodically between two coins with fixed biases, giving a non-i.i.d sequence.
\begin{definition}[Switching Process]\label{thm:switching_process}
$\textrm{SwitchProc}(\eps, \lambda)$ for $\eps\in[0, 1]$ and $\lambda \in \mathbb{N}^+$ generates
\begin{equation*}
\begin{aligned}
     x_t   &\sim \Bernoulli(y_t) ~~ \forall t \in \{1,\ldots, T\}, \text{~~where~~} 
    y_{1:T} = [\underbrace{\eps, \ldots, \eps}_{\lambda ~ \eps\text{'s}},
                \underbrace{1-\eps, \ldots, 1-\eps}_{\lambda ~ (1-\eps)\text{'s} },
                \underbrace{\eps,\ldots,\eps}_{\lambda ~\eps\text{'s}},\ldots]\,.
\end{aligned}
\end{equation*}
Here, the latent factor $\tau:=(\eps, \lambda)$. The pretraining DG $p$ is a finite mixture of the Switching Process:
\end{definition}
\begin{definition}[Random Switching Process] This DG generates sequences by
first sampling 
$\eps \sim \textrm{Uniform}([0, 1])$ and
$ \lambda  \sim \text{Uniform}(\{3, 4, 5\})$,
then
$x_{1:T} \sim \text{SwitchProc}(\eps, \lambda)$.
\end{definition}
We take $\textrm{SwitchProc}(\eps, \lambda)$ with fixed values of $\eps$
and $\lambda$ as a task DG $q$, such that $q\in\mathcal{M}_p$. Examples of $\tau$ and $y$ are shown in \cref{fig:switching}(left). Prompting here is harder; for instance, a prompt that alternates between $\lambda$ \ZEROs and $\lambda$ \ONEs is very informative of $\lambda$, but not for $\eps$ if $\eps\notin\{0, 1\}$.
We set $\Lmax=15$ and search through $s_{1:15}\in\mathcal{A}^{15}$ given data sequences with different lengths $T$.

\begin{figure*}[t]
    \centering
    \includegraphics[width=\textwidth]{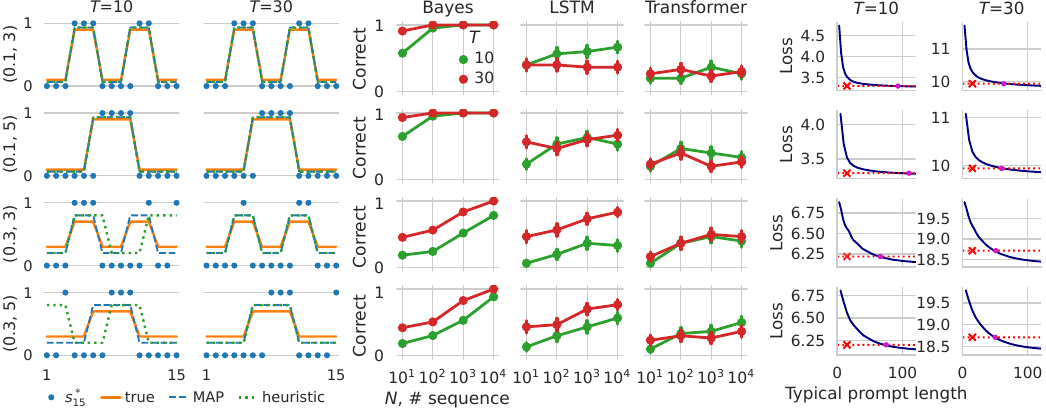}
    \vspace{-1em}
    \caption{Results for the switching DGs. Rows show tasks with different values of $(\eps, \lambda)$. 
    Left two columns: theoretical prompts $s^*$'s (dots), the true latent bias $y$ (orange solid), and a heuristic estimate of the latent $y$ based on $s^*$. 
    Middle three columns: the proportion correct of $\hat s$; \Cref{fig:switching_supp} shows additional results.
    Right two columns, log-loss of typical prompts from $q$ with increasing lengths (blue line), compared to theoretical $s^*$ with length $15$ (red).
    }
    \vspace{-0.2em}
    \label{fig:switching}
\end{figure*}

Can we interpret the optimal prompts? For each of the four tasks in \cref{fig:switching}(left), we plot the biases $y$ under the true $\tau$ (orange solid), and one example theoretical optimal prompt $s^*$ (blue dots).
All the equally optimal $s^*$'s of each $q$ induce the same posterior $p(\tau|x^*)$,
so we show the latent $y$ corresponding to the mode of $p(\tau|s^*_{15})$. The estimated $\lambda$'s are correct for all four $q$'s, and the estimates for $\eps$'s are close to the ground-truth in $q$.
Thus, in this case, the true $\lambda$ in $q$ can be recovered from the theoretical $s^*$ if we know the pretraining $p$ and its Bayes predictor.

What if we do not know $p$? \Cref{fig:switching}(left) also shows the estimated $y$'s from a heuristic method given \emph{incomplete} knowledge of $p$ (see \cref{sec:switching_details}). The estimate is close to the truth $y$ for $T=30$, but can be wrong for $T=10$. Thus, although the uniform $p_\tau$ seems to help interpreting the prompt, knowing $p$ fully is essential to recover the task $\tau$. %
\Cref{fig:switching}(middle) shows the proportion correct of empirical $\hat{s}$. Under the Bayes predictor, increasing $N$ and $T$ leads to higher chances of finding $s^*$, more so for $T=30$ compared to $T=10$. %
This trend is weaker on the neural predictors.

Given the high cost of obtaining optimal prompts, how much do we gain compared to other prompts? 
Here, we compare them to statistically typical prompts (samples) drawn from $q$, which has expected log-loss 
$
    \mathbb{E}_{s_{1:L}\sim q}[\mathcal{L}(q, p, s_{1:L})].
$
Although this is almost always higher than the log-loss \eqref{eq:log_loss_L} under the optimal prompt for the same prompt length, it is much cheaper to get \emph{longer} typical prompts to lower the log-loss at lower computational costs.
\Cref{fig:switching}(right) shows the log-loss of theoretical $s^*$ at $\Lmax=15$ and those of typical prompts with increasing lengths. Typical prompts require 3-8 times the length of the optimal prompt to reach the same log-loss, and are thus much less efficient in terms of the number of tokens, mirroring findings on natural language prompts~\citep{bhargava2023s,renze2024benefits}.
The shorter but performant optimal prompt then has the advantage of occupying shorter context windows, which is more desirable for many applications 
\citep{chang2024efficient}.
\subsubsection{Heuristic Method for Interpreting Prompts of the Switching Tasks}\label{sec:switching_details}

Given a prompt $s_{1:L}$, we want to ``guess'' the corresponding latent causes $\eps$ and $\lambda$. The heuristic method assumes that we know the sequences are generated from the switching process in Definition~\ref{thm:switching_process}, and that  
\begin{enumerate}
    \item $\eps\in[0, 1]$;
    \item $\lambda\in\{3,4,5\}$;
    \item The first bias $y_1$ associated with $s_1$ may \emph{not} be the first $\eps$ appearing in Definition~\ref{thm:switching_process}. In other words, the ``phase'' of the $y$ is unknown. Taking $\lambda=3$ as an example, $y$ can start with any of the following:  
    \begin{equation}\label{eq:all_phases}
\iftrue
    \begin{alignedat}{7}
    [&\eps, &~&  \eps, &~&\eps, &~& 1\!-\!\eps, &~& 1\!-\!\eps, &~& 1\!-\!\eps, &~& \ldots]\\
    [&\eps, &&  \eps, &&1\!-\!\eps, & & 1\!-\!\eps, & & 1\!-\!\eps, & &\eps, & & \ldots]\\
    [&\eps,     &&1\!-\!\eps,&& 1\!-\!\eps, & & 1\!-\!\eps, & &\eps, && \eps, & & \ldots]\\
    [&1\!-\!\eps,  && 1\!-\!\eps, & & 1\!-\!\eps, & &\eps, && \eps, & &\eps, && \ldots]\\
    [& 1\!-\!\eps, & & 1\!-\!\eps, & &\eps, && \eps, & &\eps, &&1\!-\!\eps,  && \ldots]\\
    [& 1\!-\!\eps, & &\eps, && \eps, & &\eps, &&1\!-\!\eps,  && 1\!-\!\eps, & & \ldots]\\
    \end{alignedat}
\else
    \begin{alignedat}{7}
    [&\epsilon, &~&  \epsilon, &~&\epsilon, &~& (1-\epsilon), &~& (1-\epsilon), &~& (1-\epsilon), &~& \ldots]\\
    [&\epsilon, &&  \epsilon, &&(1-\epsilon), & & (1-\epsilon), & & (1-\epsilon), & &\epsilon, & & \ldots]\\
    [&\epsilon,     &&(1-\epsilon),&& (1-\epsilon), & & (1-\epsilon), & &\epsilon, && \epsilon, & & \ldots]\\
    [&(1-\epsilon),  && (1-\epsilon), & & (1-\epsilon), & &\epsilon, && \epsilon, & &\epsilon, && \ldots]\\
    [& (1-\epsilon), & & (1-\epsilon), & &\epsilon, && \epsilon, & &\epsilon, &&(1-\epsilon),  && \ldots]\\
    [& (1-\epsilon), & &\epsilon, && \epsilon, & &\epsilon, &&(1-\epsilon),  && (1-\epsilon), & & \ldots]\\
    \end{alignedat}
\fi
    \end{equation}
\end{enumerate}
When the phase is unknown, it makes sense as a heuristic to first find a $\lambda$ to match the prompt. Take $\epsilon=0$ so that $y$ is binary, and enumerate all possible $y$'s of length $L$ with different phases (as in \cref{eq:all_phases}) and different values of $\lambda$. Pick the $y$ that has the fewest mismatches (smallest Hamming distance) with the binary prompt, note the best match by $y^*$ and the corresponding $\lambda$. Effectively, this $\lambda$ produces the smallest ``mismatch''  between the lower bias and a \ZERO token in the prompt, and between the higher bias and a \ONE token in the prompt.

Given the best matching binary $y^*$, we estimate $\eps$ as the proportion of incorrect matches with the prompt: $$\frac{1}{L}\sum_{i=1}^L [(1-y^*_i)(s_i) + y^*_i(1-s_i)]$$.

\subsubsection{Proportion Correct}

\Cref{fig:switching_supp} extends the results of \cref{fig:switching}(middle). Note that in this experiment we search through all possible binary prompts of length $\Lmax$, and for each prompt we compute the expectation \eqref{eq:log_loss_L} by enumerating all possible sequences of length $T$. For the Bayes predictor this can be done quite efficiently, but for neural predictors this is still quite computationally intensive. As such, we only sweep $T\in\{10, 30\}$. We observe a robust increasing pattern only on the Bayes predictor. For the neural predictors, there is a slight increasing trend only for $\eps=0.3$.

\begin{figure}[ht!]
    \centering
    \includegraphics[width=\textwidth]{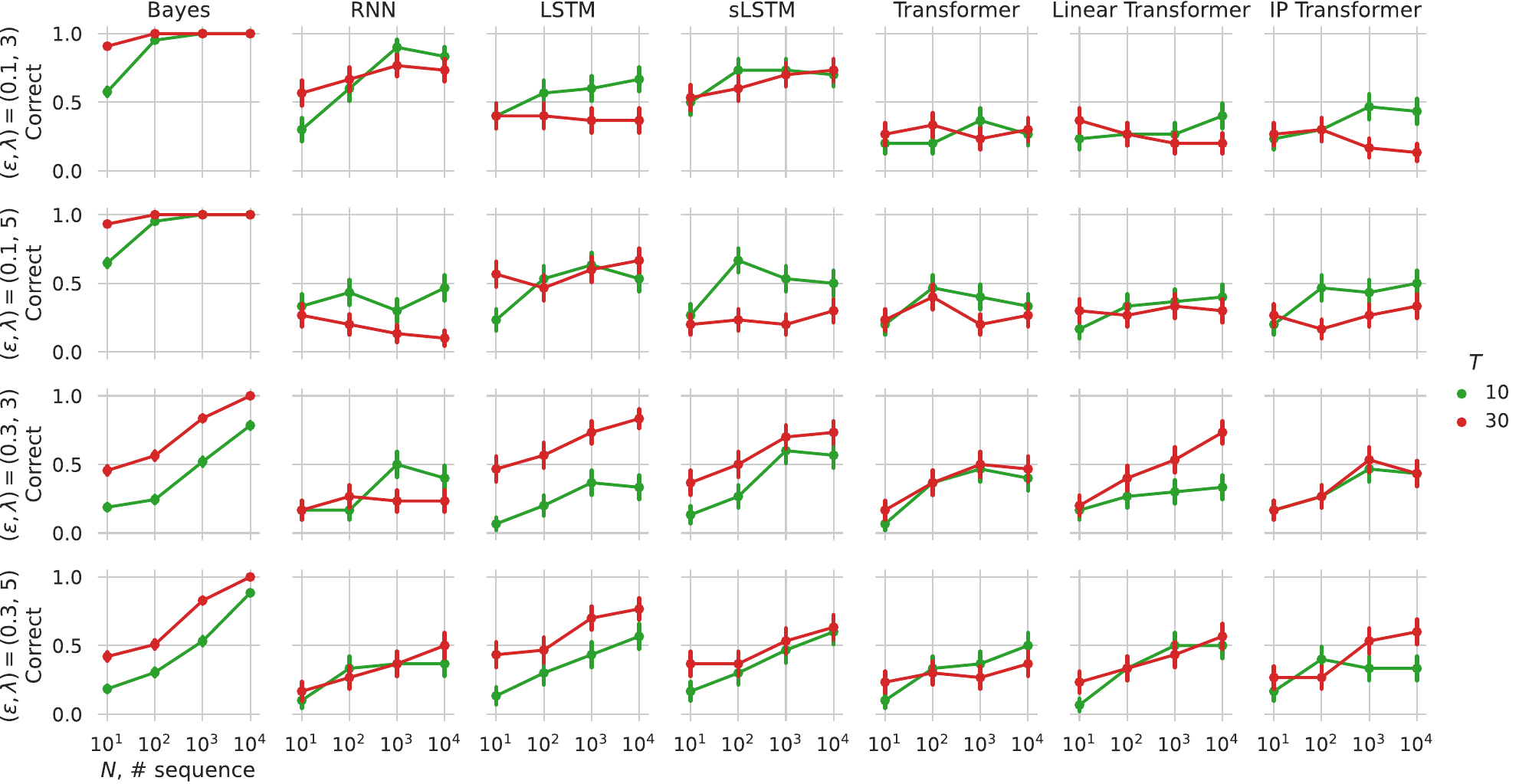}
    \caption{Same as \cref{fig:switching}(middle) but with more prompting setups, using $\Lmax=15$.}
    \label{fig:switching_supp}
\end{figure}

\clearpage

\begin{figure}[ht!]
    \centering
    \includegraphics[width=\textwidth]{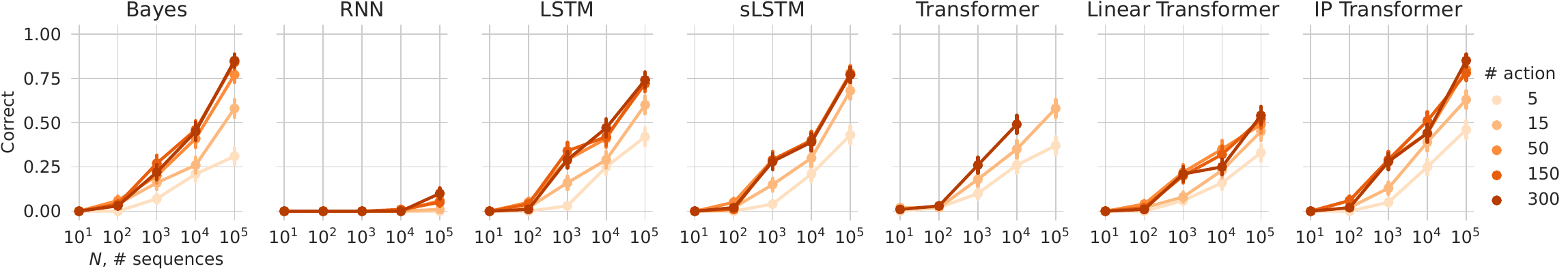}
    \caption{Same as \cref{fig:bandit}(lower left) but with more prompting setups, using $\Lmax=15$.}
    \label{fig:bandit_correct_supp}
\end{figure}

\subsection{Hierarchical DGs}\label{sec:hierarchical}

In the main text, we used CIB-DGs to demonstrate that, even in these seemingly intuitive prompt settings, the optimal prompts can show unintuitive behaviors. Natural languages exhibit more hierarchy and richer structures that break conditional independence. What would happen if we move towards more complex DGs? Here, we show that using more complex and hierarchical DGs only makes the theoretical optimal prompts even less interpretable, and thus the examples in the main text are sufficient for our purpose of being easy to understand. 

Consider a CIB-DG with $\tau$ sampled from a hierarchical model
\begin{equation}
    \begin{gathered}
        c\sim \Bernoulli(0.5),\quad p(\tau|c)=\begin{cases}
            \Beta(\tau; 10, 1) & \text{~for~} c=1\\
            \Beta(\tau; 1, 10) & \text{~for~} c=0
        \end{cases}, \quad p(x_t | \tau) = \Bernoulli(\tau) ~~\forall t\in\{1,\ldots,T\}
    \end{gathered}
\end{equation}
This simulates polarized latent factors underlying, for example, texts with extreme sentiments (e.g., Twitter US Airline Sentiment, Yelp Reviews). We use this as our pretraining DG $p$.

We use the familiar $q=\Bern(0.7)$ as the task DG (an IMD case), and find the theoretical optimal prompts under different values of $\Lmax$ but $T=100$. The results are shown in \cref{fig:hierarchical}(left).
The proportion of \ONEs in the theoretically optimal prompt starts from 1.0 and then oscillates around 0.6, before slowly converging towards the task $\tau=0.7$.
\begin{figure}[h]
    \centering
    \includegraphics[width=0.7\linewidth]{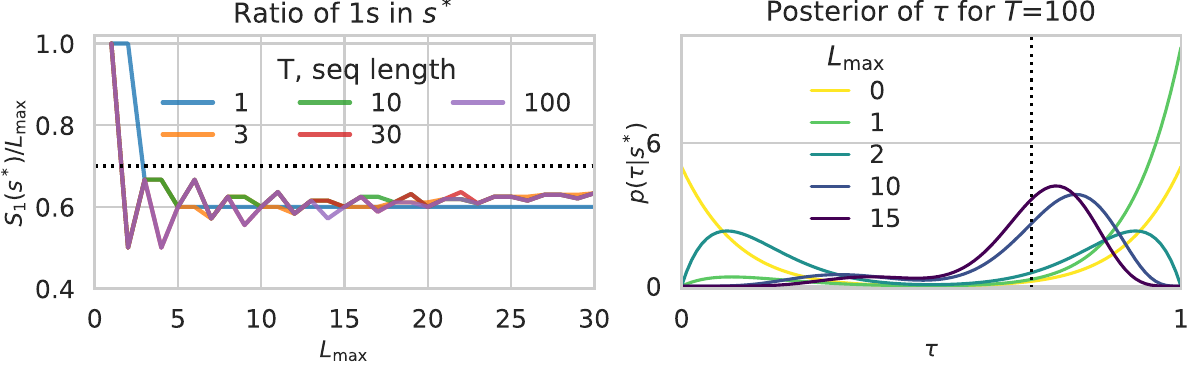}
    \caption{Results of the hierarchical DG experiments. Left, the proportion of \ONEs in the theoretically optimal prompt starts from 1.0 and then oscillates around 0.6, before slowly converging towards the task $\tau=0.7$. 
    Right, for the case of $T=100$, the posterior of $\tau$ given $s^*$. The vertical dotted line shows the true task $\tau=0.7$}
    \label{fig:hierarchical}
\end{figure}
Understanding why this arises is more difficult, because the Bayesian update involves the interaction between the prompt, the posterior mixing weights between the two components, and the posteriors of the two Beta components, so we provide a qualitative explanation. In \cref{fig:hierarchical}(right), we show the posterior distribution of the latent $\tau$ given the theoretical prompt, $p(\tau|s^*)$. The goal of the prompt is to make the probability mass around the task $\tau=0.7$ as high as possible. To do this, the optimal prompt needs to carefully balance the counts to achieve this, especially when $\Lmax$ is short. In this regime, the Beta components are still very extreme, so the mixing weights need to be more or less around 0.5. To shift this towards an intermediate value 0.7, both components need to become less extreme, which can be achieved by prompts with equal numbers of \ZEROs and \ONEs. Since $\tau=0.7>0.5$, we also end up having slightly more \ONEs. Having more \ONEs in the prompt will result in higher mixing weights for the rightmost component, biasing the posterior more towards 1.0. Due to the complexity of the problem, this explanation does not fully describe the posterior update process.

\subsection{Simple Topic Model}\label{sec:topic_model}

Next, consider a further complication of the hierarchical model above, which 
generates a binary ``document'' of 
 $K$ ``sentences''. Each sentence is a binary sequence drawn from two possible ``topics'' (two different Beta distributions).
\begin{equation}
    \begin{gathered}
        c\sim \Bernoulli(0.5),\quad p(\tau_k|c)=\begin{cases}
            \Beta(\tau_k; 2, 1) & \text{~for~} c=1\\
            \Beta(\tau_k; 1, 2) & \text{~for~} c=0
        \end{cases}~~\forall k\in\{1,\ldots,K\}.\\
        p(x_{k,t} | \tau) = \Bernoulli(\tau_k) ~~\forall (t,k)\in\{1,\ldots,T\}\times\{1,\ldots, K\}.
    \end{gathered}
\end{equation}
It is a simple topic model: two equally likely topics, each topic determines the probability of \ONEs, and a document with multiple sentences can have multiple topics. (c.f.\ Latent Dirichlet Allocation~\citep{blei2003latent} adds a distribution over topics for each document.)

The task DG is a document of either two or four sentences (sequence to generate) with lengths 2$T$ or 4$T$. For two sentences, the biases towards \ONEs for the sentences are [25\%, 70\%]. For four sentences, the task biases are [25\%, 70\%, 25\%, 70\%]. The prompt space we optimize over is two    binary ``sentences'' ($s^1$ and $s^2$), each with maximum length $\Lmax$. The prompt is expressible by two pairs of counts: $[(S_0(s^{1}), S_1(s^{1})), (S_0(s^{2}), S_1(s^{2}))]$ defined as [(\#\ZEROs, \#\ONEs) in sentence 1, (\#\ZEROs, \#\ONEs) in sentence 2].

The theoretical prompts are shown in \cref{fig:topic_model}. The pattern is in general very complicated. For longer documents (larger $T$), the sentences in the theoretical prompt become shorter, with a slight bias towards more zeros. For the 4-sentence document task, the two prompt sentences occupy different regions of the prompt space, which makes sense as there is more evidence for the two topics. Other than these, explaining why the theoretical prompts vary is very challenging.

\begin{figure}[h]
    \centering
    \includegraphics[width=0.7\textwidth]{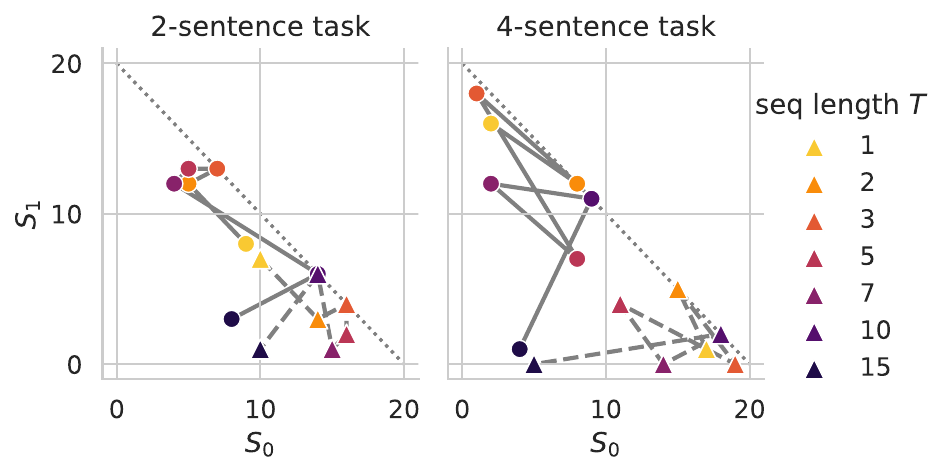}
    \caption{The theoretically optimal prompts for the topic pretraining DG. The solid line with circles represents one sentence, and the dashed line with triangles represent another sentence. The diagonal dotted line represents the maximum prompt length ($\Lmax=20$). The patterns are very challenging to explain.}
    \label{fig:topic_model}
\end{figure}

\section{BANDIT DECISION-MAKER}\label{sec:bandit_details}
\subsection{Skill Levels}\label{sec:bandit_details_skills}
In the main text, we specified that the skill parameter $\tau$ scales  the agent's counts of  the outcomes. Here, we detail how different skill levels are sampled to create the mixture of agents. 

To obtain agents with various skill levels, we modify the beliefs of the Thompson sampling (TS) agent by scaling the pseudocounts in the Beta posteriors of the reward probabilities by a skill level $\tau$. For skill level $\tau\in[0,1]$, for each arm $b\in\{\text{L}, \text{R}\}$, the posterior of the reward probability given past trajectories of $a_{1:t}\in\{\text{L},\text{R}\}^t$ and $r_{1:t}\in\{0, 1\}^t$ is

\begin{equation}\label{eq:bandit_arm_post}
v_{b,t,\tau}|a_{1:t}, r_{1:t},\tau = \Beta(1+\tau S_{b,1,t}, 1+\tau S_{b,0,t}),
\end{equation}
where
\begin{equation}\label{eq:bandit_counts}
\begin{aligned}
    S_{b,1,t}(a_{1:t}, r_{1:t}) &= \sum_{t'=1}^t \mathds{1}[a_{t'}=b] r_{t'}\,,\\
    S_{b,0,t}(a_{1:t}, r_{1:t}) &= \sum_{t'=1}^t \mathds{1}[a_{t'}=b] (1-r_{t'})\,.\\
\end{aligned}
\end{equation}
For each action, the agent first samples the reward from the posterior, and then chooses the action of the more rewarding arm. 

\cref{fig:TS_returns}(left) shows that the skill level affects the return most for lower values. As such, if we uniformly sample the skill level between 0 and 1, then there will be a lot of agents performing close to the optimal TS agent.

To avoid this, we define the skill $\tau=u^k$ where $u\sim \text{Uniform}([0,1])$ and $k>0$. We simulate the agent for different values of $k$ and $u$ for 300 actions repeated for 100k different random seeds, and show the returns as a function of $u$ in \cref{fig:TS_returns}(middle).
Through change of variable, we numerically estimate the distribution of the return for a given value of $k$ \cref{fig:TS_returns}(right). We pick $k=4$ throughout all bandit experiments, so that there is a mixture of agents across all levels.

There is still a significant proportion of performant agents, which should make prompting easier. We could have designed the transformation from $u$ to $\tau$ to be more complicated to induce a more uniform return distribution, but this is not essential to demonstrate our points.
\begin{figure}
    \centering
    \includegraphics[width=\textwidth]{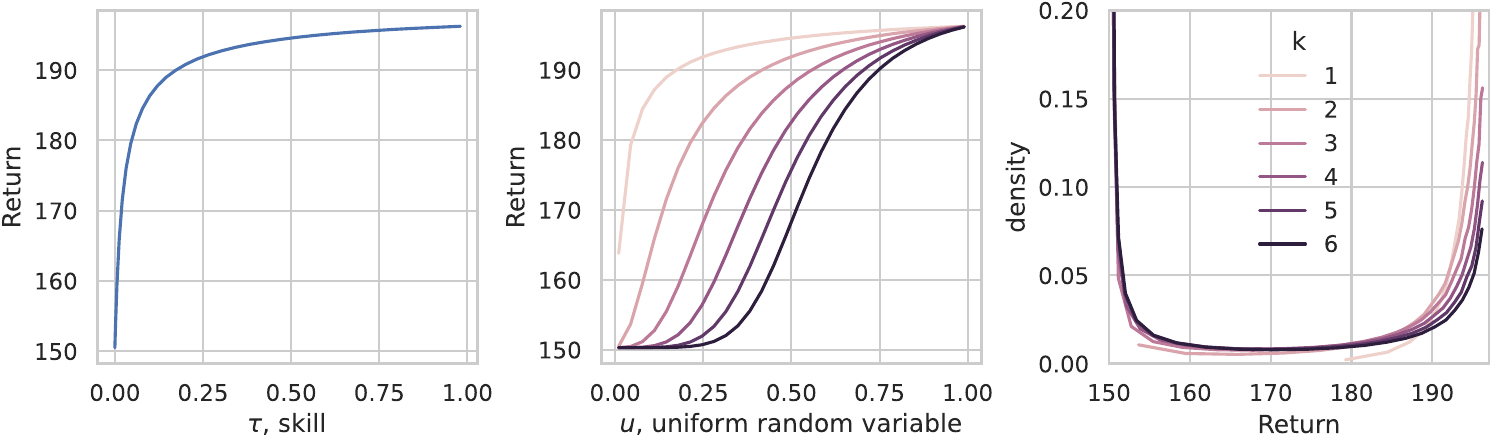}
    \caption{Relationship between the return and the uniform random variable for different values of the power index $k$.}
    \label{fig:TS_returns}
\end{figure}

\begin{algorithm}
\caption{Pretraining Trajectory Generation for the bandit task.\label{alg:bandit_pretrain}}
\begin{algorithmic}
\REQUIRE: $p(a_{t+1}|h_t,\tau)$, the TS-like agent that acts according to posterior reward probability samples with scaled pseudocounts \eqref{eq:bandit_arm_post} , given history $h_t:=(a_i,r_i)_{i=1}^t$.
\REQUIRE A Bernoulli two-arm bandit environment with a uniform distribution of reward probabilities on each arm.

\STATE Sample a skill level $\tau \sim \text{Uniform}(0, 1)$.

\COMMENT{Generate the prompt segment}
\STATE Sample bandit reward probabilities $r_\text{L}$ and $r_\text{R}$.
\STATE Let $x_1$ be an empty sequence.
\FOR{$i = 1$ to $8$}
    \STATE Sample agent action $a_i \sim p(a|h_i, \tau)$.
    \STATE Observe reward $r_i$ from the bandit given $a_i$, $r_\text{L}$ and $r_\text{R}$.
    \STATE Append $(a_i, r_i)$ to $x_1$.
\ENDFOR

\COMMENT{Generate the rollout segment}
\STATE Sample bandit reward probabilities $r_L$ and $r_R$.
\STATE Empty history $h$
\FOR{$i = 1$ to $300$}
    \STATE $a_i \sim p(a|h_i, \tau)$ where $h_i$ is the history up to step $i$.
    \STATE Observe reward $r_i$ from the bandit given $a_i$, $r_L$ and $r_R$.
    \STATE Append $(a_i, r_i)$ to $x_2$.
\ENDFOR

\STATE Concatenate $x = [x_1, \texttt{/}, x_2]$, where $\texttt{/}$ is a separator token.

\ENSURE $x$.
\end{algorithmic}
\end{algorithm}

\subsection{Approximate Bayes Predictor}\label{sec:bandit_details_bayes}
The key quantity required for predicting the action at each time step, marginalizing over all skill levels, is to compute the probability
\begin{align*}
\mathbb{P}(a_{t+1}=\text{L} | a_{1:t}, r_{1:t})
&= \int_0^1\int_0^1\mathds{1}[\vL > \vR]p(\vL, \vR| a_{1:t}, r_{1:t}) \ud \vL \ud \vR\\
&= \int_0^1\int_0^1\int_0^1
    \mathds{1}[\vL > \vR]p(\vL, \vR| a_{1:t}, r_{1:t}, \tau)
    p(\tau|a_{1:t}, r_{1:t}) \ud \vL \ud \vR \ud \tau\\
&= \int_0^1\int_0^1\int_0^1
    \mathds{1}[\vL > \vR]p(\vR| a_{1:t}, r_{1:t}, \tau)p(\vL| a_{1:t}, r_{1:t}, \tau)
    p(\tau|a_{1:t}, r_{1:t}) \ud \vL \ud \vR \ud \tau\\
&= \int_0^1\mathbb{P}(v_{\text{L},t,\tau} > v_{\text{R},t,\tau})p(\tau|a_{1:t}, r_{1:t})  \ud \tau,
\end{align*}
where the third equality uses conditional independence between reward probabilities given history and $\tau$. To approximate the last integral, we discretize the support at 1000 evenly spaced grid points. For each value of $\tau$ on the grid, we now need to compute the probability that one Beta random variable is greater than another. This does not have a closed  form solution, but we found the technique by~\citet{cook2012fast} to be fast and accurate compared to a Monte Carlo approximation. Finally, to compute $p(\tau|a_{1:t}, r_{1:t})$, we use the following recursion.
\begin{align}\label{eq:bandit_post_tau}
    p(\tau|a_{1:t}, r_{1:t}) \propto  p(\tau|a_{1:t-1}, r_{1:t-1}) p(a_t | a_{1:t-1}, r_{1:t-1}, \tau),
\end{align}
which is also approximated on the evenly spaced grid for $\tau$. 

\subsection{Theoretically Optimal Prompts}\label{sec:bandit_details_optimal_prompts}

The theoretically optimal prompt on the Bayes predictor above are found using the following steps. We first estimate the return using $10^5$ Monte Carlo rollouts for each of the $2^{16}$ prompts, using the same sequence of random seeds for actions selection and reward outcomes.
We then take the best 20 prompts with the top 20 estimated returns, and re-evaluate using $10^7$ Monte Carlo rollouts, using the same seed sequence for actions and rewards as above.
We sort the prompts according to the return evaluated on $10^5$ rollouts, and plot the return against the prompt rank in \cref{fig:bandit_prompt_index}. In this case, we can see the loss ``landscape'' is very sharp, indicating that the optimal prompts should be reliably identified. 

\begin{figure}[ht]
    \centering
    \includegraphics[width=0.8\textwidth]{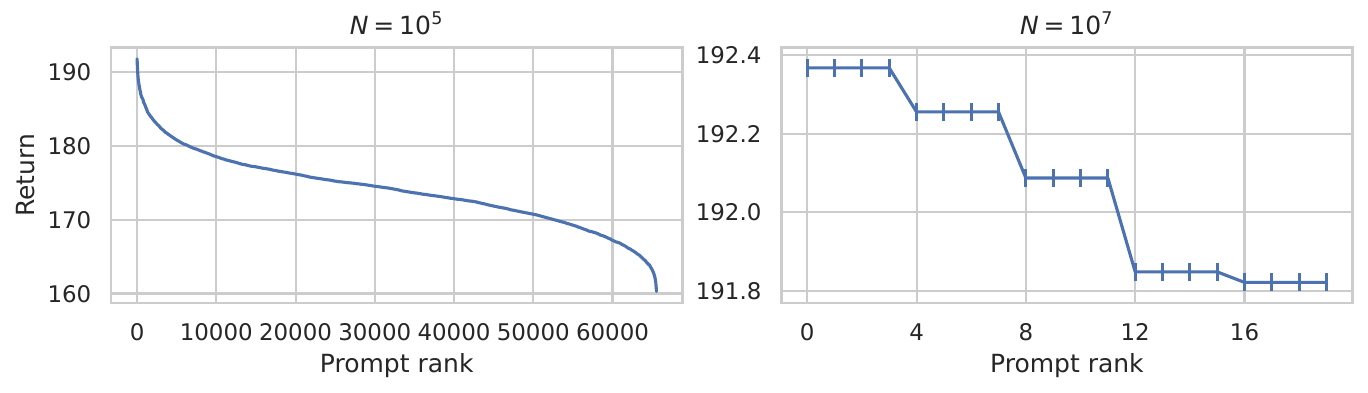}
    \caption{
    Estimated return of each prompt by Monte Carlo against the prompt rank. The prompts are sorted using estimated return under $N=10^5$ rollout trajectories. 
    \label{fig:bandit_prompt_index}
    }
\end{figure}

\paragraph{Why does the prompt look strange?}
These theoretically optimal prompts shown in \cref{fig:bandit} share the same pattern: try one arm and get no reward, then stick to the other arm and always get rewarded, except that the last reward may be missing. It may be striking to see that sticking to the first rewarding arm is the optimal strategy. 
The explanation of such persistence relies on how different skill levels are generated by the pretraining distribution: the reward pattern indicates that the second chosen arm is highly rewarding, and persistence to the more rewarding arm implies that the skill factor $\tau$ is likely large, which is desirable as it promotes more TS-like behavior. In addition, the first unrewarded outcome induces posterior Beta$(1, 1+\tau)$ over the reward probability on this arm, which is biased towards 0. The constantly rewarding streak from the other arm induces a posterior Beta$(1+7\tau,1)$ . 
Keep choosing the rewarding arm then indicates that $\tau$ is large. Essentially, a large \emph{reward gap} between the two Beta distributions helps the predictor identify $\tau$. The subtle interactions between the latent factor that we intend to manipulate $\tau$ and other latent factors (reward probabilities) resulted in the surprising prompts being in fact optimal.

Another large reward gap can be induced by other reward patterns, such as Beta$(1, 1+4\tau)$
and Beta$(1+4\tau, 1)$, which is brought by 4 unrewarded outcomes from one arm, and 4 rewarded outcomes from the other. However, choosing a previously unrewarded arm is unlikely to happen for an agent with large $\tau$, so such reward gap is not as effective as the one above in shifting the posterior of $\tau$ towards 1. 

\paragraph{Multiple optimal prompts.} The four equivalent optimal prompts are because of a simple symmetry between the left and right arm, and the fact that the posterior \eqref{eq:bandit_post_tau} does not depend on the last reward.

\subsection{Empirically Optimal Prompts}\label{sec:bandit_interp}

For each empirically optimal prompt, we estimate the return in the rollout segment using 300 actions and $N=10^6$ sequences. The results in \cref{fig:bandit_return} suggest that the performances of these prompts are very close to the ceiling on the Bayes predictor, even when optimizing using  $1000$  rollout trajectories with $50$ actions and rewards.

\begin{figure}[ht]
    \centering
    \includegraphics[width=\textwidth]{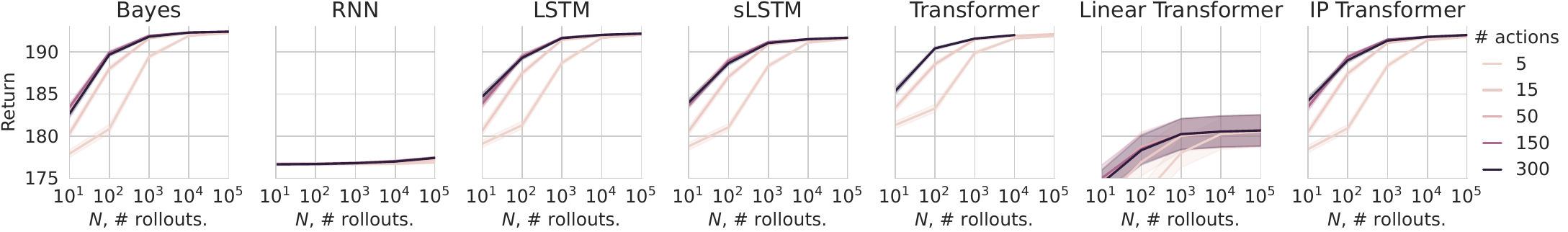}
    \caption{The estimated rollout return for empirically optimized prompts. Error bars show 1 SEM from 100 seeds.}
    \label{fig:bandit_return}
\end{figure}

However, the prompts may differ and cause inconsistent interpretations.
In order to interpret and visualize the empirically optimal prompts, we map each prompt to the Win-Stay/Lose-Shift (WSLS) probabilities which have been used in psychology to analyze human behavior on playing bandits~\citep{bruner1957perceptual,nowak1993strategy}. Specifically, WS is the probability that the previous action is repeated when receiving a reward, and LS is the probability of shifting to the other arm after an unrewarded outcome. 
For the theoretically optimal prompts, these probabilities are both 1.0. 
We compute the WSLS of the empirically optimal prompts from all predictors and show the distribution in \cref{fig:bandit_interp}. It shows that the empirical prompts may support multiple likely values for WSLS unless $N=100000$, in which case the WSLS is more concentrated at (1.0, 1.0). Therefore, the suboptimal prompts can lead to different interpretations, under the WSLS metric.

\begin{figure}[ht]
    \centering
    \includegraphics[width=\textwidth]{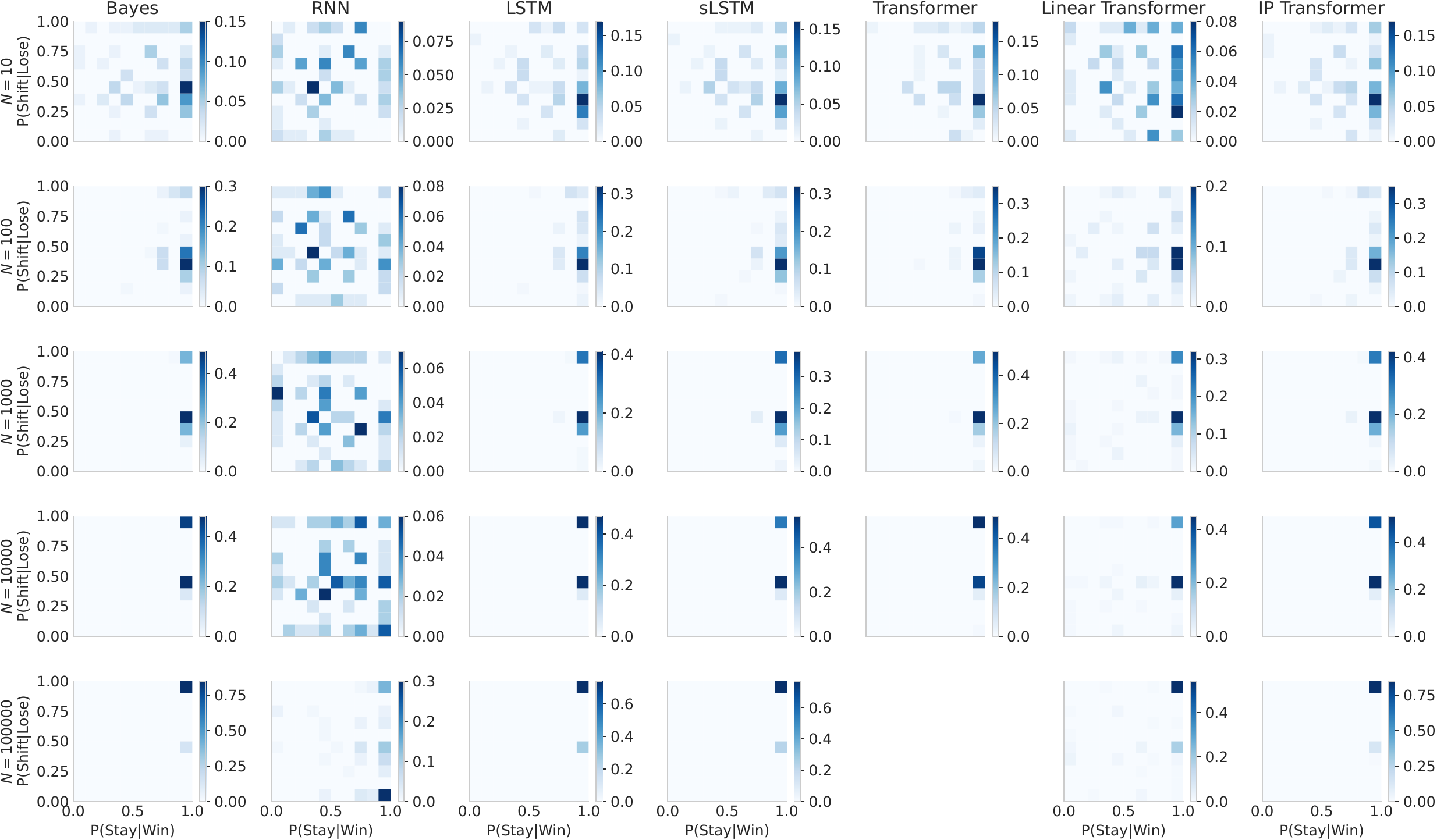}
    \caption{
    The distribution of Win-Stay/Lose-Shift for all empirical prompts for $300$ rollout actions and rewards.
    }
    \label{fig:bandit_interp}
\end{figure}

\subsection{Comparing Optimal and Typical Prompts}\label{sec:bandit_typical_prompt}

As in the switching problem in \cref{sec:switching}, we report the performance of statistically typical prompts in \cref{fig:bandit_typical_prompt}. In terms of expected total return, the optimal prompt is equivalent to typical prompts of length 60, roughly 4 times the length of the optimal prompt. This is consistent with the instantaneous regret. Note that, in the context of the bandit problem class, using longer expert demonstrations not only takes up more context window in a model, but also induces higher costs to the expert.

\begin{figure}[h]
    \centering
    \includegraphics[width=0.8\textwidth]{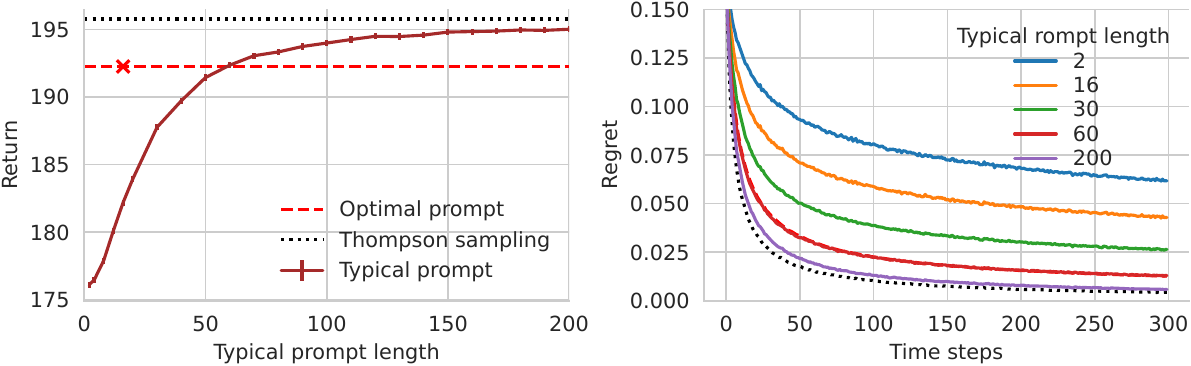}
    \caption{
    Total reward and regret of typical prompts, compared with Thompson sampling agent and the optimal prompt applied to the Bayes predictor. Estimated using $100\,000$ rollouts and typical prompts. Error bars on the left panel show 1 SEM.
    \label{fig:bandit_typical_prompt}
    }
\end{figure}

\clearpage

\section{REAL LLMS}\label{sec:real_llms_details}

\subsection{Experiment Design}
We ran experiments on open-source GPT-2~\citep[][MIT]{radford2019language} and Gemma-3~\citep[][Gemma]{team2025gemma} (Gemma3-1B-PT) and IMDB movie reviews~\citep[][non-commercial]{imdb}, using prompts generated by Gemini Pro 2.5~\citep{comanici2025gemini}. The choice of this dataset is motivated by the relatively small semantic space (positive or negative sentiments) and yet sufficient linguistic diversity in movie reviews. The goal is to explore behaviors of practical LLMs using our experimental framework, rather than to extrapolate our predictions from simplified DGs to the natural language domain.

Following the framework in \cref{sec:CIB-DG_methods}, we control the prompt setup (batch size of the reviews $N$, the length of prompt $L$, and the maximum review length at which we truncate the original reviews, similar to $T$) and the sentiment $\tau$ of the review (sequence-to-predict) and see how they affect the best prompt found under each prompt setup (referred to as optimized prompt). Inspired from the findings on binary sequences, we explore the following key features:
\begin{enumerate}
    \setlength\itemsep{-0.2em}
    \item the reliability (proportion correct) with which we can identify the optimal prompts (defined below);
    \item whether the sentiment of the prompt agrees with the sentiment of the review;
    \item the length of the optimized prompts.
\end{enumerate}

\paragraph{Finetuning.} We first finetune the pretrained GPT-2 and Gemma-3 model (no instruction tuning) on the IMDB training set. The data order has been shuffled under 50 different training random seeds. We take the snapshot with the lowest loss on the validation set. %

\paragraph{Prompt generation.} Instead of running any specific prompt optimization, we generate prompts by asking a more advanced LLM (Gemini 2.5 Pro). The prompts are descriptions of movies with positive and negative sentiment, suitable for pretrained models. We obtain prompts of various lengths (short, 1 adjective; short, \~ 5 words; and long, \~ 20 words) and of positive/negative sentiments, by giving these descriptions as prompt, together with instructions that encourage diversity and avoid repetition.  We assign a sentiment value of $+1.0$ to positive prompts, and $-1.0$ to negative prompts.

We evaluate the perplexity (results on log-loss are similar) on the full validation dataset at maximum length $T=512$, averaged over the prompts of each type. We also asked Gemini 2.5 Pro to generate single neutral adjectives as a control.  The results in \cref{fig:real_llm_logliks} shows that the single adjectives modulated the perplexity as expected. However, surprisingly, longer prompts resulted in higher average perplexity than shorter prompts for both positive and negative reviews. We later examine the optimal perplexity. 
\begin{figure}
    \centering
    \includegraphics[height=0.28\textheight]{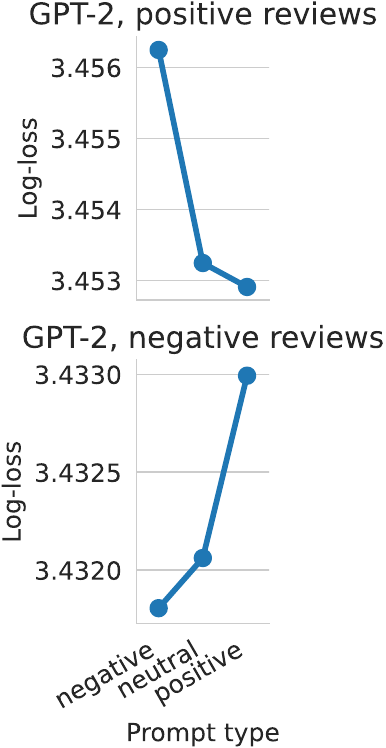}
    \includegraphics[height=0.28\textheight]{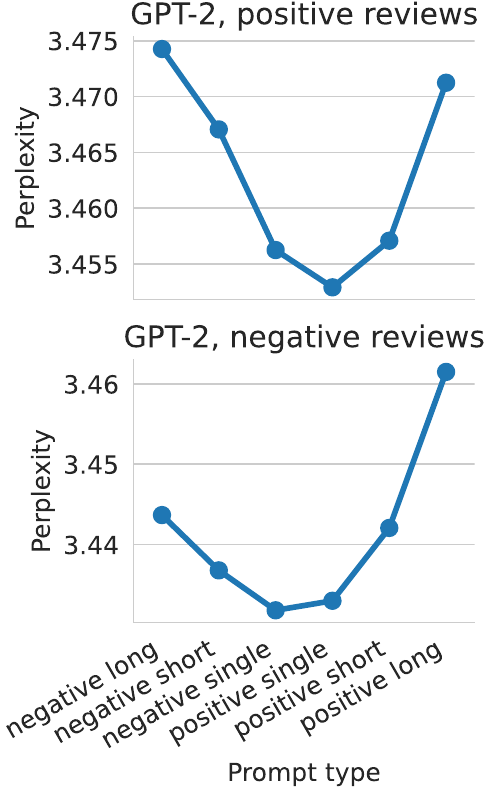}\hfill
    \includegraphics[height=0.28\textheight]{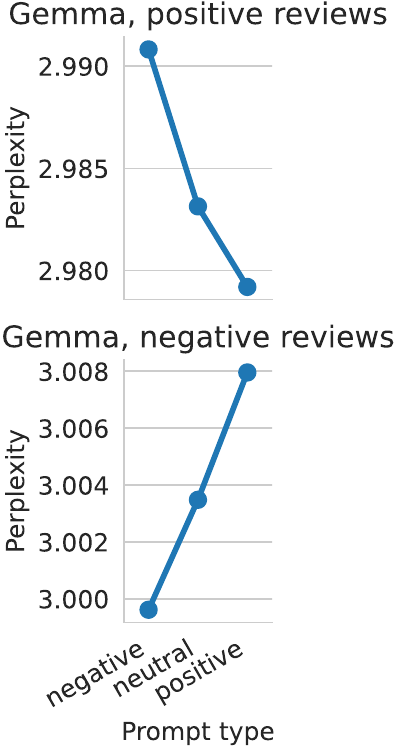}
    \includegraphics[height=0.28\textheight]{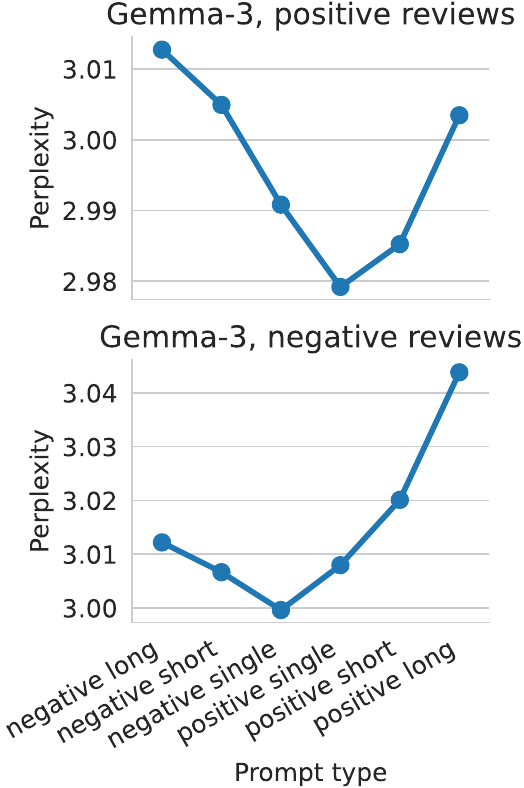}
    \caption{Perplexity of all positive and negative reviews, averaged over all prompts within each type (sentiments and lengths). The 1-sem error bars are not visible.}
    \label{fig:real_llm_logliks}
\end{figure}

\paragraph{Optimal prompt.}
For this experiment only, we choose the best 3 prompts from the prompt setup $N=25000$ and $T=512$ from all 50 finetuned models, and take their union as our set of optimal prompts. They are listed in \cref{tab:real_llms_opt_prompts}. They are mostly consistent with the target sentiment, except for the prompt ``Already forgetting about it.'' which was optimal for both sentiments on GPT-2.

\begin{table}[h]
\centering
\caption{The optimal prompts for each model and target sentiment of the reviews.}
\begin{tabular}{l|l|l}
         & \multicolumn{1}{c|}{GPT-2}                                                                                                            & \multicolumn{1}{c}{Gemma-3}                                                                                                \\ \hline
Positive & \begin{tabular}[c]{@{}l@{}}Already forgetting about it.\\ Perfect.\\ Foremost.\\ A powerful recommendation from me.\end{tabular}     & \begin{tabular}[c]{@{}l@{}}An easy 10 out of 10.\\ Recommended.\\ A+.\\ Fantastic.\\ Definitely.\\ Excellent.\end{tabular} \\
\hline
Negative & \begin{tabular}[c]{@{}l@{}}Already forgetting about it.\\ Could not have been worse.\\ Remiss.\\ Done.\\ Inconsiderate.\end{tabular} & \begin{tabular}[c]{@{}l@{}}Regrettable.\\ Negative.\\ An easy 1 out of 10.\\ Deficient.\\ Hooey.\end{tabular}             
\end{tabular}
\label{tab:real_llms_opt_prompts}
\end{table}

\subsection{Experiment Design}
The results are shown in \cref{fig:real_llms}. First, the sentiment of an optimized prompt depends on the LLM used and the target sentiment of the reviews. For GPT-2, increasing the dataset size makes the prompts more negative when prompting for negative reviews, but this is not the case when prompting for positive reviews. For Gemma-3, the prompt sentiment mostly agrees with the target sentiment, but when the data size $N$ is small, the sentiment can be less consistent for longer sequences $L$.

\begin{figure}[ht]
    \centering
    \begin{minipage}{0.3\textwidth}
    \includegraphics[width=\textwidth]{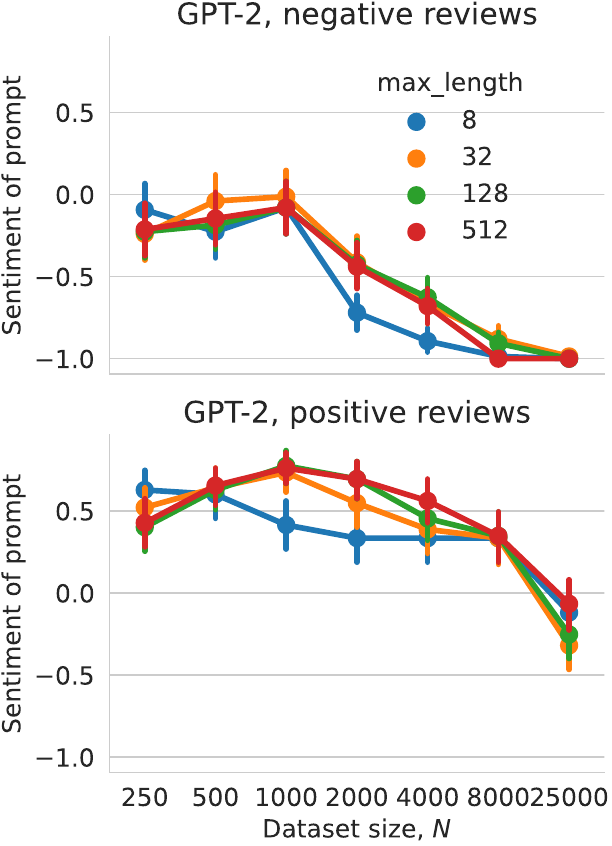}
    \end{minipage}\hfill
    \begin{minipage}{0.3\textwidth}
    \includegraphics[width=\textwidth]{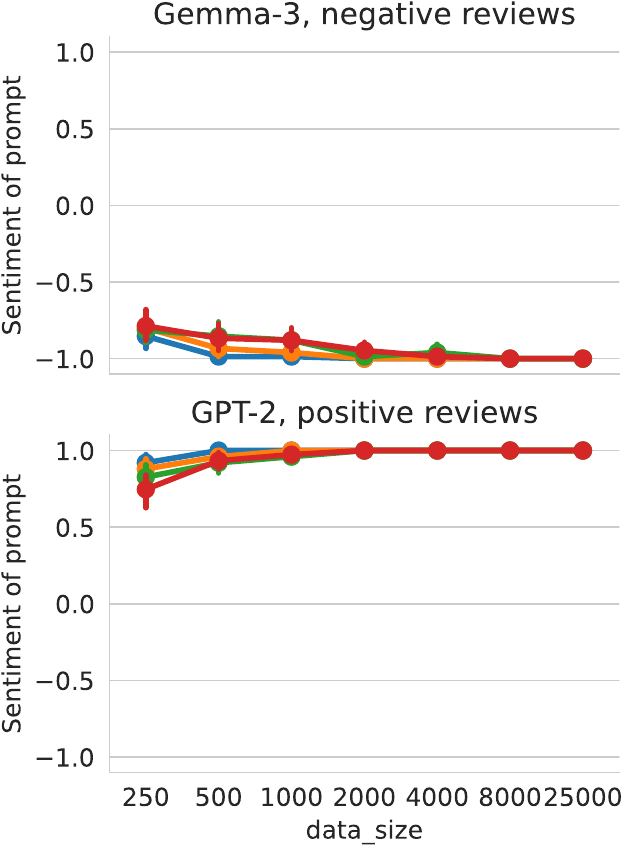}
    \end{minipage}\hfill
    \begin{minipage}{0.3\textwidth}
    \includegraphics[width=\textwidth]{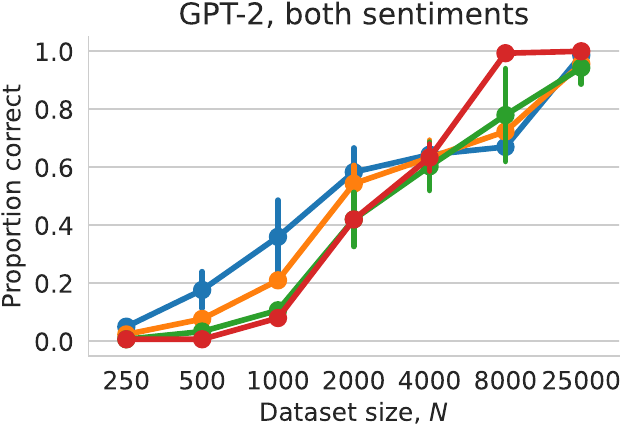}\\
    \includegraphics[width=\textwidth]{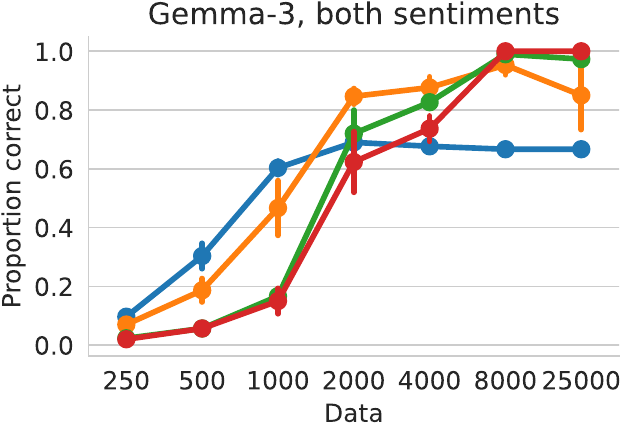}
    \end{minipage}
    \caption{Results of applying the analysis framework to LLMs. Colors show maximum review length. Error bars show 1 sem.}
    \label{fig:real_llms}
\end{figure}

Second, to reach 50\% chance of identifying the optimal prompt, one needs $N$ between 1000-2000 (\cref{fig:real_llms} right), much larger than the typically used 200~\citep{deng2022rlprompt,pryzant2023automatic,fernando2023promptbreeder,hao2024optimizing}. For dataset size $N$ between 250 and 3000, using longer sequences actually made it less likely to find the optimal prompt than using longer prompts.
These results suggest that the optimized prompts are less likely to be optimal and may not always reflect the sentiment of the review when using a small optimization batch size $N$.

Lastly, we show the average length of the optimized prompts in \cref{fig:real_llms_opt_prompt_length}. Overall, the lengths of optimized prompts tend to be short, averaging around 3 words.  This is consistent with other experimental findings that shorter prompts can be more effective~\citep{bhargava2023s,renze2024benefits,kusano2024longer,wang2025towards,lester2021power}. The length grows slightly with increasing $N$, but the effect of maximum review length is less consistent. 

\begin{figure}
    \centering
    \includegraphics[width=0.49\linewidth]{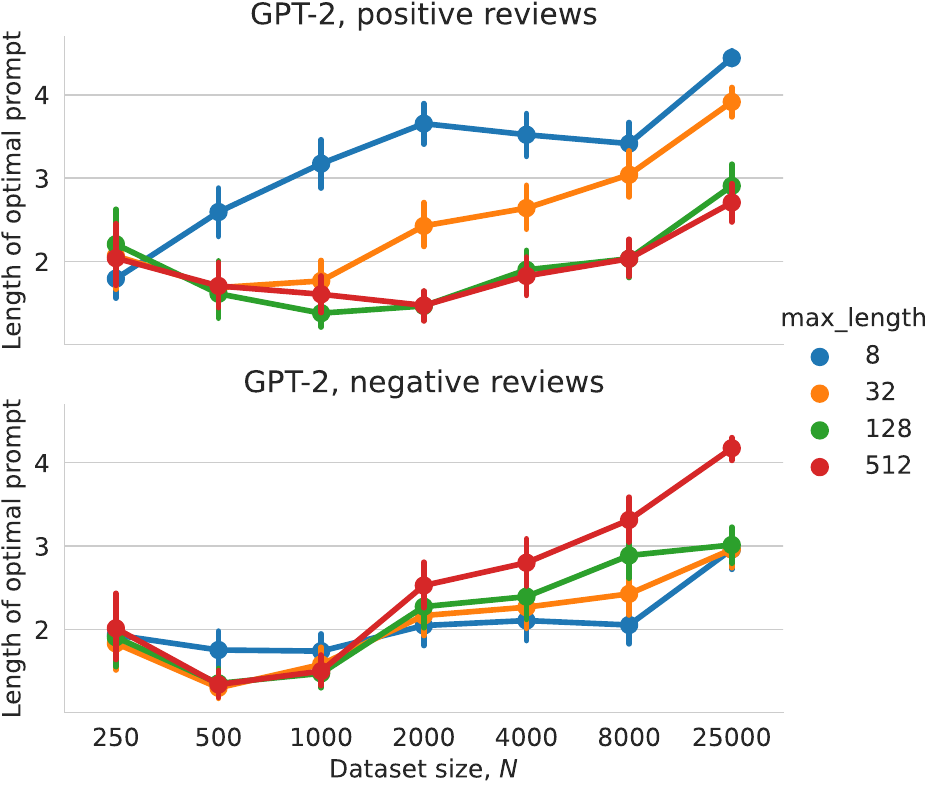}
    \includegraphics[width=0.49\linewidth]{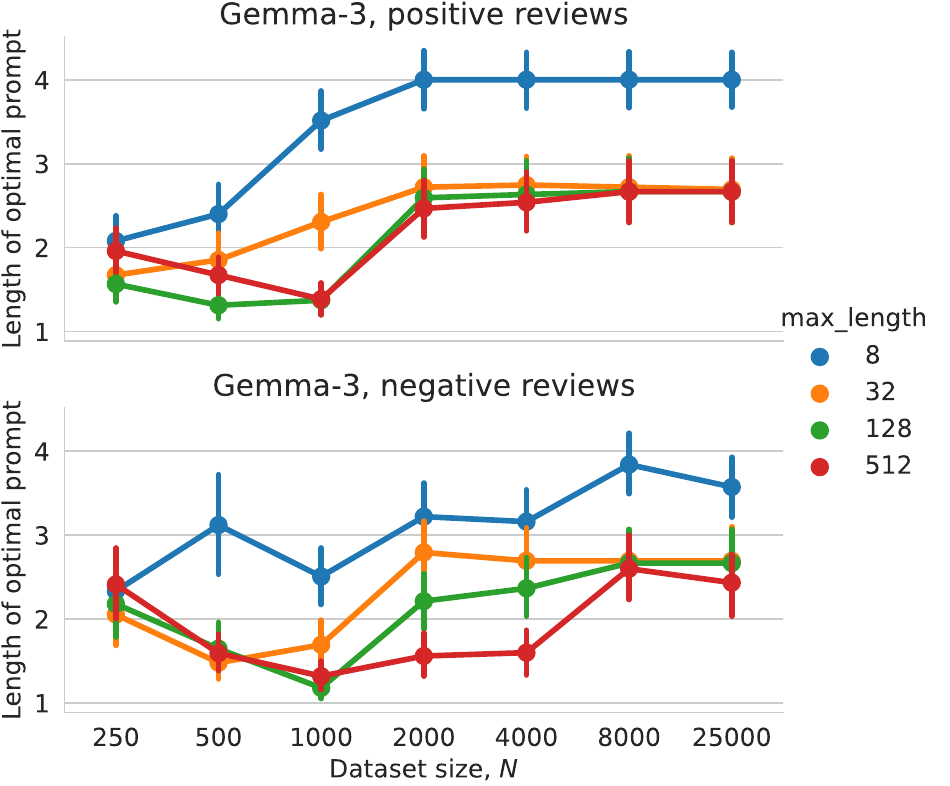}
    \caption{The average length of the optimized prompts. Colors show different maximum review length. Error bars show 1 sem.}
    \label{fig:real_llms_opt_prompt_length}
\end{figure}

\clearpage

\section{DETAILED DISCUSSIONS}

\subsection{Practical Guidance}\label{sec:practical_guidance}

Our results can lead to the following suggestions for practitioners:
\begin{enumerate}
    \item Our results show that both the support of the distribution (e.g., categories, data sources) and their proportions matter. For LLM owners, providing at least some information about the pretraining distribution can be useful for users to prompt.
    \item Try to determine whether the task is likely IMD or OOMD; this can affect the interpretability and performance of optimal prompts.
    \item When prompting an agent for in-context learning, bear in mind that some apparently idiosyncratic prompts can induce expert behavior better than expert demonstrations. 
    \item The common batch size (200) used in prompt optimization work is not enough to produce reliable prompts across different runs of the algorithm. Increasing the batch size makes the prompt more reliable and more effective, and more interpretable for IMD cases.
    \item Based on \cref{sec:ood_prompting}, the finding that shorter prompts may be better could be due to the increasing search space, which led to insufficient exploration. Spending more resources searching in this space may result in more effective prompts that are longer, but can be less cost-effective.
    \item If optimal performance is attainable by prompting, then the task is likely IMD, and the prompt could be more interpretable; if the task is OOMD, prompting may still bring some benefits, but the prompt can be less interpretable.
    \item The patterns in the optimal prompt of the bandit problem suggest that the most effective prompt that controls one latent factor, such as the skill level, assumes a trivial task setting represented by other latent factors: two arms with a huge reward gap. Future prompt engineering endeavors could consider choosing demonstrations corresponding to a variety of settings in other latent factors of the problem class, including cases that would be considered simple or trivial.

\end{enumerate}

\subsection{Relation to Previous Findings on LLM}\label{sec:related_findings}

First, the most effective prompt may not necessarily be typical samples from either the pretraining or the task distribution. 
This is because the optimal prompts serve to overcome and exploit the biases built into the predictor at pretraining to perform a task.
The fact that the pretrained distribution is usually unknown makes it very difficult to identify or to interpret optimal prompts based only on the task.
Previous work on LLMs found that the examples that make in-context learning work well can be counter-intuitive, and may even include ``wrong'' examples~\citep{min2022rethinking,yang2024prompts}. The influence of word frequencies on task performance has also been illustrated by~\citet{razeghi2022impact} and~\citet{wei2021frequency}. In our work, we additionally analyzed the effects of having discrete and continuous latent causes in the data distribution~\citep{xieexplanation,jiang2023latent}, and whether or not the pretraining distribution includes the task in terms of the latent cause distribution (IMD \emph{versus} OOMD). In particular, we showed that prompts under OOMD tasks can be harder to interpret.

Second, typical and thus interpretable samples from the task distribution may require a longer length to induce a good performance than optimal prompts do. This finding is consistent with~\citet{bhargava2023s,renze2024benefits,kusano2024longer,wang2025towards,lester2021power} who discovered that shorter/fewer \emph{selected} (not necessarily typical) demonstrations can work surprisingly well compared to longer prompts or more demonstrations.
On the bandit decision-maker experiment, heuristic prompts can also lead to unexpectedly good or bad results compared to sample trajectories from an expert agent, which is reminiscent of common experience interacting with LLMs~\citep{zamfirescu2023johnny,khurana2024and}. In our bandit experiment, again, knowing the pretraining distribution may help explain this mystery.

Third, with the optimal prompts obtained on Bayes predictors as ground-truth, we have shown that the empirically optimized prompts based on a finite dataset may or may not converge to the optimal prompts, even when using a large task dataset with long sequences. The unreliability of finding an optimal prompt means that the empirically optimized prompts on LLMs may vary substantially and unpredictably across different runs, predictors, prompt lengths, batch sizes, even for a fixed pretrained model, making interpretation more challenging. Previous works on prompt optimization methods (e.g.,~\citep{deng2022rlprompt,pryzant2023automatic,fernando2023promptbreeder}) typically do not compare or report how their methods perform as these hyperparameters vary. In particular, the batch size plays a significant role on proportion correct of empirically optimized prompts, and quite often it needs to be very large to ensure high chances of reaching the optimal prompt. 

\subsection{Limitations and Future Work}\label{sec:limitations}

\paragraph{Binary tokens vs text tokens.}

Binary tokens are a drastic simplification from text tokens.
While previous works on different data domains tries to understand the implications of different alphabet sizes~\citep{rajaraman2024toward,ieremie2024protein,gagie2012note,heurtel2024compression}, they do not directly predict how our results on optimal prompts are affected by alphabet size.

Using sequences over binary tokens does not change the nature of the sequence prediction task, but the token space is much more limited than text tokens. A single token contains much more information, richer semantics, and reflect more complex structure. Although generalizing our data generators from Bernoulli to categorical distributions is straightforward, the latent factor also becomes complex, i.e. the number of biases grows with the alphabet size. Our results on optimal prompt do not explicitly depend on dimensionality and should still hold qualitatively in those regimes, although the interpretations may change, especially when there are additional hierarchical structures behind text tokens.

Binary sequence provides a neat and intuitive data domain to demonstrate how the Bayes meta-learning theory accounts for unintuitive prompts, but it is not critical to the core theory---one can use more complex data generators (with deeper hierarchies, nonstationarity, $n$-ary tokens) and still have access to closed-form Bayes predictors; in addition to those already complex examples in \cref{sec:more_complex_DGs}. In particular, thanks to using binary sequences, we are able to summarize a sequence of tokens into informative statistics, revealing rich features in the empirically optimal prompts and even their distributions. In addition, our experimental design and some metrics extend to natural languages and LLM scales (see our new LLM experiments below). Some summary results (e.g., proportion correct) are still available under more complex DGs and even natural languages, but the prompts would not be easily visualised.

\paragraph{Permutation invariance.}
Languages are not permutation-invariant, which makes it hard to analyse and summarise. Permutation-invariance in our experiments makes it possible to describe a binary prompt by counts, helping us visualize them and even their distributions; it also makes it easy to search for theoretically optimal prompts specified by the DGs, to generate concrete optimal prompt instances from the otherwise abstract theoretical principles. Such simplified DGs still led to non-trivial patterns shown in the paper, avoiding confounds that would arise from, say, more intricate semantics in languages and suboptimal prompt search.
As such, the permutation invariance and binary tokens should be regarded as simplifying or assistive experimental constructs that still allow us to appreciate the intrinsic complexities of optimal prompts.

The neural predictors used in our experiments are not permutation-invariant; this means that they are not guaranteed to be Bayes-optimal on these tasks~\citep{chlon2025llms}. We do not make this assumption when searching for optimal prompts, and searched exhaustively over all sequences, rather than counts as done for Bayes predictor. 

\paragraph{Non-statistical aspects of language prediction.}

We have investigated specifically the mechanism of sequence prediction, from a Bayesian meta-learning perspective. In particular, the data generators in this work have Bayes predictors that can be expressed in terms of fixed-dimensional sufficient statistics. They fully summarize the prompt history of any length as a posterior distribution over a fixed-dimensional task variable. This makes it natural to take the Bayesian view to explain many interesting phenomena shown in this work.

Previous work on synthetic datasets considered more ``retrieval-like'' mechanisms~\citep{xieexplanation, jiang2023latent}. 
For example,~\citet{allen2023physics} showed the observation that permuting the sentences can improve question answering or information retrieval .
The Bayesian perspective, though certainly applicable in this and all other sequence prediction problems, 
may not provide the most intuitive explanation to such phenomena.

Further, the Bayesian view also does not account for any post-pretraining stages of real LLMs, such as supervised finetuning, human preference learning, and those that promote search and reasoning capabilities. For example, it is possible to ask a language model to produce 70\% \ONEs with the prompt ``Give me a sequence of 70\% \ONEs.'', which is clearly out-of-scope for the Bayesian meta-learning theory. It is nonetheless possible that supervised finetuning makes it easy to prompt by effectively making the latent factor distributions more uniform over certain semantic domains, thus reducing the bias from just pretraining. Previous work also found distinct ways of generalization between in-context and in-weight learning~\citep{chan2022transformers}. Understanding these post-pretraining effects by probing the latent factor distribution could be an interesting future direction.

\paragraph{Typical prompt.}

When comparing the advantage of optimal prompts to typical prompts, we computed the \emph{expected} performance of samples or demonstrations from the task distribution. However, the performance of prompting with individual prompts can vary drastically. In practice, people often perform \emph{some} form of prompt selection within allowable budget. While hitting the optimal prompt by random chance is small, it is possible that for the task has a very flat ``loss landscape'', in which case finding a performant prompt by chance can be quite high. The advantage of optimal prompt is likely diminished in this case.

\paragraph{Importance of pretraining distribution.}

We attributed the cause of unintuitiveness to unknown pretraining distribution. It is often believed that the \emph{coverage} of pretraining data affects the range of capabilities in downstream applications, leading to the notion that if some behavior is included in the dataset, then the capability can be induced by prompting. Our results on in-meta-distribution experiments indicate further that the \emph{distribution} of the pretraining data affects the \emph{difficulty} in discovering and understanding optimal prompts. In other words, the existence of some behavior in the pretraining dataset does not imply that this behavior can be intuitively prompted, but rather the distribution of all behavior matters. To verify these hypotheses, one would need to pretrain different language models on datasets with controllable latent factors. Alternatively, there are techniques to debias the generated content by building a model of the generative behavior~\citep{gagne2023inner}.

\paragraph{Interpreting prompts.}

We have chosen to interpret rather simple semantics in binary sequences: the bias of a hypothetical coin, or actions and rewards in Bernoulli bandits. There are also symmetries in the prompts that would correspond roughly to synonyms or paraphrases of the same meaning. In real language models, however, the nature of the semantics may be quite different to a real-valued coin bias. Nonetheless, the simplicity of interpreting coin biases induces minimal subjective biases, which may be an issue with interpreting natural languages.

\paragraph{Other data modalities.}

Image models are often trained by maximum-likelihood (or methods based on its lower bounds). Consequently, the Bayesian meta-learning framework still applies: the difficulty of inducing certain features in an image can be attributed to characteristics of the unknown image data distribution, and the opaque mapping from a conditioning input to the desirable predictive distribution of the pixels to generate.

Examples of text-to-image applications may be easier to relate. As a simple example, suppose we restrict the text to a sequence of binary tokens, and let the image be pictures of a random mixture of apples and bananas, with the ratio of \ONEs and the proportion of apples determined by a common latent cause 
, then applying the same modeling assumption as done in our current paper allows us to make the exact same observations: the best binary sequence may not match the desired number of apples in the image. The binary tokens can be replaced by languages that reflect preferences over apples and bananas.

In the case of image in-painting, one can form a correspondence between the latent factors in the switching DGs in \cref{sec:switching} and basic features of images:
\begin{itemize}
    \item wavelength: frequency content in images;
    \item noise level: blurriness or noise level (e.g., salt-and-pepper noise), respectively.
\end{itemize}
For an image model trained on uniformly distributed frequency content and noise levels, the results in \cref{fig:switching} translate to the following prediction: the optimal partial image may have a different frequency level to the desired frequency level to be inpainted. In terms of higher-level semantics, the noise can also be related to, for instance, the number of pedestrians on a crowded street with a background wall, and the wavelength can be related to frequency patterns of the graffiti on that wall partially occluded by the pedestrian. To generate the same graffiti pattern, the level of occlusion noise can affect the best patterns of the partial graffiti of the prompt. Of course, these are abstract extrapolations from our experiments to richer data domains, and they need to be verified in future experiments.

\paragraph{Practical prompt optimization.}
The theoretically optimal prompts are defined by the problem setup (pretraining and task DG), and are independent of any search strategy. We use exhaustive search to guarantee that we identify this optimal prompt, without introducing confounds that interfere with the Bayes meta-learning perspective. Using practical but inexhaustive search will produce suboptimal prompts, which is not what we need for the theory.

We discuss the effect of using inexhaustive prompt search on the reliability of finding the optimal prompt. The chances that an inexhaustive prompt search will land on $s^*$ depend on the number of equivalently optimal prompts (\#OP), the size of all possible prompts (\#PP), and any search strategies used. The reliability of recovering the optimal prompt will be reduced by the ratio of \#OP / \#PP. More involved update strategies may improve on this ratio. However, in cases where the loss ``landscape'' has ``local optima'' (e.g., \cref{fig:sensitivity_0.6_emp_dist}), any prompt proposal strategy based on local mutations may get stuck at a suboptimal prompt.

\paragraph {General vs specific tasks.} In our LLMs experiment, the task was to generate positive or negative reviews. There, we found that shorter prompts could be more effective. It is possible that longer prompts may provide too specific descriptions, which may be inconsistent with certain aspects of entire sets of positive and negative reviews, such as certain styles, genres, etc. However, if the task is simply to generate a specific review, then the optimal prompt could be the review itself, followed by ``Now I repeat my review verbatim:'', which is a trivial solution with very long prompt length. 
We can then consider prompting for subsets of the IMDB reviews, such as those within certain genres, directors or franchises, and see how the prompt length varies depending on the generality of the reviews.  

\clearpage

\section{LIST OF NOTATION}\label{sec:notation}
\begin{tabbing}
  \hspace{0.13\textwidth} \= \hspace{0.73\textwidth} \= \kill
  {\bf Symbol }     \> {\bf Explanation}                                                    \\[0.5ex]
  DG                \> data generator                                                \\[0.5ex]
  CIB-DG            \> conditionally independent Bernoulli data generator  \\[0.5ex]
  $\tau$            \> hidden factor (e.g., $\in[0,1]$ for Bernoulli bias parameter and bandit skill level) \\[0.5ex]
  $\cX\ni x_t$      \> binary token alphabet                                                \\[0.5ex]
  $T\in\mathbb{N}^+$\> length of a task sequence to be predicted  \\[0.5ex]
  $x_{1:T}\in\cX^T$ \> a sequence (to be predicted) of length $T$  \\[0.5ex]
  $p_{x|\tau}(x_{1:T}|\tau)$ \> distribution of sequence with bias $\tau$                   \\[0.5ex]
  $p_\tau(\tau)$    \> prior distribution over $\tau$                                            \\[0.5ex]
  $p_x(x_{1:T})$      \> = $\int p_{x|\tau}(x_{1:T}|\tau)\ud p_\tau(\tau)$, pretraining distribution \\[0.5ex]
  $p$      \> $p_x$ above, reference to a previously mentioned pretraining distribution/DG \\[0.5ex]
  $L\in\mathbb{N}^+$\> context length                                                       \\[0.5ex]
  $s_{1:L}\in\cX^L$ \> prompt of length $L$ to condition on                                                \\[0.5ex]
  $p_\text{B}(x_{1:T}|s_{1:L})$ \> Bayes predictor under DG $p$ (= $\int p_{x|\tau}(x_{1:T}|\tau)p_{\tau|x}(\tau|s_{1:T})$ for piecewise conditionally independent DG)  \\[0.5ex]
  $p_\theta(x_{1:T}|s_{1:L})$ \> Neural predictor $s_{1:L}$ trained on data from $p$ given $s_{1:L}$   \\[0.5ex]
  $S_x(s_{1:L})$    \> number of times $x$ appears in $s_{1:L}$                                           \\[0.5ex]
  $\Lmax\in\mathbb{N}^+$           \> maximal prompt length                                                \\[0.5ex]
  $q(x_{1:T})$      \> task distribution                                         \\[0.5ex]
  $q$      \> reference to a previously mentioned task distribution                                         \\[0.5ex]
  $N\in\mathbb{N}^+$               \> task dataset size                                                        \\[0.5ex]
  $\DqN$          \> $= \{x_{1:T}^n\}_{i=n}^N$ data set                                   \\[0.5ex]
  $\beta$  \> parameter in the Beta prior distribution \\[0.5ex]
  $\epsilon$ \> Bernoulli parameter in the switching process \\[0.5ex]
  $\lambda$ \> half period of the switching process \\[0.5ex]
  $w$               \> probability of $\tau_2$ in $\BernMix{\tau_1, \tau_2}$ when $p_\tau=(1-w)\delta_{\tau_1} + w\delta_{\tau_2}$                                             \\[0.5ex]
  $w_L(s_{1:T})$    \> posterior weight of $\tau_2$                                         \\[0.5ex]
  ${\cal L}(p,q,s_{1:L})$ \> Log-loss of predictor $p(x_{1:T}|s_{1:L})$ under prompt $s_{1:L}$ for data from $q(x_{1:T})$              \\[0.5ex]
  $\hat{\cal L}(p,{\DqN}, s_{1:L})$ \> empirical log-loss given dataset $\DqN$ \\[0.5ex]
  $s^*_{1:L}(p,q)$  \> theoretically optimal ($\mathcal{L}(p,q,\cdot)$-maximizing) prompt of length $L$  \\[0.5ex]
  $s^*_\Lmax(p,q)$  \> theoretically optimal ($\mathcal{L}(p,q,\cdot)$-maximizing) prompt of length $\leq\Lmax$  \\[0.5ex]
  $\hat s_{1:L}(p,q)$\> empirically optimal ($\hat{\cal L}(p, \DqN, \cdot)$-maximizing) prompt of length $L$ \\[0.5ex]  
  $\hat s_\Lmax(p,q)$  \> empirically optimal ($\hat{\cal L}(p, \DqN, \cdot)$-maximizing) prompt of length $\leq\Lmax$ \\[0.5ex]  
\end{tabbing}

\end{document}